%% file: main.tex
\documentclass[11pt]{article}

\usepackage[margin=1in]{geometry}
\usepackage{times}

\PassOptionsToPackage{numbers}{natbib}
\input{util.tex}
\usepackage{natbib}

\title{Curvature Beyond Positivity:\\ Greedy Guarantees for Arbitrary Submodular Functions}

\author{%
  Yixin Chen\thanks{Department of Computer Science \& Engineering, Texas A\&M University.
    Email: \texttt{chen777@tamu.edu}.}
  \and
  Alan Kuhnle\thanks{Department of Computer Science \& Engineering, Texas A\&M University.
    Email: \texttt{kuhnle@tamu.edu}.}
}

\date{May 8, 2026}

\begin{document}

\maketitle

\begin{abstract}
Submodular functions---functions exhibiting diminishing returns---are
central to machine learning.
When the objective is monotone and non-negative, the greedy algorithm
achieves a tight $63\%$ approximation.
But many practical objectives incorporate costs that make them negative
on some inputs, and all existing multiplicative guarantees require
non-negativity.
Prior work handles negativity through additive bounds for the
special class of decomposable functions and
non-monotonicity through partial-monotonicity parameters, but these
address each difficulty in isolation and neither extends the classical
structural theory.
We extend \emph{curvature}---a parameter measuring how far a function
deviates from linearity---to all submodular functions, handling both
non-monotonicity and negativity through a single classical concept.
A greedy algorithm with pruning achieves a curvature-controlled
multiplicative ratio for \emph{any} submodular function, including those
taking negative values---the first such guarantee beyond monotonicity
and non-negativity.
In the non-monotone regime $1 \le c_g < 2.2$, the bound strictly beats
the best known uniform ratio of $0.401$ (for non-negative $f$), and it recovers the classical $(1-e^{-c_g})/c_g$
guarantee for monotone functions.
A multilinear-extension variant extends the framework to general
combinatorial constraints via multilinear relaxation.
Experiments on cost-penalized experimental design, coverage, feature
selection, and a curvature sweep on Multi-News passage selection
support the theory.
\end{abstract}

%% ================================================================
\section{Introduction}\label{sec:intro}
%% ================================================================

\noindent\textbf{Submodular optimization in machine learning.}
Sensor placement for environmental
monitoring~\citep{krause2008near}, data summarization for information
retrieval~\citep{mirzasoleiman2016fast}, feature selection for model
training~\citep{kazemi2021regularized}, and Bayesian experimental
design~\citep{harshaw2019submodular} share a common mathematical
structure: their objective functions exhibit \emph{diminishing
returns}---adding an element to a smaller collection yields at least
as much benefit as adding it to a larger one.
Functions with this property are called \emph{submodular}, and the
canonical optimization task is to select a subset of at most $k$ items
maximizing a submodular objective~$f$.
When $f$ is \emph{monotone} (more items never hurt) and non-negative,
the simple greedy algorithm---repeatedly adding the element with the
largest marginal gain---achieves a tight
$(1-1/e) \approx 0.632$ approximation~\citep{nemhauser1978analysis}.
The greedy guarantee has made submodular optimization a practical
workhorse in these settings.

\noindent\textbf{When objectives go negative.}
In practice, each of these applications involves costs.
Deploying a sensor costs money; including a document incurs
retrieval overhead; each training example or experiment requires
resources.
The natural formulation subtracts a modular cost~$\ell$ from a
monotone submodular benefit~$g$, giving an objective $f = g - \ell$
that is still submodular but takes negative values whenever costs
exceed benefits---a sensor whose installation cost outweighs its
information gain, a feature whose $\ell_1$ penalty outweighs its
mutual information, or an experiment whose acquisition cost dominates
its value~\citep{krause2008near,kazemi2021regularized,harshaw2019submodular}.
This is the routine regime of cost-benefit trade-offs.
For non-monotone but non-negative objectives, a decade of work has
driven the best approximation from
$1/e$~\citep{BFNS2014} to
$0.401$~\citep{buchbinder2024constrained}, against a hardness
ceiling of $0.478$~\citep{oveisgharanvondrak2011,qi2024maximizing}.
But every result in this line requires $f(S) \ge 0$ for every~$S$:
the correlation gap argument~\citep{calinescu2011maximizing} that
connects the multilinear relaxation to the discrete problem collapses
when $f$ takes negative values, providing
\emph{no approximation guarantee whatsoever}.

\noindent\textbf{Two partial solutions and their limitations.}
Existing work addresses negativity \emph{or} non-monotonicity, through
separate frameworks; we extend a single classical concept to handle
both simultaneously.
\citet{harshaw2019submodular} initiated a line of work that decomposes
$f = g - \ell$ and proves that a distorted greedy achieves
$\mathbb{E}[f(S)] \ge (1-1/e)\,g(\mathrm{OPT}) - \ell(\mathrm{OPT})$,
an additive guarantee subsequently extended to several
settings~\citep{kazemi2021regularized,lu2024regularized,bodek2022maximizing}.
The additive guarantee mixes the incommensurable quantities
$g(\mathrm{OPT})$ and $\ell(\mathrm{OPT})$.
Converting it to a multiplicative ratio on~$f$ itself gives
$1 - 1/(e(1-\rho))$ where
$\rho = \ell(\mathrm{OPT})/g(\mathrm{OPT})$, which becomes
\emph{vacuous} once the cost ratio exceeds
$\rho \ge 1 - 1/e \approx 0.632$---a sharp threshold, not a graceful
degradation---and requires explicit access to $g$ and $\ell$
separately, while many objectives are given only as a value oracle.
Separately, \citet{mualem2022using} developed \emph{partial
monotonicity} as a structural parameter for non-monotone submodular
maximization, but this addresses non-monotonicity alone---it does not
handle objectives that go negative, nor does it connect to the
classical structural theory that made the greedy guarantee
instance-dependent.

\noindent\textbf{Curvature: from monotone to general.}
The concept we extend is
\emph{curvature}~\citep{conforti1984submodular}, a parameter
$c \in [0,1]$ for monotone submodular functions that measures how far
$f$ deviates from linearity.
When $c = 0$ the function is linear (modular) and greedy is exact; as
$c$ grows toward~$1$, diminishing returns become more pronounced and
the greedy guarantee degrades smoothly from~$1$ to~$1-1/e$ via the
bound $(1-e^{-c})/c$.
\citet{sviridenko2017optimal} proved this bound is tight under matroid
constraints, establishing curvature as the canonical
instance-dependent refinement for monotone submodular maximization.
The classical definition, however, diverges when marginals go
negative, and the associated guarantee says nothing.
This leads to a natural question:
\emph{can curvature be extended to non-monotone, and possibly negative,
submodular functions in a way that yields multiplicative
approximation guarantees?}

\begin{figure}[t]
  \centering
  \includegraphics[width=\textwidth]{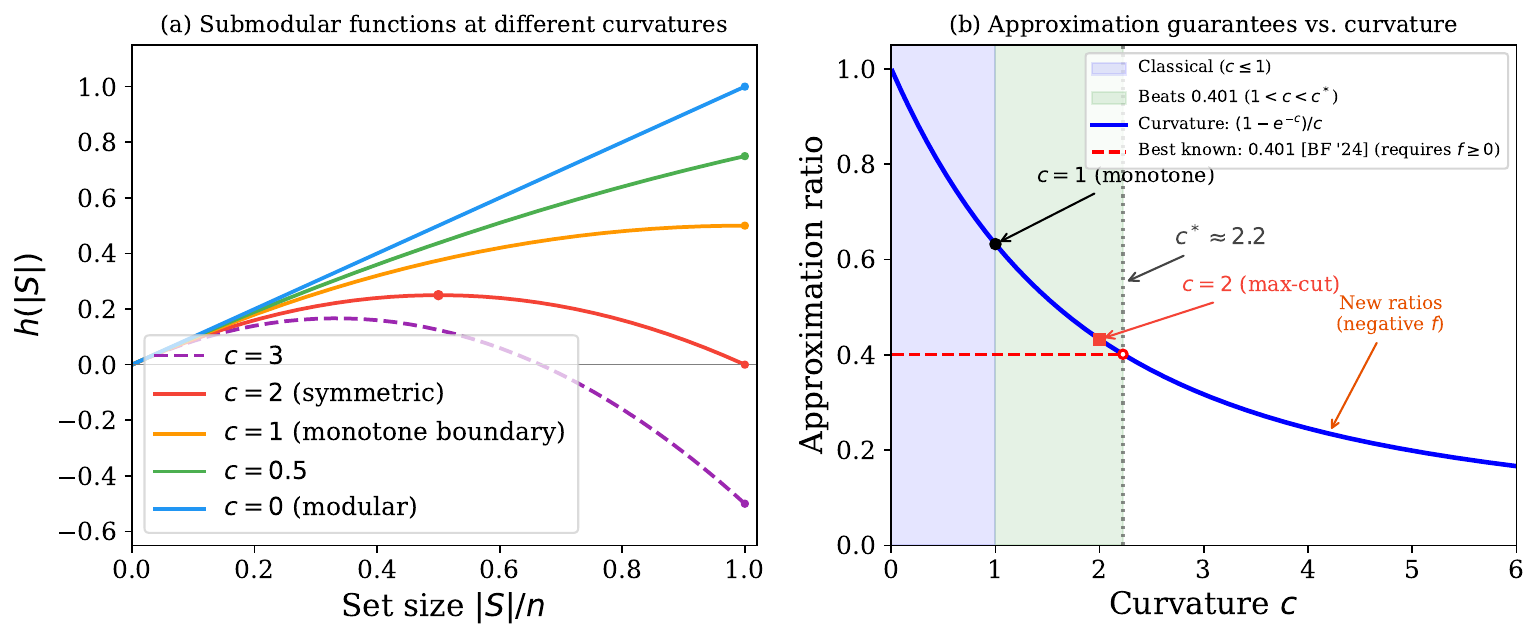}
  \caption{The curvature spectrum.
    \textbf{(a)}~Illustrative submodular functions $h(|S|)$ at
    different curvatures.
    Modular ($c=0$): linear growth.
    Low curvature ($c=0.5$): mild diminishing returns.
    Monotone boundary ($c=1$): the function flattens but
    remains non-decreasing.
    Symmetric ($c=2$): the function peaks at $|S| = n/2$ and
    returns to zero, as in max-cut.
    Beyond ($c=3$): the function goes negative for large sets.
    \textbf{(b)}~Approximation guarantees as a function of the relevant
    curvature parameter~$c$.
    The \textcolor{blue}{blue region} ($c \le 1$) recovers classical
    monotone results~\citep{conforti1984submodular,sviridenko2017optimal}.
    The \textcolor{green!50!black}{green region} ($1 < c < c^*$)
    is where curvature strictly beats the best known uniform ratio
    of $0.401$ (which requires $f \ge 0$)~\citep{buchbinder2024constrained}.
    Beyond $c^*$, the curvature guarantee is below $0.401$ but provides
    the only known multiplicative bound for negative-valued instances.}
  \label{fig:curvature-spectrum}
\end{figure}

\noindent\textbf{Contributions.}
We answer this question affirmatively: greedy with pruning achieves a
$(1-e^{-\bar c_g})/\bar c_g$ approximation for \emph{any} submodular
function---the first curvature-parameterized guarantee that extends
beyond monotonicity to negative-valued objectives.
Figure~\ref{fig:curvature-spectrum} previews the landscape; our
contributions fall into four categories.

\begin{itemize}
  \item \textbf{Curvature framework.}
    We extend the classical Conforti--Cornu\'{e}jols
    curvature---a parameter measuring deviation from
    linearity---beyond monotone functions for the first time.
    For monotone functions our definition recovers the classical one
    exactly (Definition~\ref{def:curv}); for general functions it
    provides a continuous characterization:
    $c_f \le 1$ if and only if $f$ is monotone.

  \item \textbf{Greedy guarantees.}
    A simple greedy algorithm with pruning---which deletes any current
    element whose marginal contribution to the rest has become
    non-positive---achieves a
    $(1-e^{-\bar c_g})/\bar c_g$ multiplicative approximation for
    \emph{any}
    submodular function under a cardinality constraint $|S| \le k$,
    with no monotonicity or non-negativity
    requirement (Theorem~\ref{thm:greedy}).
    Here $c_g$ is a trajectory-restricted greedy curvature
    (Definition~\ref{def:greedy-curv}) that need not be known at
    runtime. Clipping to $\bar c_g=\max\{1,c_g\}$ handles non-monotone
    cases; monotone instances use the sharper formula in terms of~$c_g$.
    The guarantee recovers the tight $1-1/e$ at $c_g = 1$ and
    strictly beats $0.401$ for $1 \le c_g < c^* \approx 2.2$
    (see~\S\ref{sec:apps-symmetric} for applications to symmetric
    submodular functions).
    The bound uses only a value oracle but is restricted to cardinality
    constraints; Section~\ref{sec:multilinear-summary} extends to
    general constraints via multilinear relaxation.

  \item \textbf{Beyond cardinality: DMCG-P.}
    The standard multilinear extension curvature~$c_F$ diverges to
    infinity for any function taking negative values
    (Theorem~\ref{thm:nn-characterization}), so extending the
    guarantee beyond cardinality constraints requires new machinery.
    We introduce a \emph{trajectory-restricted fractional
    curvature}~$c_g^F$ that sidesteps this obstruction.
    DMCG-P (Discretized Measured Continuous Greedy with
    Pruning), a discretization of the measured continuous greedy
    framework~\citep{feldman2017maximizing}, achieves
    $(1-e^{-\bar c_g^F})/\bar c_g^F \cdot f(O^*)$ minus a discretization
    error $O(n^2 C_F/T)$ that vanishes with the step count~$T$, with
    $\bar c_g^F=\max\{1,c_g^F\}$, for arbitrary
    submodular~$f$ with finite trajectory curvature, over
    integral feasible-set hulls.
    For decomposable $f = g - \ell$, a closed-form OPT-free certificate
    bounds~$c_g^F$ from the pruned trajectory alone.
    Details appear in
    Section~\ref{sec:multilinear-summary} and
    Appendix~\ref{apx:multilinear}.

  \item \textbf{Experimental validation.}
    On cost-penalized experimental design, coverage, and feature
    selection instances, greedy with pruning matches or exceeds distorted
    greedy~\citep{harshaw2019submodular}, and the curvature guarantee
    strictly exceeds the additive guarantee (which becomes vacuous as costs grow).
    At $c_g \approx 1.6$ the guarantee is $58\%$ against an observed
    ratio of $94\%$.
    On instances where computing the optimal solution is intractable,
    trajectory-based
    diagnostics remain informative without access to the optimal solution.
    A curvature sweep on Multi-News passage
    selection~\citep{fabbri2019multinews} shows that trajectory
    curvature follows the uniform-query reference curve; the
    certificate-based guarantee improves on the analytic bound by up
    to $20$ percentage points.
\end{itemize}

%% ================================================================
\section{Curvature and Greedy Algorithms}\label{sec:curvature}
%% ================================================================

We consider set functions $f: 2^\uni \to \reals$ defined on subsets of a ground set $\uni$
of size~$n$.
We assume $f$ is \emph{normalized}: $\ff{\emptyset} = 0$.
We say $f$ is \emph{submodular} if $\ff{X \cup Y} + \ff{X \cap Y} \le \ff{X} + \ff{Y}$
for all $X, Y \subseteq \uni$.
We do \emph{not} require $f$ to be monotone or non-negative.
Since we normalize $\ff{\emptyset} = 0$, we may assume that every element $e \in \uni$
satisfies $\ff{e} > 0$, as any element with $\ff{e} \le 0$ can be removed without
affecting the optimum.\footnote{%
  Normalization $\ff{\emptyset} = 0$ is standard for monotone maximization but
  typically avoided for non-monotone functions, as it may cause $f$ to take
  negative values.}

\subsection{Curvature: Definition and Basic Properties}\label{sec:curvature-defs}

Our generalization replaces the classical element-wise ratio
with a set-wise comparison: how much does the value of
$Y \setminus X$ degrade when added in the context of~$X$?
This avoids the element-level marginals that diverge for
non-monotone functions.
Since the minimum in Definition~\ref{def:curv} is taken over all pairs
with $\ff{Y \setminus X} - \ff{\emptyset} > 0$, and submodularity ensures
$\ff{X \cup Y} - \ff{X} \le \ff{Y \setminus X} - \ff{\emptyset}$,
the curvature satisfies $c_f \ge 0$ for every submodular~$f$.

\begin{definition}[Curvature]\label{def:curv}
  The \emph{curvature} $c_f \in [0, \infty)$ of a submodular function $f: 2^\uni \to \reals$
  is defined as
  \[
    c_f = 1 - \min_{\substack{X, Y \subseteq \uni \\
    \ff{Y \setminus X} - \ff{\emptyset} > 0}}
    \frac{\ff{X \cup Y} - \ff{X}}{\ff{Y \setminus X} - \ff{\emptyset}}.
  \]
\end{definition}

Equivalently, $c_f$ is the minimum value such that, for all $X, Y \subseteq \uni$
with $\ff{Y \setminus X} - \ff{\emptyset} > 0$:
\begin{equation}\label{eq:curvature-ineq}
  \ff{X \cup Y} - \ff{X} \ge (1 - c_f)\bigl(\ff{Y \setminus X} - \ff{\emptyset}\bigr).
\end{equation}
When $\ff{\emptyset}=0$ (as we assume throughout), the $-\ff{\emptyset}$ terms vanish.
When $c_f \le 1$, this says adding $Y \setminus X$ to $X$ remains non-negative
and captures at least a $(1-c_f)$ fraction of the standalone value of
$Y \setminus X$.
When $c_f > 1$, the union can \emph{degrade} value: adding $Y \setminus X$ to $X$
may reduce the function value by up to $(c_f - 1)\,\ff{Y \setminus X}$.
Submodularity also gives the companion upper bound
$\ff{X\cup Y}-\ff{X}\le \ff{Y\setminus X}$.

\begin{example}
For cardinality-constrained MaxCut, $c_f = 2$ (Proposition~\ref{prop:symmetric}).
\end{example}

The global curvature $c_f$ may be much larger than what the algorithm
actually encounters along its trajectory---a worst-case set pair may never
arise during execution.  To capture this, we define a tighter,
trajectory-specific variant.

\begin{definition}[Greedy curvature]\label{def:greedy-curv}
  Fix a cardinality constraint~$k$.
  Let $\mathcal{O} = \argmax_{|S| \le k} \ff{S}$ be the set of optimal solutions
  and let $A_0,A_1,\ldots,A_k$ be the active sets generated by
  greedy with pruning (Algorithm~\ref{alg:greedy}), repeating the
  terminal set if the algorithm stops early.
  These active sets are not required to be nested: pruning may delete
  elements chosen in earlier iterations.
  The \emph{greedy curvature} is
  \[
    c_g \;=\; 1 - \min_{O^* \in \mathcal{O}}\; \min_{i:\, \ff{A_i \setminus O^*} - \ff{\emptyset} > 0}
    \frac{\ff{O^* \cup A_i} - \ff{O^*}}{\ff{A_i \setminus O^*} - \ff{\emptyset}}.
  \]
  By construction, $c_g \le c_f$, since the minimization is over the
  trajectory pairs $(O^*,A_i)$, a subset of the pairs allowed in
  Definition~\ref{def:curv}.
\end{definition}

\begin{proposition}[Monotonicity characterization]\label{prop:c-monotone}
  A submodular function $f$ with $\ff{e} > 0$ for all $e \in \uni$ is
  monotone if and only if $c_f \le 1$.
\end{proposition}
This result means curvature cleanly partitions the function
space: $c_f \le 1$ is precisely the monotone regime, so
functions with $c_f > 1$ are \emph{necessarily} non-monotone.
The curvature parameter thus simultaneously measures distance from
linearity and distance from monotonicity.
See Appendix~\ref{apx:deferred} for the proof.

\begin{proposition}[Classical equivalence]\label{prop:classical}
  For monotone submodular $f$, the curvature $c_f$ equals the total curvature
  $\alpha$ of \citet{conforti1984submodular}:
  $1 - \alpha = \min_{e: \marge{e}{\emptyset} > 0}
  \marge{e}{\uni \setminus e}/\marge{e}{\emptyset}$.
\end{proposition}
This equivalence ensures backward compatibility: every result
proved using the CC parameter~$\alpha$ transfers directly to~$c_f$
for monotone functions, while our generalization extends the same
quantity beyond the monotone boundary.
See Appendix~\ref{apx:deferred} for the proof.

\noindent\textbf{Notation.}
Throughout, $c_f$ denotes the \emph{set-function curvature}
(Definition~\ref{def:curv}), a global quantity; $c_g$ denotes the
\emph{greedy curvature} (Definition~\ref{def:greedy-curv}), restricted
to the greedy trajectory and satisfying $c_g \le c_f$.
The multilinear analogues $c_F$ and $c_g^F$ (stated in
Section~\ref{sec:multilinear-summary}; formally defined in
Appendix~\ref{apx:multilinear}) extend these to continuous relaxations. When comparing to the classical
Conforti--Cornu\'{e}jols literature, $\alpha$ and $\alpha_g$ refer to
the CC curvature of a monotone submodular function, where $\alpha = c_f$
by Proposition~\ref{prop:classical}.

\subsection{Discrete Greedy with Pruning}\label{sec:greedy}

The cleanest algorithmic setting for the set-level curvature is the
cardinality-constrained discrete greedy algorithm
(Algorithm~\ref{alg:greedy}).  Here the
curvature inequality can be applied directly to sets along the greedy
trajectory, without first passing through a relaxation.
The multilinear obstruction discussed below motivates the
trajectory machinery needed for the continuous lift in Section~\ref{sec:multilinear-summary}.

The key algorithmic device in the discrete result is a \emph{pruning
loop}: after each
greedy addition, any element whose marginal to the rest of the current
set has turned non-positive is removed.
The pruning loop supplies the local positivity needed by the curvature
recurrence: every nonempty subset of the active set has positive value,
and $f$ is monotone within the active set (Remark~\ref{rem:pruning-loop}).

\begin{theorem}[Greedy guarantee]\label{thm:greedy}
  Let $f$ be submodular with greedy curvature~$c_g$
  (Definition~\ref{def:greedy-curv}), and let $O^* \in \argmax_{|S| \le k} \ff{S}$.
  Put $\bar c_g=\max\{1,c_g\}$.
  Then $\greedy(f,k)$ (Algorithm~\ref{alg:greedy}, Appendix~\ref{apx:algorithm}) returns $A_k$ with
  $\ff{A_k} \ge \frac{1-e^{-\bar c_g}}{\bar c_g}\,\ff{O^*}$.
  If $f$ is monotone, the classical Conforti--Cornu\'ejols analysis gives
  the sharper $\frac{1-e^{-c_g}}{c_g}$ bound when $c_g<1$.
\end{theorem}
\begin{proof}
  Write $A_0,A_1,\ldots,A_k$ for the active sets produced by the algorithm.
  Let $\bar c_g=\max\{1,c_g\}$.  Since $\bar c_g\ge c_g$, the curvature
  inequality is valid with $\bar c_g$ in place of~$c_g$.
  \begin{align*}
    \ff{A_{i+1}} - \ff{A_i}
    &\ge \ff{A_{i+1}'} - \ff{A_i} \tag{pruning can only increase} \\
    &\ge \frac{1}{k}\sum_{o \in O^* \setminus A_i} \marge{o}{A_i} \tag{greedy selection} \\
    &\ge \frac{1}{k}[\ff{O^* \cup A_i} - \ff{A_i}] \tag{submodularity} \\
    &\ge \frac{1}{k}[\ff{O^*} + (1-\bar c_g)\ff{A_i \setminus O^*} - \ff{A_i}] \tag{curvature} \\
    &\ge \frac{1}{k}[\ff{O^*} - \bar c_g\,\ff{A_i}]. \tag{$\bar c_g\ge1$, local monotonicity on $A_i$}
  \end{align*}
  The curvature step uses Definition~\ref{def:greedy-curv} at the pair $(O^*, A_i)$
  along the greedy trajectory; only the trajectory-restricted $c_g$ is needed.
  If $A_i\setminus O^*$ is nonempty, Remark~\ref{rem:pruning-loop} gives
  $\ff{A_i\setminus O^*}>0$, so the pair is admissible in
  Definition~\ref{def:greedy-curv}. If $A_i\setminus O^*=\emptyset$,
  the same display holds with the curvature term equal to zero.
  The last inequality uses $\ff{A_i \setminus O^*} \le \ff{A_i}$, which follows
  from the local monotonicity of $f$ on the pruned active set
  (Remark~\ref{rem:pruning-loop});
  this is the step that requires using a curvature parameter at least~$1$.
  If $\bar c_g \le k$, the recurrence
  $\ff{A_{i+1}} \ge \ff{O^*}/k + (1 - \bar c_g/k)\ff{A_i}$ solves to
  $\ff{A_k} \ge \frac{\ff{O^*}}{\bar c_g}(1 - (1-\bar c_g/k)^k) \ge \frac{\ff{O^*}}{\bar c_g}(1 - e^{-\bar c_g})$.
  If $\bar c_g>k$, the first greedy step already gives
  $\ff{A_1}\ge \ff{O^*}/k > \ff{O^*}/\bar c_g$, and subsequent
  accepted additions and prunings do not decrease the value; this is
  stronger than the displayed bound.
\end{proof}

\begin{remark}[The pruning loop]\label{rem:pruning-loop}
  The pruning loop removes elements with non-positive marginal contribution.
  By submodularity, for any nonempty $S \subseteq A_i$ the telescoping sum
  $\ff{S} = \sum_{j=1}^{|S|} \marge{s_j}{\{s_1,\ldots,s_{j-1}\}}
  \ge \sum_{j=1}^{|S|} \marge{s_j}{A_i \setminus s_j} > 0$,
  since every remaining element has strictly positive marginal to the rest.
  The same argument, applied to $A_i\setminus S$ and telescoping from $S$
  up to $A_i$, gives $\ff{S}\le \ff{A_i}$ for every $S\subseteq A_i$.
  Hence every nonempty subset of $A_i$ has positive value, and $f$ is
  monotone inside the active set.  These are exactly the local properties
  used in the proof above.
  The algorithm does not need to know~$c_g$.
\end{remark}

Algorithm~\ref{alg:greedy} runs in $O(nk)$ oracle queries.
Lazy (priority-queue) greedy can be used between pruning events, but
after pruning shrinks the active set the cached marginal upper bounds
must be rebuilt or revalidated. With that caveat, the practical query
count is often close to $O(n + k\log n)$. Whether a near-linear-time
threshold greedy variant exists is an open question
(see Section~\ref{sec:discussion}).

\noindent\textbf{Toward the multilinear extension.}
A natural question is whether our set-level curvature extends
to the multilinear extension~$F$ of~$f$.  Such an extension would
provide one route to lifting the greedy guarantee from cardinality
constraints to general combinatorial constraints---matroids,
knapsacks, and their intersections---via continuous relaxation
methods.
For strictly positive functions ($\ff{S}>0$ for all nonempty~$S$),
this works: the multilinear curvature~$c_F$ equals the discrete
curvature~$c_f$ (Theorem~\ref{thm:nn-characterization}(a),
Appendix~\ref{apx:multilinear}).
However, there is an obstruction.
When $f$ takes any negative value, $c_F = \infty$
(Proposition~\ref{prop:cF-infty}, Appendix~\ref{apx:multilinear}):
mass can be concentrated on an inclusion-minimal negative set,
driving the curvature denominator to zero.
Section~\ref{sec:multilinear-summary} resolves this via
DMCG-P, which restricts curvature to the algorithm's trajectory
and uses pruning to maintain a positive-slope invariant.

\section{Beyond Cardinality: the DMCG-P Guarantee}\label{sec:multilinear-summary}
%% ================================================================

The discrete greedy result (Theorem~\ref{thm:greedy}) handles
cardinality constraints.
For general combinatorial constraints---matroids, knapsack, and
their intersections---a standard approach is to optimize the
multilinear extension~$F$ over a relaxation of the feasible family.
The classical continuous-greedy analysis requires bounded global
curvature~$c_F$, but for negative-valued~$f$ this diverges
(Proposition~\ref{prop:cF-infty},
Appendix~\ref{apx:multilinear}).
We resolve this via DMCG-P (Discretized Measured Continuous Greedy
with Pruning; Algorithm~\ref{alg:dmcgp-main}), which restricts
curvature to the algorithm's trajectory.

The algorithmic ideas mirror the discrete proof but in fractional
form.  Each step chooses an integral feasible direction by exact
linear optimization over the current multilinear slopes, takes a
small $\delta=1/T$ step, and then zeroes coordinates whose slopes have become
non-positive.  Discretization lets the continuous trajectory be
analyzed through set-level greedy/pruning increments, while the
positive-slope pruning invariant removes the immediate zero-slope
obstruction that makes the global curvature~$c_F$ unusable.  The
theorem below still states finite trajectory curvature as a
hypothesis; in the applications section, decomposable objectives give
one concrete certificate for that hypothesis.

\begin{theorem}[DMCG-P guarantee; see Appendix~\ref{apx:multilinear}]
  \label{thm:dmcgp-summary}
  Let $\mathcal I \subseteq 2^\uni$ be downward-closed with
  $\polytope = \conv\{\mathbf{1}_S : S \in \mathcal I\}$, and assume
  exact linear optimization over~$\mathcal I$.
  For any submodular~$f$ (no monotonicity or non-negativity required)
  with trajectory-restricted fractional curvature
  $c_g^F < \infty$ and $\bar c_g^F\triangleq\max\{1,c_g^F\}$, DMCG-P
  with $T$ steps (step size $\delta=1/T$, with
  $\delta\bar c_g^F \le 1$) returns
  $\tilde S_T \in \polytope$ satisfying
  \[
    F(\tilde S_T)
    \;\ge\;
    \frac{1 - e^{-\bar c_g^F}}{\bar c_g^F}\,\ff{O^*}
    \;-\; O\!\left(\frac{n(n{-}1)\, C_F}{T}\right),
  \]
  where $n=|\uni|$, and
  $C_F = \max_{j \ne \ell,\,\vx\in[0,1]^\uni}|\partial_{j\ell}F(\vx)|$
  is the smoothness constant, with
  $O^* \in \argmax_{S \in \mathcal I} f(S)$.
  The step count~$T$ is independent of the constraint; choosing
  $T$ large enough makes the error negligible at the cost of $O(Tn)$
  oracle calls.
\end{theorem}

\noindent
The finite-curvature hypothesis uses the positive-denominator
convention in Definition~\ref{def:cgF}; pruning supplies the slope
invariant used in the proof, but a certificate is needed to verify
finiteness in a concrete instance.

In the application section below, decomposable objectives
$f = g - \ell$ provide one concrete way to certify the finite-curvature
hypothesis: the pruned trajectory yields an OPT-free ratio
$\hat r_F<1$ with
$c_g^F \le \alpha_g / (1 - \hat r_F)$
(Proposition~\ref{prop:dmcgp-decomp-cert}).
For non-negative non-monotone~$f$, a damped weighted variant
(wDMCG-P; Algorithm~\ref{alg:wdmcgp-damped}) recovers the $e^{-1}$
guarantee (Theorem~\ref{thm:wdmcgp-nonneg-main}).
Table~\ref{tab:regime-summary} in Appendix~\ref{apx:multilinear}
summarizes the full landscape.

%% ================================================================

With the general guarantee in hand, we now apply it to three
concrete function classes, each compared to the best previously
available tool.

%% ================================================================
\section{Applications and Experiments}\label{sec:applications}\label{sec:experiments}
%% ================================================================

Each application class answers a different question about the curvature
framework.
\emph{Symmetric submodular functions} test whether curvature subsumes
specialized structural results: with $c_f=2$, the guarantee recovers
the known $0.432$ ratio for graph cuts and
clustering~\citep{vondrak2013symmetry}
(Section~\ref{sec:apps-symmetric}).
\emph{Decomposable objectives} $f=g-\ell$ test whether curvature can
replace additive bounds: closed-form certificates yield per-instance
multiplicative guarantees that improve on the HFWK additive bound at
moderate cost ratios (Section~\ref{sec:apps-decomp}).
\emph{GCLin diversity objectives} test practical relevance: the relevance-minus-redundancy
family introduced for summarization
by~\citet{lin2011class} has curvature at most~$2\lambda$ under uniform query weights
($\lambda\le1$), giving
a $\lambda$-dependent multiplicative guarantee that remains positive
where the partial-monotonicity parameter
of~\citet{mualem2022using} has already vanished
(Section~\ref{sec:gclin-diversity}).
The decomposable experiments use greedy+pruning
(Algorithm~\ref{alg:greedy}) with distorted
greedy~\citep{harshaw2019submodular} as a baseline; the MaxCut
experiment adds standard greedy and random
greedy~\citep{BFNS2014}; the GCLin experiment uses
greedy with best-prefix selection.
Full details are in Appendix~\ref{apx:exp-setup}.

\subsection{Symmetric Submodular Functions}\label{sec:apps-symmetric}

Symmetric submodular functions---those satisfying
$f(S)=f(\uni\setminus S)$---are a central class in combinatorial
optimization, encompassing graph cuts, hypergraph
partitioning, and clustering
objectives~\citep{vondrak2013symmetry,feldman2017maximizing}.
Any normalized symmetric submodular function is automatically
non-negative: submodularity applied to $S$ and $\uni\setminus S$ gives
$2f(S)=f(S)+f(\uni\setminus S)\ge f(\uni)+f(\emptyset)=0$.
For general non-monotone maximization, the uniform guarantee is~$1/e$;
symmetric objectives are one structured setting where specialized
algorithms exceed it~\citep{feldman2017maximizing,wan2025symmetric}, but
curvature provides a structural explanation for why
greedy+pruning already exceeds this barrier.

\begin{proposition}[Symmetric functions]\label{prop:symmetric}
  Let $f$ be normalized, symmetric ($\ff{S} = \ff{\uni \setminus S}$),
  and submodular, and suppose $f$ has a positive singleton.
  Then $f$ is non-negative and $c_f = 2$.
\end{proposition}
\noindent\textbf{Proof idea.}
Symmetry forces $\ff{X \cup Y} - \ff{X} \ge -\ff{Y \setminus X}$
(ratio $\ge -1$, so $c_f \le 2$); choosing $X = \uni \setminus Y$ gives
$\ff{\uni} - \ff{Y} = -\ff{Y}$ (ratio $= -1$, so $c_f \ge 2$).
See Appendix~\ref{apx:deferred}.

Since $c_g\le c_f$, Theorem~\ref{thm:greedy} gives
\[
  f(A_k) \;\ge\; \frac{1-e^{-2}}{2}\,f(O^*) \;\approx\; 0.432\,f(O^*)
\]
under a cardinality constraint.
While specialized algorithms also achieve above-$1/e$ guarantees for this
class~\citep{feldman2017maximizing, wan2025symmetric}, curvature provides
a \emph{structural lens}: the guarantee
comes from deterministic greedy+pruning via a single parameter~$c_f$,
without using symmetry at runtime, and the same framework seamlessly
covers non-symmetric and negative-valued objectives.
This is also where the comparison with partial monotonicity is clearest:
for cut-type symmetric instances the monotonicity ratio can be zero,
while the curvature bound remains strictly above~$1/e$.
Appendix~\ref{apx:maxcut} validates this on cardinality-constrained
MaxCut instances, where greedy+pruning consistently improves on
standard greedy at large budget ratios.

Symmetric functions demonstrate that curvature subsumes specialized
structural results; we now turn to a class where it \emph{improves} on
the only available prior guarantee.

\subsection{Decomposable Objectives: the Curvature Certificate}\label{sec:apps-decomp}

Cost-penalized objectives $f = g - \ell$ arise whenever a
monotone benefit~$g$ competes with a modular cost~$\ell$: sensor
placement with deployment
costs~\citep{harshaw2019submodular}, regularized feature
selection~\citep{kazemi2021regularized}, and budget-constrained
experimental design are canonical examples.
\citet{harshaw2019submodular} gave the only prior guarantee for
this class: an \emph{additive} bound
$(1-1/e)\,g(O^*)-\ell(O^*)$ that becomes vacuous at moderate
cost ratios.
Curvature provides \emph{multiplicative} certificates.
Writing $g$ with CC curvature~$\alpha_g$ and $\ell$ modular
(so $F = G - L$ with $L$ linear in the multilinear extension),
the pruned trajectory itself gives an OPT-free bound on the curvature
encountered by the algorithm; neither statement requires
non-negativity of~$f$.

\begin{proposition}[OPT-free decomposable trajectory certificate]\label{prop:opt-free}
  Let $f = g - \ell$ where $g: 2^\uni \to \reals_{\ge 0}$ is monotone submodular
  with Conforti--Cornu\'{e}jols curvature $\alpha_g$, and $\ell: 2^\uni \to \reals_{\ge 0}$
  is modular.
  Let $A_0,A_1,\ldots$ be the active-set trajectory of greedy with pruning, and define
  \[
    \hat r
    \;=\;
    \max_{i}\max_{e\in A_i}
    \frac{\ell(e)}{\Delta_g(e\mid A_i\setminus\{e\})},
  \]
  with value~$0$ if all active sets are empty.  Then $\hat r<1$ and
  $c_g\le \alpha_g/(1-\hat r)$.
\end{proposition}

The denominator is observable after the run.  Moreover, pruning makes it
valid: every active element has positive $f$-marginal to the rest of the
current active set, so
$\Delta_g(e\mid A_i\setminus\{e\})>\ell(e)$.
The proof in Appendix~\ref{apx:certificate} shows that this elementwise
control lifts to every set difference $A_i\setminus O^*$, avoiding any
need to know the optimal solution.  A sharper OPT-aware ratio is stated
in Appendix~\ref{apx:deferred}.

A fractional analogue for DMCG-P trajectories, giving
$c_g^F \le \alpha_g/(1-\hat r_F)$ via the pruning invariant and
DR-submodularity, is proved in
Appendix~\ref{apx:dmcgp}
(Proposition~\ref{prop:dmcgp-decomp-cert}).

In the dominance regime $g \ge (1+1/\delta)\ell$
\citep{kazemi2021regularized}, we get $c_g \le (1+\delta)\alpha_g$
and the curvature guarantee dominates the HFWK additive bound for
all $\delta > 0$; HFWK becomes vacuous at $\delta \ge e-1 \approx 1.72$.
When $f$ takes negative values (dominance fails), HFWK is inapplicable,
while the trajectory certificates still yield the $\bar c_g$-curvature
bound via Theorem~\ref{thm:greedy}.  They are per-instance: after running
the algorithm, one computes the removal-marginal ratios along the pruned
trajectory and reports the resulting bound.

\noindent\textbf{Experimental validation (small instances).}
Small instances ($n = 20$, $k = 5$, 10 seeds) with exact OPT
validate the curvature framework on three decomposable objectives:
Bayesian experimental design ($\alpha_g = 0.667$), coverage
($\alpha_g = 1$), and feature selection ($\alpha_g \approx 0.79$).

\begin{table}[t]
\centering
\caption{Tier~1 (small instances, exact OPT) results: $n=20$, $k=5$.
  GP = greedy+pruning, DG = distorted greedy.
  Curv.\ guar.\ is $(1-e^{-\bar c_g})/\bar c_g$ using the empirical $c_g$.
  Cert.\ guar.\ is $(1-e^{-\bar{\hat c}_g})/\bar{\hat c}_g$, where
  $\hat{c}_g = \alpha_g/(1-\hat{r})$ is the OPT-free removal-marginal
  certificate (Proposition~\ref{prop:opt-free}).
  Under additional assumptions, a tighter certificate
  is available (Appendix~\ref{apx:certificate}).
  Values are means over 10 seeds.}
\label{tab:tier1}
\small
\begin{tabular}{@{}llccccccc@{}}
\toprule
Application & $\rho$ & GP ratio & DG ratio & $c_g$ & Curv.\ guar. & Cert.\ guar. & HFWK \\
\midrule
\multirow{4}{*}{Exp.\ Design}
  & $0.09$ & $1.00$ & $1.00$ & $0.13$ & $0.95$ & $0.63$ & $0.60$ \\
  & $0.29$ & $1.00$ & $1.00$ & $0.18$ & $0.93$ & $0.61$ & $0.49$ \\
  & $0.56$ & $0.98$ & $0.87$ & $0.76$ & $0.74$ & $0.38$ & $0.17$ \\
  & $0.76$ & $0.94$ & $0.60$ & $1.61$ & $0.58$ & $0.16$ & \textbf{$-0.56$} \\
\midrule
\multirow{4}{*}{Coverage}
  & $0.08$ & $0.98$ & $0.98$ & $0.68$ & $0.75$ & $0.58$ & $0.60$ \\
  & $0.31$ & $0.94$ & $0.94$ & $1.17$ & $0.60$ & $0.27$ & $0.46$ \\
  & $0.50$ & $0.89$ & $0.89$ & $1.41$ & $0.58$ & $0.14$ & $0.27$ \\
  & $0.68$ & $0.84$ & $0.78$ & $2.49$ & $0.37$ & $0.11$ & \textbf{$-0.14$} \\
\midrule
\multirow{4}{*}{Feature Sel.}
  & $0.11$ & $1.00$ & $1.00$ & $0.00$ & $1.00$ & $0.63$ & $0.59$ \\
  & $0.39$ & $1.00$ & $1.00$ & $0.00$ & $1.00$ & $0.48$ & $0.40$ \\
  & $0.57$ & $1.00$ & $0.91$ & $0.00$ & $1.00$ & $0.26$ & $0.10$ \\
  & $0.77$ & $1.00$ & $0.90$ & $0.00$ & $1.00$ & $0.16$ & \textbf{$-2.01$} \\
\bottomrule
\end{tabular}
\end{table}

Table~\ref{tab:tier1} summarizes.
Greedy+pruning achieves empirical ratios $\ge 0.84$ even at high cost
levels, substantially outperforming distorted greedy.
The curvature guarantee is conservative---at $c_g \approx 1.6$ the
guarantee is $58\%$ against an observed ratio of $94\%$.
The OPT-free certificate (Proposition~\ref{prop:opt-free})
provides a formal guarantee computable from the
pruned trajectory alone, without the small-penalty condition
required by the singleton approach; it remains valid even for coverage
($\alpha_g=1$).
The feature-selection tier is a near-modular sanity check: the
$\ell_1$-regularized mutual information is nearly linear on these
small instances, giving $c_g=0$ along the tested trajectories.
Experimental design and
coverage provide the main diminishing-returns stress tests.
Full cost-sweep curves and curvature-vs-ratio scatter plots appear
in Appendix~\ref{apx:exp-additional}.

\noindent\textbf{Moderate-scale instances.}
At moderate scales ($n$ up to $300$), exact OPT is intractable.
Trajectory-based curvature diagnostics stay positive at every cost level,
while HFWK becomes vacuous at moderate costs
(Appendix~\ref{apx:exp-additional}, Table~\ref{tab:tier2} and Figure~\ref{fig:tier2}).

\subsection{GCLin Diversity Objective}\label{sec:gclin-diversity}

Graph-cut-based objectives are a workhorse for document
summarization and information retrieval.
\citet{lin2011class} showed that the relevance--redundancy
tradeoff---each selected passage adds coverage of source
material but also redundancy with other selections---is
naturally captured by this family.
The graph-cut-linear (GCLin) objective
\(\ff{S}=R(S)-\lambda D(S)\), where
\(R(S)=\sum_{i\in\uni}\sum_{j\in S}s_{i,j}\) rewards relevance
and \(D(S)=\sum_{\substack{i,j\in S\\i\neq j}}s_{i,j}\) penalizes
redundancy~\citep{lin2011class}, instantiates this
directly.  LLM context selection---choosing which retrieved
passages to include in a prompt---poses the same structural
problem: each passage adds relevant evidence but also
redundancy, and the context budget imposes a cardinality
constraint.  Our experiments on Multi-News
(below) show that
mean trajectory curvature follows the uniform-query reference curve~$2\lambda$,
showing the redundancy penalty meaningfully changes the curvature profile.

\begin{proposition}\label{prop:movie-curv}
  Suppose the GCLin diversity-relevance objective has
  symmetric non-negative similarities, uniform query weights ($w_i = 1$),
  and $0\le\lambda\le1$.  Then $c_f \le 2\lambda$.
\end{proposition}

Figure~\ref{fig:movie-bounds}(a) compares the resulting uniform-query
curvature guarantee $(1-e^{-2\lambda})/(2\lambda)$ for
$\lambda\le1$ with the
partial-monotonicity-style guarantee of~\citet{mualem2022using}.
Partial monotonicity measures value loss under additions.
Curvature measures diminishing returns along the same tradeoff.
For GCLin, the partial-monotonicity bound vanishes by
$\lambda=1$, while the uniform-query curvature bound remains positive at
that boundary.  This is a relevant separation for diversity-heavy passage selection.
For $\lambda \le 1/2$, $f$ is monotone and both the discrete and DMCG-P
guarantees give $(1-e^{-2\lambda})/(2\lambda)$.
See Appendix~\ref{apx:deferred} for the proof.

\noindent\textbf{Curvature sweep on Multi-News passage selection.}\label{sec:lambda-sweep}
We vary curvature via~$\lambda$ on Multi-News~\citep{fabbri2019multinews}:
for 100 validation examples ($n = 200$, $k = 10$, TF-IDF embeddings),
greedy selects passages under $\text{GCLin}_\lambda$ for
$\lambda \in \{0.1, 0.25, 0.5, 0.75, 1.0, 1.5\}$ and a local
\texttt{openai/gpt-oss-120b} endpoint summarizes them
(prompt and decoding details in Appendix~\ref{apx:exp-setup}).

Figure~\ref{fig:lambda-sweep} shows two key findings:
\textbf{(b)}~the mean trajectory-curvature proxy
$\hat{c}_{\text{traj}}$ stays below the uniform-query reference bound~$2\lambda$
for $\lambda\le1$; individual instances occasionally exceed it
because the TF-IDF query weights are non-uniform
(the proposition assumes uniform weights).
$\lambda=1.5$ is reported only as an empirical high-redundancy diagnostic;
\textbf{(c)}~the resulting clipped trajectory guarantee is tighter than the
uniform-query reference curve for $\lambda\le1$ and remains informative as a
diagnostic at $\lambda=1.5$.

\begin{figure}[t]
  \centering
  \begin{subfigure}[t]{0.31\textwidth}
    \centering
    \includegraphics[width=\textwidth]{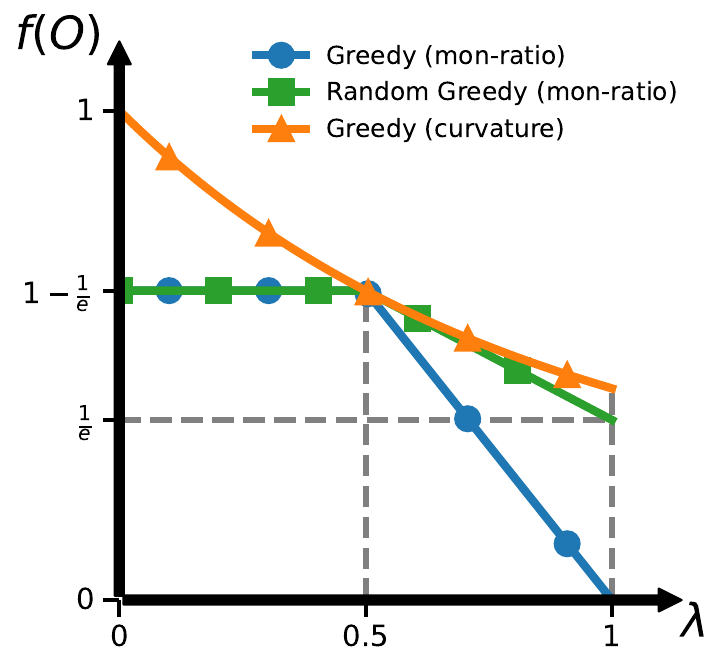}
    \caption{}
  \end{subfigure}\hfill
  \begin{subfigure}[t]{0.31\textwidth}
    \centering
    \includegraphics[width=\textwidth]{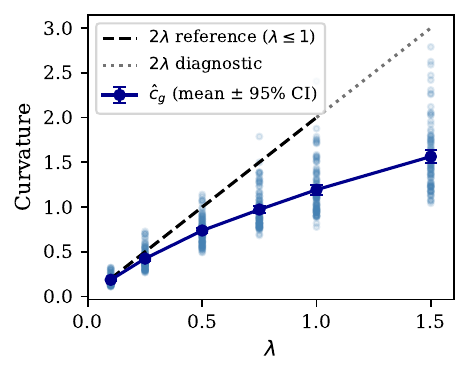}
    \caption{}
  \end{subfigure}\hfill
  \begin{subfigure}[t]{0.31\textwidth}
    \centering
    \includegraphics[width=\textwidth]{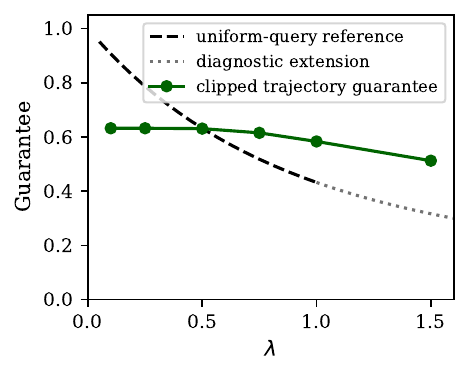}
    \caption{}
  \end{subfigure}
  \caption{GCLin diversity objective: theoretical and empirical analysis.
    \textbf{(a)}~Uniform-query curvature guarantee
    $(1{-}e^{-2\lambda})/(2\lambda)$ for $\lambda\le1$,
    compared with the partial-monotonicity bound.
    \textbf{(b)--(c)}~$\lambda$-sweep on Multi-News
    ($n{=}200$, $k{=}10$, 100~examples):
    \textbf{(b)}~the mean empirical trajectory-curvature proxy stays below the
    uniform-query reference bound $2\lambda$ for $\lambda\le1$;
    \textbf{(c)}~the clipped trajectory guarantee is tighter than the
    uniform-query reference curve.}
  \label{fig:lambda-sweep}
  \label{fig:movie-bounds}
\end{figure}

\noindent\textbf{Summary across applications.}
Symmetric functions: curvature \emph{subsumes} specialized results ($c_f{=}2$ recovers $0.432$ without using symmetry).
Decomposable objectives: curvature \emph{improves} on additive bounds, which fail at moderate cost ratios.
GCLin diversity: mean trajectory curvature follows the uniform-query reference curve~$2\lambda$ and yields tighter per-instance curves.

%% ================================================================
\section{Discussion and Conclusion}\label{sec:conclusion}\label{sec:discussion}
%% ================================================================

Curvature provides a single, continuous parameter that extends the
classical Conforti--Cornu\'{e}jols theory beyond monotonicity,
giving the first curvature-controlled multiplicative guarantees for
negative-valued submodular objectives.
Greedy with pruning achieves $(1-e^{-\bar c_g})/\bar c_g$ under
cardinality constraints; DMCG-P extends this to general combinatorial
constraints under explicit trajectory-curvature hypotheses.
Experiments on three function classes confirm that the guarantees are
conservative, and that trajectory-based curvature is
often substantially tighter than worst-case bounds.

A key limitation is that curvature is trajectory-specific: $c_g$
certifies the realized greedy-with-pruning path, not an input parameter
known before the algorithm runs.
For decomposable objectives, the removal-marginal certificate
(Proposition~\ref{prop:opt-free}) gives an OPT-free post-hoc bound from
the pruned trajectory, with OPT-aware and multilinear variants deferred
to the appendix.
Comparable certificates for broader value-oracle classes remain open.

Concrete open directions include: (i)~near-linear threshold or streaming
variants that preserve the local pruning invariant; (ii)~tight
value-oracle lower bounds for curvature-controlled greedy under
cardinality and matroid constraints; (iii)~certificates for weighted
relevance-minus-redundancy objectives, including nonuniform query
weights; and (iv)~algorithms that adapt pruning or selection to the
observed trajectory curvature.

%% ================================================================
%% References
%% ================================================================
\bibliographystyle{plainnat}
\bibliography{references}

%% ================================================================
%% Appendix
%% ================================================================
\appendix

\section{Additional Related Work}\label{sec:related}
%% Works already discussed in the introduction (Conforti--Cornu\'{e}jols,
%% Sviridenko--Vondr\'{a}k--Ward, BFNS, Buchbinder--Feldman, Harshaw et
%% al., Kazemi et al., partial monotonicity) are not repeated here.

\noindent\textbf{Symmetric submodular functions.}
\citet{feldman2017maximizing} proved that a variant of measured continuous
greedy achieves $(1-e^{-2})/2 \approx 0.432$ for maximizing a non-negative
symmetric submodular function under a cardinality constraint.
Section~\ref{sec:apps-symmetric} explains how the same threshold follows
from Proposition~\ref{prop:symmetric}: symmetric functions have $c_f=2$,
and hence $c_g\le2$.
More recently, \citet{wan2025symmetric} gave a deterministic
$O(kn)$-query algorithm achieving the same $0.432$ ratio for symmetric
functions under a cardinality constraint, confirming the tractability of
this class.
Both results are specific to symmetric functions and require
non-negativity, while our framework derives the same ratio from a
single structural parameter that applies to all submodular functions,
including those that go negative.

\noindent\textbf{Partial monotonicity.}
\citet{mualem2022using} introduced the \emph{monotonicity ratio}
$m = \min_{S \subseteq T} f(T)/f(S)$ as a structural parameter for
non-monotone submodular maximization, proving that standard greedy achieves
$m(1-1/e)$ and randomized greedy achieves $m(1-1/e) + (1-m)/e$ under a
cardinality constraint.
Their framework requires $f \geq 0$ throughout.
For symmetric functions, $m = 0$ (since $f(\emptyset) = f(N) = 0$ for cut
functions), so the standard greedy guarantee is vacuous and the random greedy
guarantee reduces to~$1/e$.
In contrast, our curvature guarantee of $(1-e^{-2})/2 \approx 0.432$ for
symmetric functions strictly exceeds~$1/e$ and uses a deterministic algorithm.
More broadly, curvature and partial monotonicity capture different structural
properties: the monotonicity ratio measures how much value can be lost by
adding elements (a global property), while curvature measures the rate of
diminishing returns along the greedy trajectory (a local property).
Neither parameter dominates the other---functions with high curvature can have
moderate monotonicity ratio and vice versa---but curvature applies to
possibly-negative functions where partial monotonicity is undefined.

The experimental comparison in Section~\ref{sec:applications} confirms the
theoretical separation: the partial-monotonicity guarantee is either
vacuous ($m = 0$ for symmetric-type instances) or inapplicable ($f < 0$
at moderate-to-high cost levels), while the curvature guarantee remains
meaningful throughout.
The GCLin analysis (Figure~\ref{fig:movie-bounds}) further illustrates
this: the curvature bound stays positive in regimes where the
monotonicity-ratio bound vanishes, but the two parameters are
incomparable in general.

\noindent\textbf{Distributional curvature.}
The multilinear curvature $c_F$ is inherently distributional:
it averages the pointwise curvature inequality over product distributions
via the coupling in Theorem~\ref{thm:nn-characterization}.
One can define a \emph{distributional curvature} $c_D$ as the expected curvature
under random subset pairs drawn from such distributions.
For strictly positive $f$, the coupling proof gives $c_D = c_F = c_f$:
the distribution cannot improve on the worst case.
For functions that go negative, $c_D$ may be finite even when $c_F = \infty$,
suggesting that distributional curvature could yield tighter bounds in some regimes.

\noindent\textbf{Comparison with additive guarantees.}
The additive-ratio framework of \citet{harshaw2019submodular}---and its
extensions~\citep{kazemi2021regularized,lu2024regularized,bodek2022maximizing}---gives
$f(S) \ge (1-1/e)\,g(O^*) - \ell(O^*)$
for decomposable $f = g - \ell$.
This becomes vacuous when
$\rho = \ell(O^*)/g(O^*) \ge 1 - 1/e$.
Subsequent works~\citep{kazemi2021regularized,lu2024regularized,nikolakaki2021efficient,bodek2022maximizing}
extend distorted greedy to streaming, distributed, and stochastic settings,
all inheriting the additive guarantee structure.
The curvature framework provides a multiplicative alternative
in the offline setting; extending it to streaming, distributed,
and stochastic variants is a natural direction.

\noindent\textbf{Weak submodularity and generalizations.}
\citet{chen2018weakly} studied maximization of $\gamma$-weakly submodular
functions, which relax submodularity by a multiplicative ratio.
The HFWK distorted greedy guarantee also incorporates the weak
submodularity parameter.
Our curvature is orthogonal: it measures diminishing-returns \emph{degradation}
(how much unions destroy value) rather than the submodularity gap.
\citet{bian2017continuous} studied DR-submodular maximization in continuous
domains; our multilinear extension curvature $c_F$
(Appendix~\ref{apx:multilinear}) provides a bridge between the discrete and
continuous settings: $c_F = c_f$ when $f$ is strictly positive,
and $c_F = \infty$ when $f$ takes negative values
(Theorem~\ref{thm:nn-characterization}).

\section{Greedy with Pruning: Pseudocode}\label{apx:algorithm}

\begin{algorithm}[h]
  \DontPrintSemicolon
  \caption{$\greedy(f, k)$}\label{alg:greedy}
  \KwIn{oracle $f$, cardinality constraint $k$}
  $A_0 \gets \emptyset$\;
  \For{$i \gets 1$ to $k$}{
    $a_i \gets \argmax_{x \in \uni} \marge{x}{A_{i-1}}$ \tcp*{break ties by the fixed ground-set order}
    \If{$\marge{a_i}{A_{i-1}} \le 0$\label{line:greedy-if}}{
      $A_k \gets A_{i-1}$\; \textbf{break}\;
    }
    $A_i' \gets A_{i-1} \cup \{a_i\}$\;
    \While{$\exists\, a \in A_i'$ with $\marge{a}{A_i' \setminus \{a\}} \le 0$\label{line:greedy-while}}{
      choose the first such $a$ in the fixed ground-set order\;
      $A_i' \gets A_i' \setminus \{a\}$\;
    }
    $A_i \gets A_i'$\;
  }
  \KwRet{$A_k$}\;
\end{algorithm}

\section{Measured Continuous Greedy Proof}\label{apx:mcg}

This appendix proves the conditional Theorem~\ref{thm:mcg} (MCG guarantee)
from the supporting step lemma plus the explicit monotonicity assumption
below.  MCG~\citep{feldman2017maximizing,calinescu2011maximizing}
operates on the multilinear extension of $f$ subject to any
downward-closed solvable polytope $\polytope$ (matroid, knapsack, or
their intersection, when a separation oracle for $\polytope$ is
available); the per-step analysis below depends only on this polytope
structure and not on the specific constraint form.

\begin{lemma}[{\citet[Corollary 3.2]{feldman2017maximizing}}]\label{lem:mcg-step}
  For every $0 \le t < T$:
  $F(\mathbf{y}(t+\delta)) - F(\mathbf{y}(t))
  \ge \delta[F(\mathbf{y}(t) \lor \mathbf{1}_{O^*}) - F(\mathbf{y}(t))]
  - O(n^3\delta^2)\ff{O^*}$.
\end{lemma}

\begin{assumption}[MCG monotonicity condition]\label{assump:mcg-mono}
  $F(\mathbf{x}) \le F(\mathbf{y}(t))$ for all $0 \le t \le T$
  and $\mathbf{x} \le \mathbf{y}(t)$.
\end{assumption}
This is the auxiliary monotonicity property supplied by
\citet[Lemma~3.3]{feldman2017maximizing} in the settings covered by
their MCG analysis. We state it explicitly because strict positivity
alone does not imply it for arbitrary non-monotone submodular functions.
The pruned DMCG-P analysis does not assume this global condition: it
uses the local positive-slope invariant enforced by its pruning loop.

\noindent\textbf{Restated claim.}
Under the hypotheses of Theorem~\ref{thm:mcg}---strict positivity,
submodularity, curvature~$c_f>1$, the MCG monotonicity condition above,
and a downward-closed solvable polytope---measured continuous greedy
returns a point $\mathbf y(T)$ satisfying
\[
  F(\mathbf{y}(T)) \ge
  \frac{1-e^{-c_fT}-o(1)}{c_f}\,\ff{O^*},
\]
where $O^*$ is an optimal feasible set and $T\in[0,1]$ is the time
horizon.

\begin{proof}[Proof of Theorem~\ref{thm:mcg}]
  We solve a scalar recurrence for $a(t)\triangleq F(\mathbf y(t))$.
  From Lemma~\ref{lem:mcg-step},
  \[
    F(\mathbf{y}(t+\delta)) - F(\mathbf{y}(t))
    \ge
    \delta\bigl[F(\mathbf{y}(t) \lor \mathbf{1}_{O^*})
      - F(\mathbf{y}(t))\bigr]
    - O(n^3\delta^2)\ff{O^*}.
  \]

  Apply Definition~\ref{def:F-curv} with $\mathbf{x} = \mathbf{1}_{O^*}$
  and $\mathbf{y} = \mathbf{y}(t)$, using $c_F = c_f$
  (Theorem~\ref{thm:nn-characterization}(a)):
  \[
    F(\mathbf{y}(t) \lor \mathbf{1}_{O^*}) - F(\mathbf{1}_{O^*})
    \ge (1 - c_f)\,F(\mathbf{y}(t) - \mathbf{y}(t) \land \mathbf{1}_{O^*}).
  \]
  Since $\mathbf{y}(t) - \mathbf{y}(t) \land \mathbf{1}_{O^*} \le \mathbf{y}(t)$,
  Assumption~\ref{assump:mcg-mono} gives
  $F(\mathbf{y}(t) - \mathbf{y}(t) \land \mathbf{1}_{O^*}) \le F(\mathbf{y}(t))$.
  Rearranging:
  $F(\mathbf{y}(t) \lor \mathbf{1}_{O^*}) - F(\mathbf{y}(t))
  \ge \ff{O^*} - F(\mathbf{y}(t)) + (1-c_f)F(\mathbf{y}(t))
  = \ff{O^*} - c_f\,F(\mathbf{y}(t))$.
  Combining with the MCG step:
  \[
    a(t+\delta) \ge (1 - c_f\delta)\,a(t)
    + \delta(1 - O(n^3\delta))\ff{O^*}.
  \]
  Let $m=T/\delta$ and unroll the recurrence from $a(0)=0$:
  \[
    a(T)
    \ge \delta(1-O(n^3\delta))\ff{O^*}
       \sum_{r=0}^{m-1}(1-c_f\delta)^r.
  \]
  The geometric sum is
  \[
    \sum_{r=0}^{m-1}(1-c_f\delta)^r
    = \frac{1-(1-c_f\delta)^m}{c_f\delta}.
  \]
  Hence
  \[
    F(\mathbf y(T))
    \ge
    \frac{1-(1-c_f\delta)^{T/\delta}}{c_f}\,
    (1-O(n^3\delta))\,\ff{O^*}.
  \]
  Taking the step size small enough that $\delta\le n^{-5}$, and then
  sending $\delta\to0$, gives
  $(1-c_f\delta)^{T/\delta}\to e^{-c_fT}$ and absorbs
  $O(n^3\delta)$ into the $o(1)$ term.
\end{proof}

\section{Computing the OPT-Free Curvature Certificate}\label{apx:certificate}

A key practical advantage of the greedy curvature framework over additive
guarantees is that the algorithm can output a \emph{provable quality
certificate} alongside its solution, without access to the optimal set or
its value.  For decomposable objectives, pruning gives exactly the
right observable denominators.

\subsection{The Removal-Marginal Certificate}

Given a decomposable objective $f = g - \ell$ with $g$~monotone submodular
(CC curvature~$\alpha_g$) and $\ell$~modular, let
$A_0,A_1,\ldots,A_k$ be the active-set trajectory of greedy with pruning.
Define
\[
  \hat r
  \;\triangleq\;
  \max_i \max_{e\in A_i}
  \frac{\ell(e)}{\Delta_g(e\mid A_i\setminus\{e\})},
\]
with value~$0$ if all active sets are empty.
This is the quantity used in Proposition~\ref{prop:opt-free}.

The denominator is positive automatically.  At the end of each pruning
loop, every active element~$e\in A_i$ satisfies
\[
  \Delta_f(e\mid A_i\setminus\{e\})
  =
  \Delta_g(e\mid A_i\setminus\{e\})-\ell(e)
  >0,
\]
and hence \(\Delta_g(e\mid A_i\setminus\{e\})>\ell(e)\).  Therefore
\(\hat r<1\).  The resulting certificate is
\[
  \frac{f(A_k)}{f(\text{OPT})}
  \;\ge\; \frac{1-e^{-\bar{\hat c}_g}}{\bar{\hat c}_g},
  \qquad
  \hat c_g=\frac{\alpha_g}{1-\hat r},
  \qquad
  \bar{\hat c}_g=\max\{1,\hat c_g\}.
\]

\begin{proof}[Proof of Proposition~\ref{prop:opt-free}]
  Let $\mathcal O=\argmax_{|S|\le k}f(S)$.  Fix
  $O^*\in\mathcal O$ and a trajectory step~$i$ with
  \(f(A_i\setminus O^*)>0\), since these are exactly the positive-denominator
  pairs in Definition~\ref{def:greedy-curv}.  Write \(A=A_i\) and
  \(T=A_i\setminus O^*\).

  The first ingredient is the usual CC-curvature inequality for the
  monotone component \(g\):
  \begin{equation}\label{eq:optfree-r-gain}
    g(O^*\cup A_i)-g(O^*) \;\ge\; (1-\alpha_g)\,g(T).
  \end{equation}
  To see this directly, enumerate \(T=\{e_1,\ldots,e_m\}\) and telescope
  the gain from adding \(T\) to \(O^*\).  Each marginal of \(g\) is at
  least \(1-\alpha_g\) times the corresponding marginal when the same
  element is added inside \(T\); summing gives~\eqref{eq:optfree-r-gain}.

  The second ingredient is the set-level penalty bound supplied by the
  removal marginals.  Enumerate \(T=\{e_1,\ldots,e_m\}\) again and write
  \(T_{<q}=\{e_1,\ldots,e_{q-1}\}\).  Since \(T_{<q}\subseteq A_i\setminus\{e_q\}\),
  submodularity gives
  \[
    \Delta_g(e_q\mid T_{<q})
    \;\ge\;
    \Delta_g(e_q\mid A_i\setminus\{e_q\}).
  \]
  By definition of \(\hat r\),
  \(\ell(e_q)\le \hat r\,\Delta_g(e_q\mid A_i\setminus\{e_q\})\).
  Hence
  \begin{align}
    \ell(T)
    &= \sum_{q=1}^m \ell(e_q) \notag\\
    &\le \hat r \sum_{q=1}^m \Delta_g(e_q\mid A_i\setminus\{e_q\})
     \;\le\; \hat r \sum_{q=1}^m \Delta_g(e_q\mid T_{<q})
     \;=\; \hat r\,g(T),
     \label{eq:optfree-r-set}
  \end{align}

  Set \(\tau=\ell(T)/g(T)\).  The qualifying condition
  \(f(T)>0\) gives \(\tau<1\), and~\eqref{eq:optfree-r-set} gives
  \(0\le\tau\le\hat r\).  Combining~\eqref{eq:optfree-r-gain} with
  modularity of \(\ell\),
  \[
    \frac{f(O^*\cup A_i)-f(O^*)}{f(T)}
    \;\ge\;
    \frac{(1-\alpha_g)g(T)-\ell(T)}{g(T)-\ell(T)}
    \;=\;
    \frac{(1-\alpha_g)-\tau}{1-\tau}.
  \]
  The last expression is decreasing in \(\tau\), since its derivative is
  \(-\alpha_g/(1-\tau)^2\).  Therefore
  \[
    \frac{f(O^*\cup A_i)-f(O^*)}{f(T)}
    \;\ge\;
    \frac{(1-\alpha_g)-\hat r}{1-\hat r}
    \;=\;
    1-\frac{\alpha_g}{1-\hat r}.
  \]
  Taking the double minimum over \(O^*\in\mathcal O\) and qualifying
  trajectory steps in Definition~\ref{def:greedy-curv} gives
  \(c_g\le \alpha_g/(1-\hat r)\).
\end{proof}

\begin{remark}[Cheaper singleton diagnostic]\label{rem:singleton-diagnostic}
  A cheaper quantity is the singleton ratio
  \[
    \hat s=\max_{e\in\cup_i A_i}\frac{\ell(e)}{g(\{e\})}.
  \]
  Since \(\Delta_g(e\mid A_i\setminus\{e\})\le g(\{e\})\), the
  removal-marginal certificate is generally weaker:
  \(\hat r\ge \hat s\).  The singleton ratio can nevertheless be used as
  a formal certificate under the additional hypothesis
  \(\hat s<1-\alpha_g\), in which case the singleton argument yields
  \(c_g\le \alpha_g/(1-\hat s)\).  Without this extra
  hypothesis, singleton ratios are reported only as diagnostics in the
  experiments.
\end{remark}

\noindent\textbf{Contrast with additive guarantees.}
The HFWK guarantee $(1-1/e)\,g(\text{OPT}) - \ell(\text{OPT})$ requires
knowledge of both $g(\text{OPT})$ and $\ell(\text{OPT})$ to evaluate.
At test time on an instance where OPT is intractable, neither is available --- one can only
\emph{upper-bound} $g(\text{OPT})$ (e.g., by the greedy value of~$g$ alone),
which loosens the guarantee further.
The removal-marginal certificate depends only on $\alpha_g$ (a property of
the function class, not the instance) and $\hat r$ (observable from the
algorithm's own trajectory).

\subsection{Analytic $\alpha_g$ for Application Classes}\label{apx:analytic-alpha}

The CC curvature $\alpha_g = 1 - \min_{e,\, S\not\ni e}
\Delta_g(e \mid S) / g(\{e\})$ has closed-form expressions for the
applications considered in this paper.

\noindent\textbf{Experimental design (Bayesian A-optimality).}
$g(S) = \operatorname{tr}(\Sigma_{\text{prior}}) -
\operatorname{tr}(\Sigma_{S|})$, where $\Sigma_{S|}$ is the posterior
covariance after observing experiments in~$S$.
The CC curvature of $g$ is determined by the eigenvalues of the
information matrix:
\[
  \alpha_g \;=\; 1 - \frac{\lambda_{\min}}{\lambda_{\max}}
  \;=\; 1 - \frac{1}{\kappa},
\]
where $\kappa$ is the condition number of $X^\top X / \sigma^2 +
\Sigma_{\text{prior}}^{-1}$ and $X$ is the design matrix.
For well-conditioned designs ($\kappa \approx 2$), $\alpha_g \approx 0.5$;
for ill-conditioned designs ($\kappa \approx 20$), $\alpha_g \approx 0.95$.

\noindent\textbf{Coverage functions (directed vertex cover).}
$g(S) = |\{j : j \text{ covered by some } v \in S\}|$.
Since any element's coverage can be fully subsumed by the rest of the
ground set, $\min_{e,S} \Delta_g(e \mid S) = 0$, giving $\alpha_g = 1$.
This recovers the standard $1 - 1/e$ guarantee as a special case of the
curvature framework.

\noindent\textbf{Feature selection (Gaussian mutual information).}
$g(S) = \tfrac{1}{2}\log\det(I + \Sigma_{SS}/\sigma^2)$, where $\Sigma_{SS}$
is the covariance submatrix for features in~$S$.
The CC curvature depends on the spectral structure of the full covariance:
\[
  \alpha_g \;=\; 1 - \frac{1}{\kappa(I + \Sigma/\sigma^2)},
\]
where $\kappa(\cdot)$ denotes the condition number.
When features are nearly independent, $\alpha_g \approx 0$;
with high redundancy (correlated groups), $\alpha_g$ approaches~$1$.

\subsection{Computing $\hat r$ from the Greedy Trajectory}

The certificate \(\hat r\) is computed from the pruned active sets
\(A_i\), not from the terminal output alone.  For every element that is
active at a step~\(i\), evaluate the removal marginal
\(\Delta_g(e\mid A_i\setminus\{e\})\) and form
\[
  \frac{\ell(e)}{\Delta_g(e\mid A_i\setminus\{e\})}.
\]
The maximum over these ratios is \(\hat r\).  For decomposable objectives
where \(g\) and \(\ell\) are given separately, this is a direct
post-processing pass over the trajectory.  The same pass also verifies
\(\hat r<1\), although the pruning invariant already proves it.

For the \emph{non-decomposable} case (where only a value oracle for $f$ is
available), $\hat r$ cannot be computed directly.  However, the empirical
greedy curvature $c_g$ can still be bounded without decomposition --- see
the per-step curvature analysis in Appendix~\ref{apx:multilinear}.

\subsection{Relation to the Singleton Diagnostic}

Table~\ref{tab:tier1} reports the guarantee
$(1-e^{-\bar{\hat c}_g})/\bar{\hat c}_g$ induced by the
removal-marginal certificate (Proposition~\ref{prop:opt-free}),
which is the primary OPT-free bound for decomposable greedy with pruning.
A cheaper singleton diagnostic \(\hat s = \max_{e\in\mathcal{A}}
\ell(e)/g(\{e\})\) is also computable and becomes a formal certificate
under the additional condition in
Remark~\ref{rem:singleton-diagnostic}.

\section{Experimental Setup}\label{apx:exp-setup}
%% ================================================================

This appendix provides full reproducibility details for the experiments
in Section~\ref{sec:experiments}.
We organize results into two tiers:
\emph{Tier~1} uses small instances ($n=20$, $k=5$) where exact OPT
can be computed by brute-force enumeration, enabling ground-truth
approximation ratios;
\emph{Moderate-scale} experiments use instances with $n$ up to $300$, where OPT is
intractable and only trajectory-based diagnostics are reported.

\noindent\textbf{Tier 1 instances ($n=20$, exact OPT).}
Each configuration uses $k=5$ and $10$ random seeds (seeds $0$--$9$).
Exact OPT is computed by enumerating all subsets of size at most~$k$.

\emph{Experimental design (Bayesian A-optimality).}
A $20 \times 5$ design matrix~$X$ is generated with prescribed
condition number $\kappa=5$ via SVD: $X = U \diag(\sigma_1,\ldots,\sigma_5) V^\top$
where $U,V$ are random orthogonal and $\sigma_i$ are linearly spaced
from~$1$ to~$\sqrt{5}$.
Prior: $\Sigma_{\mathrm{prior}} = I_5$; noise variance $\sigma^2=1$.
$g(S)=\tr(\Sigma_{\mathrm{prior}})-\tr(\Sigma_{S|})$ is the
variance reduction.
Costs: $\ell(e) = c_s \cdot \|x_e\|_2/\overline{\|x\|}$
where $c_s$ is the cost scale and $\overline{\|x\|}$ is the mean row
norm.
$\alpha_g = 1 - 1/\kappa = 0.667$.
Cost scales: $c_s \in \{0, 0.03, 0.06, 0.10, 0.15, 0.20, 0.28\}$.

\emph{Coverage with costs.}
Random bipartite graph: $n=20$ vertices, $m=40$ items,
each (vertex, item) edge included independently with probability~$0.2$.
$g(S)$ = number of items covered by at least one vertex in~$S$.
Costs: $\ell(v) = c_s \cdot \deg(v)/\overline{\deg}$
(degree-normalized).
$\alpha_g = 1$.
Cost scales: $c_s \in \{0, 0.5, 1.0, 2.0, 3.5, 5.0, 8.0\}$.

\emph{Feature selection (mutual information minus $\ell_1$ penalty).}
$p=20$ features in $4$ correlated groups with pairwise
within-group correlation $0.7$; small cross-group noise
($0.05 \cdot Z$ for random $Z$, symmetrized and shifted to ensure
positive definiteness).
$g(S) = \tfrac{1}{2}\log\det(I + \Sigma_{S,S}/\sigma^2)$,
$\sigma^2 = 1$.
Costs: $\ell(e)=c_s \cdot (1 + 0.3\,\epsilon_e)$, $\epsilon_e$
standard normal, clipped to $\ge 0.01$.
$\alpha_g \approx 0.79$.
Cost scales: $c_s \in \{0, 0.05, 0.1, 0.2, 0.3, 0.5, 0.8\}$.

\noindent\textbf{Moderate-scale instances ($n$ up to $300$, trajectory diagnostics).}
Each configuration uses 5 seeds and 7 cost levels.
OPT is intractable; the table reports the singleton curvature diagnostic
(removal marginals were not recorded for these runs).

\emph{Experimental design}: $n = 200$, $k = 20$, $d = 10$,
$\kappa = 5$ ($\alpha_g = 1 - 1/\kappa$).
Cost scales: $c_s \in \{0, 0.02, 0.05, 0.10, 0.15, 0.22, 0.30\}$.
\emph{Coverage}: $n = 300$, $k = 30$, $m = 600$ items, edge
probability $0.05$ ($\alpha_g = 1$).
Cost scales: $c_s \in \{0, 0.5, 1.5, 3.0, 5.0, 8.0, 12.0\}$.
\emph{Feature selection}: $n = 100$, $k = 15$, $10$ correlated groups.
Cost scales: $c_s \in \{0, 0.03, 0.08, 0.15, 0.25, 0.4, 0.6\}$.

\noindent\textbf{$\lambda$-sweep (LLM summarization).}
Dataset: Multi-News~\citep{fabbri2019multinews}, a multi-document
summarization benchmark.
For each of 100 validation examples, source documents are chunked into
150-word passages padded with random distractors to reach
$n = 200$, then embedded with TF-IDF vectors; pairwise cosine
similarities define the GCLin objective (random seed~$42$).
Greedy with best-prefix selection (the prefix $A_i$ maximizing $f(A_i)$
is returned, appropriate for non-monotone objectives) selects $k = 10$
passages under $\text{GCLin}_\lambda$ for
$\lambda \in \{0.1, 0.25, 0.5, 0.75, 1.0, 1.5\}$
($2\lambda$ is the uniform-query reference bound for the five values with
$\lambda\le1$; $\lambda=1.5$ is included as an empirical
high-redundancy diagnostic).
A local \texttt{openai/gpt-oss-120b} endpoint (prompt:
``Summarize the following passages,'' temperature~$0.3$,
max tokens~$1024$) summarizes the selected passages, and
ROUGE-2 is computed against the reference summary.

\noindent\textbf{Algorithms.}
Tier~1 and MaxCut experiments use greedy+pruning (Algorithm~\ref{alg:greedy})
with lazy (priority-queue) evaluation;
the $\lambda$-sweep uses greedy with best-prefix selection.
The comparison baseline is distorted greedy~\citep{harshaw2019submodular}.
The removal-marginal certificate $\hat r$ (Proposition~\ref{prop:opt-free})
is computed by a post-processing pass over the greedy trajectory,
evaluating $g(A_i)$ and $g(A_i\setminus\{e\})$ for each $e\in A_i$.
Code and a CC BY 4.0 paper license will be released with the
camera-ready version.

%% ================================================================
\section{Additional Experimental Results}\label{apx:exp-additional}
%% ================================================================

This appendix contains figures and tables deferred from the main text.

\subsection{Tier 1: Cost-Sweep Curves}

Figure~\ref{fig:cost-sweep} shows the full cost-sweep curves for
experimental design and coverage.
GP is competitive with DG and often higher at larger costs, while the
curvature guarantee stays positive as the HFWK bound goes negative.

\begin{figure}[t]
  \centering
  \begin{minipage}[t]{0.48\textwidth}
    \centering
    \includegraphics[width=\linewidth]{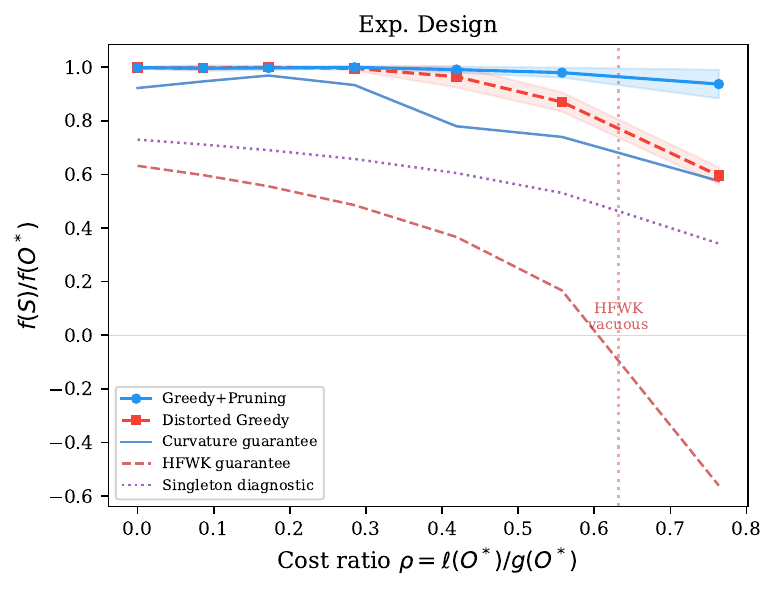}
  \end{minipage}\hfill
  \begin{minipage}[t]{0.48\textwidth}
    \centering
    \includegraphics[width=\linewidth]{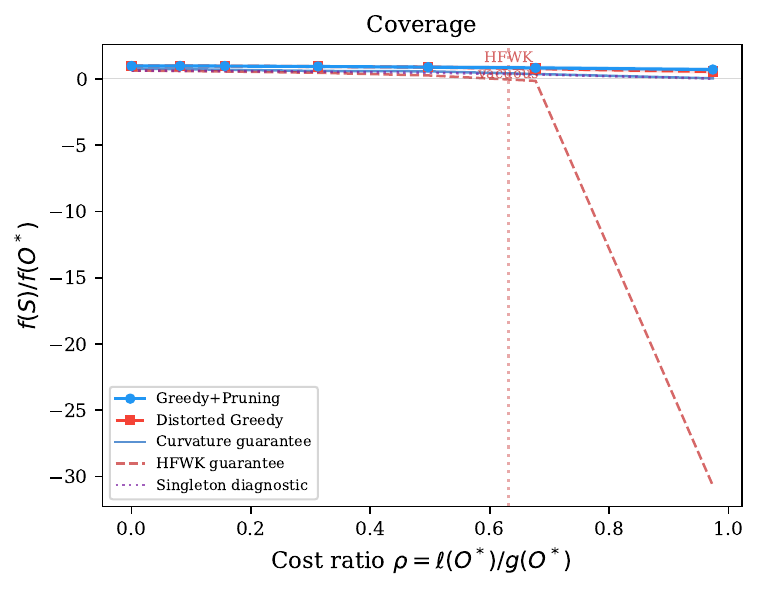}
  \end{minipage}
  \caption{Tier 1 cost sweeps ($n=20$, exact OPT).
    Solid lines: empirical ratios; dashed: theoretical guarantees.
    GP (blue) is competitive with DG (red) and often higher at larger costs.
    The curvature guarantee (dark blue) stays positive while the HFWK bound
    (dark red) goes negative at high costs.}
  \label{fig:cost-sweep}
\end{figure}

\subsection{Tier 1: Curvature vs. Observed Ratios}

Figure~\ref{fig:tightness} plots the empirical greedy curvature~$c_g$
(computed exactly using OPT) against the achieved $f(S)/f(\text{OPT})$
ratio.
The curvature guarantee $(1-e^{-\bar c_g})/\bar c_g$ follows the empirical
trend in the non-monotone range (and the classical formula applies in the
monotone small-curvature range).
The removal-marginal certificate (reported in Table~\ref{tab:tier1}) provides
a formal guarantee from the same trajectories.

\begin{figure}[t]
  \centering
  \begin{minipage}[t]{0.48\textwidth}
    \centering
    \includegraphics[width=\linewidth]{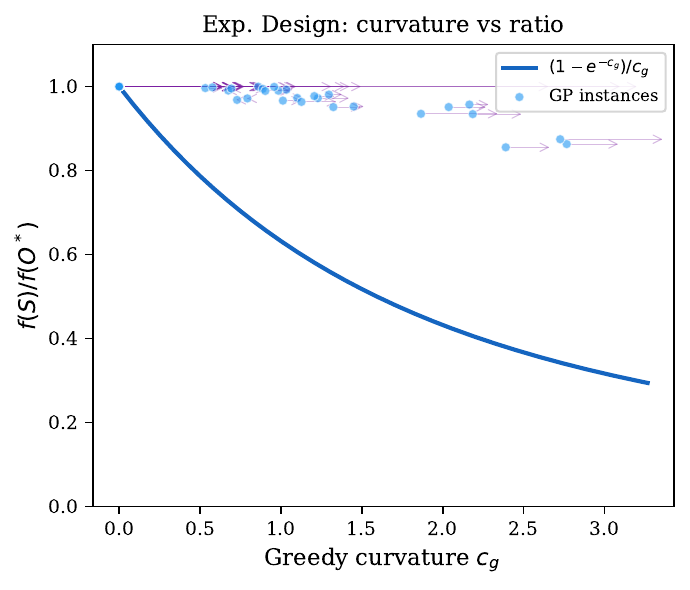}
  \end{minipage}\hfill
  \begin{minipage}[t]{0.48\textwidth}
    \centering
    \includegraphics[width=\linewidth]{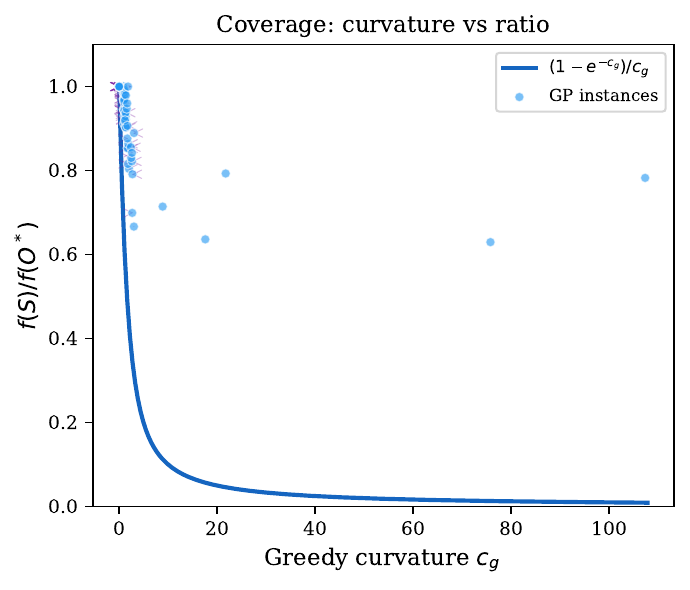}
  \end{minipage}
  \caption{Curvature vs.\ observed ratios ($n=20$, exact OPT).
    Each dot is a (curvature, ratio) pair from one instance.
    The curve is the theoretical guarantee, using the classical formula in
    the monotone small-curvature regime and $(1-e^{-\bar c_g})/\bar c_g$
    otherwise.
    GP ratios lie well above the curve.
    Arrows show the gap from the singleton diagnostic $\hat{c}_g$ to the empirical~$c_g$;
    Proposition~\ref{prop:opt-free} gives the formal removal-marginal certificate.}
  \label{fig:tightness}
\end{figure}

\subsection{Moderate-Scale Instances (No Exact OPT)}

\begin{table}[t]
\centering
\caption{Moderate-scale results (no exact OPT; singleton diagnostics).
  $f_{\text{GP}}$ and $f_{\text{DG}}$ report mean $f(S)$ over 5 seeds.
  Diag.\ is the multiplicative value of the singleton formula, included
  as a diagnostic rather than the formal removal-marginal certificate.
  HFWK uses heuristic best-known in place of OPT.
  \textbf{Bold} marks vacuous (negative) HFWK guarantees.}
\label{tab:tier2}
\small
\begin{tabular}{@{}llccccc@{}}
\toprule
Application & Cost & $f_{\text{GP}}$ & $f_{\text{DG}}$ & $\hat{c}_g$ & Diag. & HFWK \\
\midrule
\multirow{4}{*}{\shortstack[l]{Exp.\ Design\\($n{=}200$, $k{=}20$)}}
  & $0.02$ & $2.56$ & $2.56$ & $0.79$ & $0.69$ & $0.55$ \\
  & $0.05$ & $1.72$ & $1.72$ & $1.11$ & $0.60$ & $0.34$ \\
  & $0.10$ & $0.53$ & $0.52$ & $3.42$ & $0.29$ & \textbf{$-0.42$} \\
  & $0.15$ & $0.05$ & $0.04$ & $44.3$ & $0.05$ & \textbf{$-5.15$} \\
\midrule
\multirow{4}{*}{\shortstack[l]{Coverage\\($n{=}300$, $k{=}30$)}}
  & $1.5$  & $517$ & $517$ & $1.05$ & $0.62$ & $0.60$ \\
  & $5.0$  & $397$ & $400$ & $1.20$ & $0.58$ & $0.48$ \\
  & $8.0$  & $309$ & $312$ & $1.36$ & $0.55$ & $0.37$ \\
  & $12.0$ & $214$ & $218$ & $1.66$ & $0.49$ & $0.20$ \\
\midrule
\multirow{4}{*}{\shortstack[l]{Feature Sel.\\($n{=}100$, $k{=}15$)}}
  & $0.08$ & $5.63$ & $5.62$ & $1.06$ & $0.62$ & $0.58$ \\
  & $0.25$ & $4.14$ & $4.14$ & $1.57$ & $0.51$ & $0.44$ \\
  & $0.40$ & $2.88$ & $2.82$ & $2.85$ & $0.34$ & $0.20$ \\
  & $0.60$ & $1.32$ & $1.23$ & $55.0$ & $0.05$ & \textbf{$-0.40$} \\
\bottomrule
\end{tabular}
\end{table}

\begin{figure}[t]
  \centering
  \begin{minipage}[t]{0.48\textwidth}
    \centering
    \includegraphics[width=\linewidth]{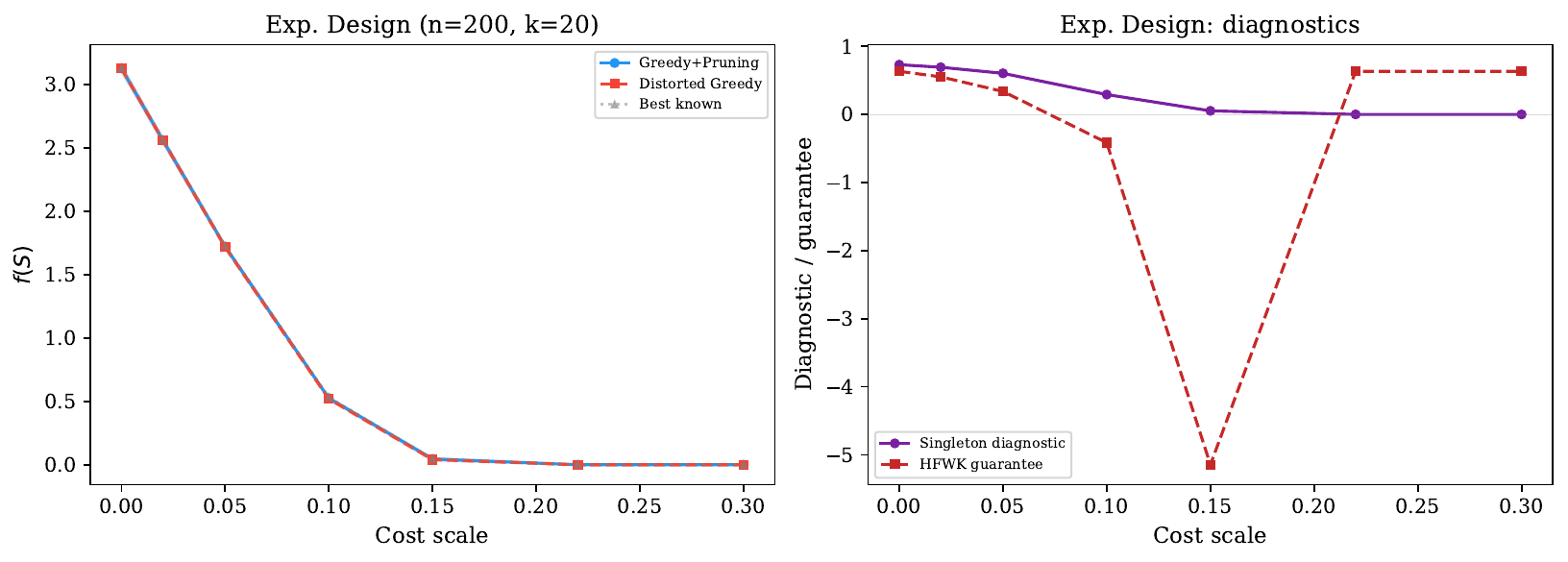}
  \end{minipage}\hfill
  \begin{minipage}[t]{0.48\textwidth}
    \centering
    \includegraphics[width=\linewidth]{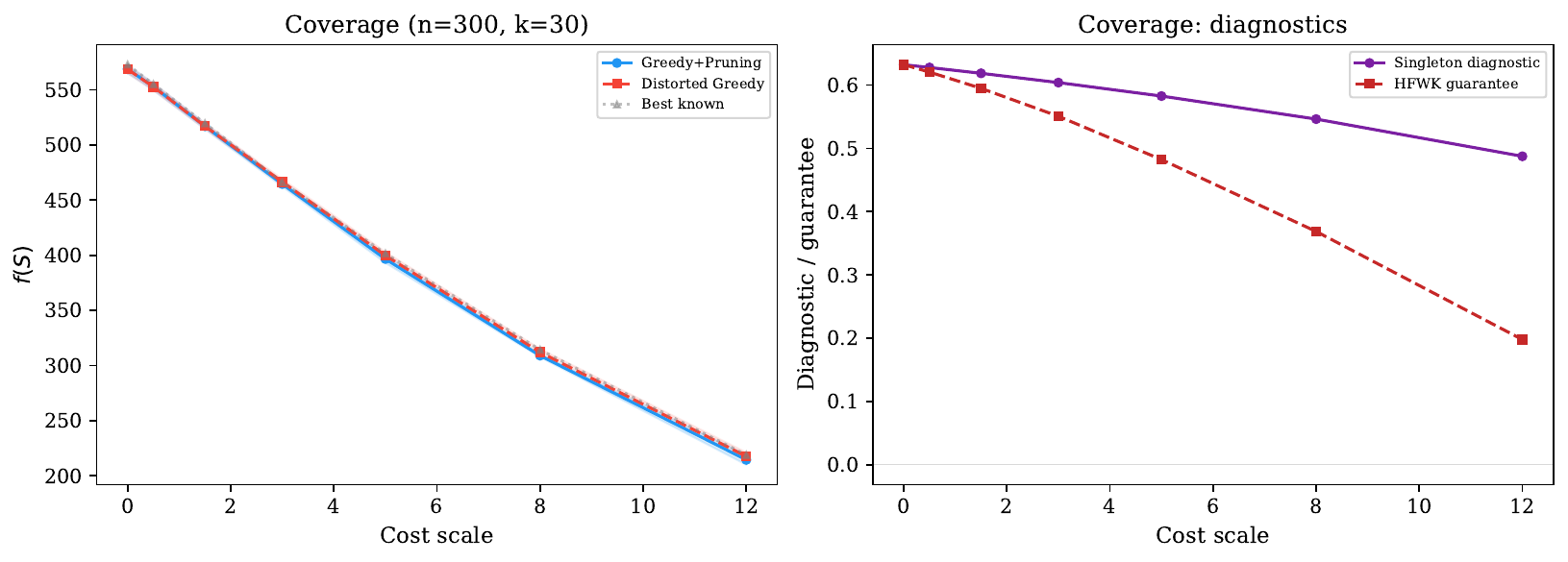}
  \end{minipage}
  \caption{Moderate-scale experiments (no exact OPT).
    Left panels: $f(S)$ values; GP (blue) matches or exceeds DG (red).
    Right panels: singleton curvature diagnostics (purple) vs HFWK (dark red).
    Diagnostics remain positive across all cost levels tested
    (removal marginals were not recorded for these runs).}
  \label{fig:tier2}
\end{figure}

\noindent\textbf{Results.}
GP and DG achieve similar solution quality on most instances,
with GP producing slightly better values at high costs
(e.g., $f_{\text{GP}} = 2.88$ vs $f_{\text{DG}} = 2.82$ for feature
selection at cost $0.40$).
The singleton curvature diagnostic stays positive at every cost level.
The HFWK bound becomes vacuous on experimental design at cost $0.10$
($\rho \ge 0.63$) and on feature selection at cost $0.60$.
The displayed values are diagnostics rather than certified guarantees;
the formal certificate in Appendix~\ref{apx:certificate} requires
removal-marginal post-processing along the trajectory.

%% ================================================================

%% ================================================================

\subsection{Curvature vs.\ Additive Guarantee Comparison}\label{sec:comparison}

\begin{table}[t]
\centering\small
\caption{Curvature vs.\ HFWK additive guarantee (as multiples of $f(O^*)$).
  HFWK uses $1-1/(e(1-\rho))$; curvature uses
  $c_g \le \alpha_g/(1-r)$ (Proposition~\ref{prop:greedy-curv-decomp}).}
\label{tab:comparison}
\begin{tabular}{@{}lccrr@{}}
\toprule
Application & $\alpha_g$ & $\rho$ & HFWK & Curvature \\
\midrule
Exp.\ design (low) & $1$ & $0.1$ & $0.591$ & $\ge 0.604$ \\
Exp.\ design (med) & $1$ & $0.4$ & $0.387$ & $\ge 0.487$ \\
Exp.\ design (high) & $1$ & $0.6$ & $0.080$ & $\ge 0.367$ \\
Feature sel.\ ($\alpha_g\!=\!0.5$) & $0.5$ & $0.3$ & $0.474$ & $\ge 0.715$ \\
Feature sel.\ ($\alpha_g\!=\!0.5$) & $0.5$ & $0.5$ & $0.264$ & $\ge 0.632$ \\
Vertex cover & $1$ & $0.5$ & $0.264$ & $\ge 0.432$ \\
\bottomrule
\end{tabular}
\end{table}

Table~\ref{tab:comparison} compares the curvature and HFWK frameworks
across applications from the additive-ratio literature.
The curvature guarantee dominates whenever the cost ratio $\rho$
exceeds approximately $0.15$ (for coverage-type $g$ with $\alpha_g = 1$).

\subsection{MaxCut: Symmetric Submodular Functions}\label{apx:maxcut}

We test whether greedy+pruning (Algorithm~\ref{alg:greedy}) improves
on standard greedy for cardinality-constrained MaxCut,
a canonical symmetric submodular function with $c_f = 2$.
By Proposition~\ref{prop:symmetric} and Theorem~\ref{thm:greedy},
the guarantee is $(1-e^{-2})/2 \approx 0.432$.

\noindent\textbf{Objective.}
Given an undirected graph $G = (V, E)$ with $|V| = n$, the MaxCut
objective is $f(S) = |\delta(S)| = |\{(u,v) \in E : u \in S, v \notin S\}|$,
subject to $|S| \le k$.

\noindent\textbf{Instance families.}
\emph{Planted distractor (stochastic block model).}
Three communities of sizes $|A| = k$, $|B| = k$, $|C| = n - 2k$.
Edge probabilities: $p_{AA} = 0.8$, $p_{BB} = 0.2$, $p_{CC} = 0.3$,
$p_{AB} = 0.3$, $p_{AC} = p_{BC} = 0.4$.
Community $A$ is a dense ``bait'' cluster whose high-degree vertices
attract greedy early, but whose intra-cluster edges become cut-reducing
once many $A$-vertices are selected.
\emph{Erd\H{o}s--R\'{e}nyi.}
Homogeneous baseline: each edge present independently with
probability $p = 0.3$.

\noindent\textbf{Configurations.}
$n \in \{16, 20\}$, $k \in \{\lfloor n/4 \rfloor,
\lfloor n/3 \rfloor, \lfloor n/2 \rfloor\}$,
giving 12 configurations (6 planted distractor, 6 Erd\H{o}s--R\'{e}nyi).
Each configuration uses 20 random seeds.
Table~\ref{tab:maxcut} reports the 8 configurations where pruning
activates on at least one seed; the 4 remaining Erd\H{o}s--R\'{e}nyi
configurations ($k \le n/3$) are omitted because GP and greedy
coincide on homogeneous graphs at these budgets.

\noindent\textbf{Algorithms.}
Standard greedy, greedy+pruning (Algorithm~\ref{alg:greedy}),
and random greedy~\citep{BFNS2014}.

\noindent\textbf{OPT.}
Exact enumeration over all subsets of size at most~$k$.

\noindent\textbf{Reproducibility.}
Graph generation is deterministic: \texttt{numpy.RandomState}
with seed $\mathtt{seed} \times 1000 + n \times 100 + k$.
Code: \texttt{experiments/objectives/maxcut.py} (objective and graph
generators) and \texttt{experiments/run\_maxcut.py} (runner).
Full results: \texttt{experiments/results/maxcut\_results.json} (720 records).

\noindent\textbf{Results.}
Table~\ref{tab:maxcut} reports mean approximation ratios (relative to
exact OPT) and the number of instances where greedy+pruning (GP)
strictly beats standard greedy.

\begin{table}[h]
\centering\small
\caption{MaxCut experiment: mean $f(S)/f(O^*)$ over 20 seeds.
  ``GP $>$ Greedy'' counts instances where pruning strictly improves
  on standard greedy.
  Theory guarantee: $(1-e^{-2})/2 \approx 0.432$.}
\label{tab:maxcut}
\begin{tabular}{@{}lrrcccc@{}}
\toprule
Instance type & $n$ & $k$ & Greedy & GP & Random & GP $>$ Greedy \\
\midrule
Planted distractor & 16 & 4  & 0.979 & 0.979 & 0.947 & 0/20 \\
Planted distractor & 16 & 5  & 0.974 & 0.974 & 0.956 & 0/20 \\
Planted distractor & 16 & 8  & 0.935 & 0.966 & 0.911 & 13/20 \\
Planted distractor & 20 & 5  & 0.980 & 0.980 & 0.957 & 0/20 \\
Planted distractor & 20 & 7  & 0.975 & 0.975 & 0.937 & 0/20 \\
Planted distractor & 20 & 10 & 0.956 & 0.975 & 0.915 & 12/20 \\
\midrule
Erd\H{o}s--R\'{e}nyi & 16 & 8  & 0.971 & 0.986 & 0.895 & 6/20 \\
Erd\H{o}s--R\'{e}nyi & 20 & 10 & 0.966 & 0.981 & 0.922 & 9/20 \\
\bottomrule
\end{tabular}
\end{table}

At small budgets ($k \le n/3$), pruning never activates:
greedy does not over-commit, so GP and greedy coincide.
At $k \approx n/2$, greedy selects vertices whose neighbors
increasingly fall inside~$S$, reducing cut value.
Pruning detects these ($f(v \mid S \setminus \{v\}) \le 0$) and
ejects them, freeing slots for better choices.
On planted-distractor graphs the effect is strongest
(GP wins 13/20 at $n{=}16$, $k{=}8$; 12/20 at $n{=}20$, $k{=}10$),
because the dense bait cluster~$A$ traps greedy more reliably.
Erd\H{o}s--R\'{e}nyi graphs show a weaker but still present effect
(6--9 wins out of 20).
Overall, GP strictly improves on greedy in 40 of 240 instances (16.7\%),
concentrated in the $k = n/2$ regime.
All observed ratios far exceed the $(1-e^{-2})/2 \approx 0.432$
guarantee, confirming that the bound is conservative while the pruning
mechanism produces measurable empirical gains.

\section{Multilinear Extension, DMCG-P, and MCG}\label{apx:multilinear}
%% ================================================================

This appendix contains the multilinear extension analysis, the DMCG-P
algorithm and its guarantees, and a conditional MCG comparison for
strictly positive functions satisfying the standard MCG monotonicity
lemma.  These results extend the discrete greedy guarantee of
Theorem~\ref{thm:greedy} from cardinality constraints to integral
feasible-set hulls for DMCG-P, and to the standard continuous setting
for MCG.

\subsection{Curvature of the Multilinear Extension}\label{sec:multilinear}

We now lift curvature to the continuous relaxation.
The multilinear extension is the standard bridge between discrete
submodular optimization and continuous relaxations---it underlies
the measured continuous greedy algorithm (Section~\ref{sec:mcg}) and
the correlation gap framework~\citep{calinescu2011maximizing}.
The question is whether curvature lifts cleanly to this
continuous setting.
Classical non-monotone algorithms (e.g., BFNS randomized greedy~\citep{BFNS2014},
measured continuous greedy) analyze the multilinear extension and
rely on the correlation gap: each step of the continuous relaxation
bounds $F(\mathbf{x} \lor \mathbf{y}) - F(\mathbf{x})$ from below
using $F(\mathbf{y} - \mathbf{x} \land \mathbf{y})$.  In the
coupled set-level view, one wants a lower bound on
$\ff{R_1 \cup R_2} - \ff{R_1}$ in terms of $\ff{R_2}$ for disjoint
realizations~$R_1,R_2$.  Zero-valued sets are delicate:
if $\ff{R_2}=0$, submodularity only gives the upper bound
$\ff{R_1\cup R_2}-\ff{R_1}\le0$, not the lower bound needed for a
positive-denominator curvature comparison.
Our curvature framework makes this precise:
strict positivity of~$f$ turns out to be the key structural property.

For a set function $f: 2^\uni \to \reals$, its \emph{multilinear extension}
$F: [0,1]^\uni \to \reals$ is
\[
  F(\mathbf{x}) = \ex{\ff{\mathcal{R}(\mathbf{x})}}
  = \sum_{R \subseteq \uni} \ff{R} \prod_{i \in R} x_i \prod_{j \notin R}(1 - x_j),
\]
where $\mathcal{R}(\mathbf{x})$ includes each $i$ independently with probability $x_i$.

\begin{definition}[Multilinear curvature]\label{def:F-curv}
  The \emph{curvature} $c_F$ of the multilinear extension is
  \[
    c_F = 1 - \min_{\substack{\mathbf{x}, \mathbf{y} \in [0,1]^\uni\\
      F(\mathbf{y} - \mathbf{x} \land \mathbf{y}) > 0}}
    \frac{F(\mathbf{x} \lor \mathbf{y}) - F(\mathbf{x})}
         {F(\mathbf{y} - \mathbf{x} \land \mathbf{y})}.
  \]
  As in Definition~\ref{def:curv}, the minimum ranges only over
  positive-denominator pairs, and we set $c_F=\infty$ if the admissible
  ratios are unbounded below.
\end{definition}

The multilinear curvature is inherently \emph{distributional}: it averages the
pointwise curvature inequality over coupled product distributions.
The coupling proof below makes this explicit.

\begin{proposition}[Discrete curvature bounds multilinear curvature]\label{prop:c-leq-cF}
  $c_f \le c_F$.
\end{proposition}
\begin{proof}
  The discrete vertices $\{0,1\}^\uni \subset [0,1]^\uni$, and every
  positive-denominator pair in Definition~\ref{def:curv} is an
  admissible pair in Definition~\ref{def:F-curv}, so the minimum
  defining $c_F$ ranges over a larger domain.
\end{proof}

\begin{proposition}[$c_F = \infty$ for negative functions]\label{prop:cF-infty}
  Suppose $f$ is normalized submodular and $\ff{e}>0$ for every
  singleton $e\in\uni$.  If $f$ takes a negative value, then
  $c_F = \infty$.
\end{proposition}
Intuitively, the denominator $F(\mathbf{z}_\varepsilon)$ in the
curvature ratio can be driven to zero by concentrating mass on an
inclusion-minimal negative set, while the numerator remains bounded
away from zero by that set's negative value. This concentration
structure motivates the pruning invariant introduced in
\S\ref{sec:dmcgp}.
\begin{proof}[Proof idea]
  Let $\uni'$ be an inclusion-minimal set with $\ff{\uni'} < 0$; minimality
  gives $\ff{S} \ge 0$ for all $S \subsetneq \uni'$.
  Construct $\mathbf{z}_\varepsilon =
  \varepsilon\,\mathbf{1}_{\uni' \setminus \{e'\}} + \mathbf{1}_{\{e'\}}$
  so that the denominator $F(\mathbf{z}_\varepsilon)$ interpolates continuously
  from $\ff{e'} > 0$ (by singleton positivity) to $\ff{\uni'} < 0$,
  crossing zero at some $\varepsilon_0$.
  The numerator stays at most $\ff{\uni'} < 0$ throughout,
  so the ratio diverges as $\varepsilon \to \varepsilon_0^-$.
  See Appendix~\ref{apx:deferred} for the complete construction.
\end{proof}

The following is the key structural result connecting pointwise and distributional curvature.

\begin{theorem}[Strict positivity characterization]\label{thm:nn-characterization}
  Let $f: 2^\uni \to \reals$ be normalized submodular with
  $\ff{e}>0$ for every singleton $e\in\uni$.
  \textup{(a)}~If $f$ is strictly positive ($\ff{S} > 0$ for every nonempty~$S$),
  then $c_f = c_F$.
  \textup{(b)}~If $f$ takes a negative value, then $c_F = \infty$.
\end{theorem}
\noindent\textbf{Proof idea.}
Part~(b) is Proposition~\ref{prop:cF-infty}.
For part~(a), we construct a product-distribution coupling
$(R_1, R_2)$ with disjoint supports matching the multilinear
marginals.
Strict positivity ensures $\ff{R_2} > 0$ whenever $R_2 \ne \emptyset$,
so the discrete curvature inequality applies realization-by-realization
and survives expectation: $c_F \le c_f$.
See Appendix~\ref{apx:deferred} for the full coupling construction.

\begin{remark}[Strict positivity and the continuous relaxation]\label{rem:nn-continuous}
  Theorem~\ref{thm:nn-characterization} identifies strict positivity as a
  sufficient condition under which pointwise curvature (a combinatorial,
  worst-case property) lifts perfectly to the multilinear extension
  (a continuous, averaged property).
  When $f$ is strictly positive, $c_F = c_f$; MCG gives the stated
  global-curvature comparison when its trajectory monotonicity lemma
  also applies.
  When $f$ takes negative values, $c_F = \infty$; the curvature
  guarantees in this paper therefore use trajectory-restricted
  quantities, either for discrete greedy with pruning
  (\S\ref{sec:greedy}) or for its multilinear counterpart DMCG-P
  (\S\ref{sec:dmcgp}).
  The gap between strict positivity and non-negativity is genuine:
  non-negative functions with $\ff{S} = 0$ for some nonempty~$S$
  can have $c_F > c_f$ (see Section~\ref{sec:discussion}).
\end{remark}

\subsection{Fractional Greedy with Pruning (DMCG-P)}\label{sec:dmcgp}

The discrete algorithm of the previous subsection is tied to
cardinality constraints; integral feasible-set hull constraints beyond
cardinality can be handled by multilinear relaxations. But
Theorem~\ref{thm:nn-characterization} showed that $c_F = +\infty$
as soon as $f$ takes a negative value, so the classical multilinear
analysis---which routes through $c_F$---breaks down.
We resolve this by mirroring the discrete construction: define a
\emph{trajectory-restricted} fractional curvature $c_g^F$, and prune
to keep only positive-slope active coordinates.
The key analytical ingredients are a \emph{slope invariant}
showing that pruning preserves the trajectory structure needed in the
descent proof and an \emph{elementwise multilinear
Conforti--Cornu\'ejols inequality} lifting the discrete curvature
bound to the continuous setting.

The resulting algorithm---a pruning variant of the standard
fixed-step discretization of measured continuous
greedy~\citep{feldman2017maximizing}---we call DMCG-P (Discretized
Measured Continuous Greedy with Pruning). The algorithm itself is a
combination of known ingredients; the contribution is in the
\emph{analysis}: the structural observation that pruning keeps the
trajectory away from the concentration witnesses of
Proposition~\ref{prop:cF-infty}, and the resulting
$(1-e^{-\bar c_g^F})/\bar c_g^F$ guarantee for arbitrary submodular~$f$ with
finite $c_g^F$. Full proofs appear in Appendix~\ref{apx:dmcgp}.

\subsubsection{The DMCG-P Algorithm}\label{sec:dmcgp-alg}

Let $\mathcal I\subseteq 2^\uni$ be a downward-closed family of
feasible sets, let
$\polytope=\conv\{\mathbf{1}_S:S\in\mathcal I\}$, and assume exact
linear optimization over $\mathcal I$ (equivalently over
$\polytope$).  For submodular
$f$, write $F: [0,1]^\uni \to \reals$ for the multilinear extension
and $\partial_j F(\mathbf{x}) = F(\mathbf{x} \mid x_j = 1) - F(\mathbf{x} \mid x_j = 0)$
for the multilinear slope; by multilinearity, $\partial_j F$ is
independent of $x_j$.

\begin{algorithm}[h]
  \DontPrintSemicolon
  \caption{$\dmcgp(F, \polytope, T)$ --- Discretized Measured Continuous Greedy with Pruning}
  \label{alg:dmcgp-main}
  \KwIn{multilinear oracle $F$ of submodular $f$,
    $\polytope=\conv\{\mathbf{1}_S:S\in\mathcal I\}$ for a downward-closed
    family $\mathcal I$, step count $T \in \mathbb{N}$}
  $\tilde S_0 \gets \mathbf{0}$\;
  \For{$i \gets 0$ \textbf{to} $T-1$}{
    $B_i \gets \arg\max_{B\in\mathcal I}\sum_{j\in B}\partial_jF(\tilde S_i)$\;\label{line:dmcgp-select}\tcp*[r]{exact linear optimization}
    $v_i \gets \mathbf{1}_{B_i}$\;
    $S_{i+1} \gets \tilde S_i + \tfrac{1}{T}\,v_i$\;\label{line:dmcgp-update}
    $\tilde S_{i+1} \gets S_{i+1}$\;
    \While{$\exists\, j$ with $(\tilde S_{i+1})_j > 0$ and $\partial_j F(\tilde S_{i+1}) \le 0$\label{line:dmcgp-prune}}{
      $(\tilde S_{i+1})_j \gets 0$\;
    }
  }
  \KwRet{$\tilde S_T$}\;
\end{algorithm}

\noindent\textit{Integral selection.}
Because $\polytope$ is the convex hull of feasible incidence vectors,
the linear optimization step can be written directly over
$B_i\in\mathcal I$, and the selected direction is the integral vector
$v_i=\mathbf{1}_{B_i}$ used in the analysis.

The pruning loop supplies the slope invariant used below: every
surviving coordinate of $\tilde S_i$ has strictly positive multilinear
slope $\partial_j F(\tilde S_i) > 0$.  This keeps the trajectory away
from the specific concentration witnesses that drive $c_F=\infty$ in
the proof of Proposition~\ref{prop:cF-infty}, but finiteness of the
trajectory curvature is still stated as an explicit hypothesis.

\begin{definition}[Fractional greedy curvature]\label{def:cgF}
  Let $\mathcal{O} = \argmax_{S\in\mathcal I}\ff{S}$ and let
  $\tilde S_0, \dots, \tilde S_{T-1}$ be the iterates of
  Algorithm~\ref{alg:dmcgp-main}.  The \emph{fractional greedy curvature}
  along the trajectory is
  \[
    c_g^F \;=\; 1 \;-\; \min_{O^* \in \mathcal{O}}\;
      \min_{\substack{0 \le i < T \\ F(\tilde S_i - \tilde S_i \land \mathbf{1}_{O^*}) > 0}}\;
      \frac{F(\tilde S_i \lor \mathbf{1}_{O^*}) - F(\mathbf{1}_{O^*})}
           {F(\tilde S_i - \tilde S_i \land \mathbf{1}_{O^*})}.
  \]
\end{definition}

Definition~\ref{def:cgF} is the exact multilinear analogue of the
discrete greedy curvature $c_g$
(Definition~\ref{def:greedy-curv}): the minimization is over
iterates, not the whole cube, and only positive-denominator pairs are
admissible.  Lemma~\ref{lem:dmcgp-slope-main} below gives the
positive-slope invariant used in the proof; finiteness of~$c_g^F$ is
kept as a theorem hypothesis or supplied by an application-specific
certificate.

\begin{lemma}[Non-negative multilinear slope after pruning]
  \label{lem:dmcgp-slope-main}
  For every $i \ge 1$ and every $j \in \uni$, either
  $(\tilde S_i)_j = 0$ or $\partial_j F(\tilde S_i) > 0$.
\end{lemma}
\begin{proof}
  Direct from the while-loop termination condition in
  Algorithm~\ref{alg:dmcgp-main}; see
  Lemma~\ref{lem:dmcgp-slope} in Appendix~\ref{apx:dmcgp} for
  the full statement and surrounding structural lemmas.
\end{proof}

\subsubsection{DMCG-P for general submodular \texorpdfstring{$f$}{f}}\label{sec:dmcgp-general}

The headline result of this subsection is that DMCG-P delivers a
curvature-aware $(1-e^{-\bar c_g^F})/\bar c_g^F$ guarantee for \emph{arbitrary}
submodular $f$ on such an integral feasible-set hull, with no
monotonicity or non-negativity assumption.  The only hypothesis
beyond submodularity is that the trajectory-restricted curvature
$c_g^F$ of Definition~\ref{def:cgF} is finite and the run is in
the step-size regime $\delta\bar c_g^F\le 1$ for
$\bar c_g^F=\max\{1,c_g^F\}$ and step size $\delta=1/T$---always
achievable by choosing~$T$ large enough.  Pruning
controls the trajectory through the positive-slope invariant, and the
descent argument of Theorem~\ref{thm:greedy} lifts to the
multilinear setting using only submodularity and the slope
invariant---never monotonicity of~$f$.

\begin{theorem}[General DMCG-P guarantee]\label{thm:dmcgp-general-main}
  Let $\mathcal I\subseteq 2^\uni$ be downward-closed, let
  $\polytope=\conv\{\mathbf{1}_S:S\in\mathcal I\}$, and assume exact
  linear optimization over $\mathcal I$.  Let $f : 2^\uni \to \reals$
  be submodular, let $O^*\in\argmax_{S\in\mathcal I}\ff{S}$, and let
  $c_g^F \in [0, \infty)$ be
  the trajectory-restricted fractional greedy curvature of
  Definition~\ref{def:cgF}.  Put $\bar c_g^F=\max\{1,c_g^F\}$.
  Assume $c_g^F<\infty$, $\delta\bar c_g^F\le 1$ for step size
  $\delta=1/T$.
  Then
  $\textsc{DMCG-P}(F, \polytope, T)$ returns $\tilde S_T \in \polytope$
  with
  \[
    F(\tilde S_T) \;\ge\; \frac{1 - e^{-\bar c_g^F}}{\bar c_g^F}\,\ff{O^*}
           \;-\; \frac{n(n-1)\,C_F}{T},
  \]
  where $n=|\uni|$, and
  $C_F = \max_{j \ne \ell,\,\vx\in[0,1]^\uni}|\partial_{j\ell}F(\vx)|$
  is the smoothness constant.
  The step count~$T$ is independent of the constraint; choosing~$T$
  large makes the discretization error negligible.
  No monotonicity or non-negativity of $f$ is required;
  the finite-curvature condition is an explicit trajectory hypothesis
  using the positive-denominator convention of Definition~\ref{def:cgF}.
\end{theorem}

The full proof is given in Appendix~\ref{apx:dmcgp} as
Theorem~\ref{thm:dmcgp-general}, via the per-step descent
Lemma~\ref{lem:dmcgp-descent}.
The proof is a line-by-line lift of the non-monotone branch of the
discrete proof of
Theorem~\ref{thm:greedy}: each discrete inequality (pruning, greedy
selection, submodularity, curvature, and the slope invariant on the
active set) has an immediate multilinear analogue via the standard
properties of the multilinear
extension~\citep{calinescu2011maximizing}, and the per-step
recurrence closes in the same form modulo an $O(1/T)$ discretization
correction.  The appendix states the finite-$T$ bound for
$\delta\bar c_g^F\le 1$, which yields the displayed fixed-curvature asymptotic.
Unlike discrete greedy, where each step adds one element and the step
count equals~$k$, the continuous greedy permits finer discretization:
choosing $T > k$ reduces the error without changing the constraint.
Two consequences are worth stating explicitly.  First, in the
strictly positive case $c_g^F \le c_F = c_f$
(Theorem~\ref{thm:nn-characterization}(a)), so
Theorem~\ref{thm:dmcgp-general-main} recovers the classical
continuous-greedy
guarantee~\citep{calinescu2011maximizing}; when $c_g^F < c_f$
strictly, the bound is a sharper instance-specific certificate.
Second---and this is the whole point of the algorithm---when $f$
takes negative values, $c_F = \infty$ globally
(Proposition~\ref{prop:cF-infty}), the theorem can still apply through
the trajectory-restricted parameter~$c_g^F$.  What remains in that
regime is the computational task of proving or certifying a finite
upper bound on $c_g^F$.
When $f$ is \emph{decomposable}---meaning $f = g - \ell$ with $g$
monotone submodular and $\ell$ modular non-negative, encompassing
regularized feature selection, cost-penalized experimental design,
and similar ML objectives---$c_g^F$ admits a closed-form OPT-free
removal-slope certificate from the pruned trajectory
(Section~\ref{sec:apps-decomp}); if that certificate is below~$1$,
the theorem uses $\bar c_g^F=1$ unless monotonicity supplies
the sharper classical small-curvature analysis.

\subsubsection{Non-negative non-monotone: wDMCG-P and the \texorpdfstring{$1/e$}{1/e} guarantee}\label{sec:dmcgp-nonneg}

This subsubsection treats a regime that falls outside the purview
of Theorem~\ref{thm:dmcgp-general-main}: $f$ is non-negative and
submodular, but neither monotone nor presented as $g - \ell$.
Here $f$ has no decomposable structure, so
$c_g^F$ need not be bounded; the curvature-aware descent pattern
of \S\ref{sec:dmcgp-general} therefore does not close.
A damped weighted variant of DMCG-P (wDMCG-P), with no discrete
counterpart in this paper, recovers the classical $e^{-1}$ guarantee
via a correlation-gap descent that invokes
\citet[Lem.~2.2]{feldman2011unified} and therefore relies on
non-negativity of~$f$ in an essential way. The bound is
regime-independent in $c_g^F$: it does \emph{not} improve with
curvature. The next subsection
(\S\ref{sec:mcg}) addresses this gap for the narrower class of
\emph{strictly positive} non-monotone $f$, where classical measured
continuous greedy recovers a curvature-aware bound that wDMCG-P
cannot.

Throughout this subsubsection, $O^* \in \argmax_{\mathbf{1}_S \in \polytope}\ff{S}$
denotes an optimal feasible set. A \emph{weighted, damped} variant
of DMCG-P (wDMCG-P) modifies the selection and the update
rule. The selection solves
$B_i = \argmax_{B \in \mathcal I} \sum_j(1 - \tilde S_{i,j})\partial_j F(\tilde S_i)\mathbf{1}_{B,j}$,
weighting each coordinate by how much room remains. The update is
$\tilde S_{i+1,j} = \tilde S_{i,j} +
(1/k)(1 - \tilde S_{i,j})\mathbf{1}_{B_i,j}$, ensuring
coordinates stay in $[0,1]$. This achieves a
$(1 - 1/k)^{k-1} \to e^{-1}$ guarantee via a descent lemma that
invokes the \citet{feldman2011unified} correlation-gap bound.

\begin{theorem}[Non-negative DMCG-P guarantee]\label{thm:wdmcgp-nonneg-main}
  Let $f : 2^\uni \to \reals_{\ge 0}$ be submodular on a downward-closed
  solvable polytope $\polytope$.  The damped weighted DMCG-P algorithm
  returns $\tilde S_k \in \polytope$ with
  \[
    F(\tilde S_k) \;\ge\; (1 - 1/k)^{k-1}\,\ff{O^*} \;-\; O(1/k)
      \;\xrightarrow[k \to \infty]{}\; e^{-1}\,\ff{O^*}.
  \]
\end{theorem}

The full proof is given in Appendix~\ref{apx:dmcgp} as
Theorem~\ref{thm:wdmcgp-nonneg}, via feasibility
(Lemma~\ref{lem:wdmcgp-damped-feasibility}), the coordinate bound
(Lemma~\ref{lem:wdmcgp-coord-bound}), and the per-step descent
(Lemma~\ref{lem:wdmcgp-descent}).
Non-negativity of~$f$ is essential: the descent lemma uses
\citet[Lem.~2.2]{feldman2011unified} in the form
$F(\vec{y} \lor \mathbf{1}_S) \ge (1 - \lVert \vec{y}\rVert_\infty)\ff{S}$,
which is known to fail for negative~$f$ and has no simple sign
correction.  The damped update is what makes the coordinate bound
$\lVert \tilde S_i\rVert_\infty \le 1 - (1 - 1/k)^i$ drive the
FNS step in the right direction; the unweighted, undamped DMCG-P of
\S\ref{sec:dmcgp-alg} does not suffice in this regime.  Pruning is
compatible with the analysis but plays no load-bearing role here---it
can only increase $F$ by Lemma~\ref{lem:dmcgp-slope-main} and is
retained for algorithmic consistency.

\begin{table}[!htbp]
  \centering
  \caption{Approximation guarantees by regime.  DMCG-P
    (Theorem~\ref{thm:dmcgp-general-main}) covers arbitrary submodular
    $f$ with finite trajectory curvature $c_g^F$, where
    $\bar c_g^F=\max\{1,c_g^F\}$; wDMCG-P
    (Theorem~\ref{thm:wdmcgp-nonneg-main}) covers the complementary
    non-negative non-monotone regime.
    MCG (Theorem~\ref{thm:mcg}) appears in the curvature row only
    for the strictly-positive sub-case.}
  \label{tab:regime-summary}
  \begin{tabular}{lll}
    \toprule
    Regime & DMCG-P / wDMCG-P & MCG \\
    \midrule
    Submodular $f$ (any sign) with $c_g^F < \infty$
      & $(1-e^{-\bar c_g^F})/\bar c_g^F$
      & $(1-e^{-c_f})/c_f{}^{\,*}$ \\
    Non-negative non-monotone $f$ (no bound on $c_g^F$)
      & $1/e$
      & N/A \\
    \bottomrule
  \end{tabular}

  \smallskip
  \footnotesize
  $^{*}$MCG provides an alternative $c_f$-parameterized guarantee
  under the additional hypothesis of strict positivity
  ($f(S) > 0$ for every nonempty~$S$), where $c_F = c_f$
  (Theorem~\ref{thm:nn-characterization}(a)). DMCG-P also covers
  this case via $c_g^F \le c_f$, with the sharper small-curvature
  formula supplied by monotonicity.

  \smallskip
	  \textbf{How to bound $c_g^F$.}  (i)~When $f$ is monotone
	  submodular, $c_g^F \le c_f \in [0,1]$ automatically.
	  (ii)~When $f = g - \ell$ is decomposable, a closed-form
	  certificate bounds $c_g^F$ from the pruned trajectory alone
	  (Section~\ref{sec:apps-decomp}), covering decomposable negative-$f$ cases.
\end{table}

Table~\ref{tab:regime-summary} summarizes the two regimes of this
subsection. The DMCG-P row is the headline
Theorem~\ref{thm:dmcgp-general-main}: whenever the trajectory
curvature $c_g^F$ is finite and $\delta\bar c_g^F\le 1$---with decomposable
certificates providing one concrete way to verify this---DMCG-P delivers
the $(1 - e^{-\bar c_g^F})/\bar c_g^F$ bound regardless of monotonicity or
sign of~$f$.  The wDMCG-P row is the complementary non-negative
non-monotone regime, where $c_g^F$ need not be bounded and the
correlation-gap analysis recovers a regime-independent~$1/e$.
Under the additional hypotheses of strict positivity and the standard
MCG monotonicity lemma, the classical measured continuous greedy
analysis offers an \emph{alternative}
$(1 - e^{-c_f})/c_f$ bound parameterized by the globally-computable
set-function curvature~$c_f$, via the equality $c_F = c_f$ of
Theorem~\ref{thm:nn-characterization}(a).
The next subsection records this result.

\subsection{Measured Continuous Greedy}\label{sec:mcg}

We record a conditional classical comparison for measured continuous
greedy (MCG).
DMCG-P already covers this regime via
Theorem~\ref{thm:dmcgp-general-main}: when $f$ is strictly
positive, $c_g^F \le c_F = c_f < \infty$
(Theorem~\ref{thm:nn-characterization}(a)), so DMCG-P delivers
the $\bar c_g^F$-parameterized guarantee above.
The measured continuous greedy (MCG) guarantee below is
parameterized by the \emph{global} curvature~$c_f$ rather than the
trajectory-restricted~$c_g^F$. Its advantage is that $c_f$ is a
known quantity for certain function classes (e.g., graph-cut
functions have $c_f = 2$), whereas bounding $c_g^F$ may require a
per-instance certificate. Its disadvantage is that it requires
both strict positivity and the standard MCG monotonicity lemma stated
in Appendix~\ref{apx:mcg}; it does not apply when $f$ takes negative
values, and it cannot exploit $c_g^F < c_f$. In contrast, DMCG-P uses
pruning to enforce the local positive-slope condition needed in its own
descent proof.
The algorithm and proof are from~\citet{feldman2017maximizing},
lifted to the curvature setting in Appendix~\ref{apx:mcg}.

\begin{theorem}[MCG guarantee]\label{thm:mcg}
  Let $f$ be strictly positive, submodular with curvature $c_f > 1$
  (the non-monotone case; for $c_f \le 1$, $f$ is monotone by
  Proposition~\ref{prop:c-monotone}),
  and $\polytope$ a downward-closed solvable polytope.  Assume the MCG
  trajectory satisfies Assumption~\ref{assump:mcg-mono}.
  Then \mcg returns $\mathbf{y}(T)$ with
  \[
    F(\mathbf{y}(T)) \ge \frac{1 - e^{-c_f T} - o(1)}{c_f}\,\ff{O^*},
  \]
  where $O^* \in \argmax_{\mathbf{1}_S \in \polytope} \ff{S}$ and $T \in [0,1]$.
\end{theorem}

\input{cdgCurvature.tex}

%% ================================================================
\section{Quadratic Programming Application}\label{apx:qp}
%% ================================================================

We consider DR-submodular QP instances
$\max_{\mathbf{x} \in \polytope}
F(\mathbf{x}) = \frac{1}{2}\mathbf{x}^\top \mathbf{H}\mathbf{x}
+ \mathbf{h}^\top\mathbf{x}$
where $\mathbf{H}$ is symmetric entry-wise non-positive
and $\mathbf{h} = -\beta\,\mathbf{H}^\top\mathbf{u}$ for $\beta \in (1/2, 1)$.

\begin{proposition}\label{prop:qp-curv}
  The QP instance is normalized, non-negative, and DR-submodular
  with curvature $c_F \le 2/(2\beta - 1)$.  It is non-monotone whenever
  $(\mathbf H\mathbf u)_j<0$ for some feasible coordinate~$j$.
\end{proposition}
This is a curvature calculation for a continuous DR-submodular objective;
the DMCG-P theorem above is stated for multilinear extensions of set
functions over integral feasible-set hulls.  See
Appendix~\ref{apx:deferred} for the proof.

%% ================================================================
\section{Deferred Proofs}\label{apx:deferred}
%% ================================================================

This appendix collects proofs deferred from the main text.

\noindent\textbf{Restatement of Proposition~\ref{prop:c-monotone}.}
For a submodular function $f$ with strictly positive singleton values
$\ff{e}>0$ for every $e\in\uni$, the curvature parameter satisfies
$c_f\le1$ if and only if $f$ is monotone.

\begin{proof}[Proof of Proposition~\ref{prop:c-monotone}]
  $(\Rightarrow)$ If $f$ is monotone, then $\ff{X \cup Y} - \ff{X} \ge 0$
  and $\ff{Y \setminus X} \ge 0$, so the ratio is non-negative
  and $c_f \le 1$.

  $(\Leftarrow)$ If $c_f \le 1$, suppose for contradiction there exist
  $X, Y$ with $\ff{X \cup Y} - \ff{X} < 0$.
  Decomposing $Y \setminus X = \{e_1, \ldots, e_m\}$ and writing
  $Y_i = \{e_1, \ldots, e_i\}$, the telescoping sum
  $\ff{X \cup Y} - \ff{X} = \sum_{i=1}^m \marge{e_i}{X \cup Y_{i-1}}$
  is negative, so some marginal $\marge{e_i}{X \cup Y_{i-1}} < 0$.
  Since $\marge{e_i}{\emptyset} > 0$ by assumption, the ratio
  $\marge{e_i}{X \cup Y_{i-1}}/\marge{e_i}{\emptyset} < 0$
  and $c_f \ge 1 - \marge{e_i}{X \cup Y_{i-1}}/\marge{e_i}{\emptyset} > 1$,
  a contradiction.
  Therefore $\ff{X \cup Y} \ge \ff{X}$ for all $X, Y$, so $f$ is monotone.
\end{proof}

\noindent\textbf{Restatement of Proposition~\ref{prop:classical}.}
For a monotone submodular function, the set-wise curvature $c_f$ equals the
Conforti--Cornu\'ejols total curvature
\[
  \alpha = 1-\min_{e:\ff{e}>0}\frac{\marge{e}{\uni\setminus e}}{\marge{e}{\emptyset}}.
\]

\begin{proof}[Proof of Proposition~\ref{prop:classical}]
  We prove the two inequalities separately.

  \emph{Upper bound $c_f \le \alpha$.}
  Fix an admissible pair $X,Y$ with $\ff{Y\setminus X}>0$, and write
  $Z=Y\setminus X=\{z_1,\dots,z_m\}$ in an arbitrary order.
  By telescoping and submodularity,
  \[
    \ff{X\cup Y}-\ff{X}
    =\sum_{t=1}^{m}\marge{z_t}{X\cup\{z_1,\dots,z_{t-1}\}}
    \ge \sum_{t=1}^{m}\marge{z_t}{\uni\setminus z_t}.
  \]
  The definition of $\alpha$ gives
  $\marge{z_t}{\uni\setminus z_t}\ge(1-\alpha)\marge{z_t}{\emptyset}$ for
  each~$t$.  Since $f$ is monotone submodular, it is subadditive, so
  $\sum_t \marge{z_t}{\emptyset}=\sum_t\ff{z_t}\ge \ff{Z}=\ff{Y\setminus X}$.
  Therefore
  \[
    \ff{X\cup Y}-\ff{X}\ge(1-\alpha)\ff{Y\setminus X},
  \]
  and every admissible ratio in Definition~\ref{def:curv} is at least
  $1-\alpha$.

  \emph{Lower bound $c_f \ge \alpha$.}
  Let $e$ attain the minimum in the definition of~$\alpha$.
  Use Definition~\ref{def:curv} with $X=\uni\setminus e$ and $Y=\{e\}$.
  The denominator is $\ff{e}>0$, and the ratio is exactly
  $\marge{e}{\uni\setminus e}/\ff{e}=1-\alpha$.
  Hence the minimum ratio in Definition~\ref{def:curv} is at most
  $1-\alpha$, so $c_f\ge\alpha$.
\end{proof}

\noindent\textbf{Restatement of Proposition~\ref{prop:symmetric}.}
If $f$ is normalized, symmetric
($\ff{S}=\ff{\uni\setminus S}$), and submodular, and has a positive
singleton, then $f$ is non-negative and $c_f=2$.

\begin{proof}[Proof of Proposition~\ref{prop:symmetric}]
  First, $f$ is non-negative.  Applying submodularity to
  $S$ and $\uni\setminus S$ gives
  $f(S)+f(\uni\setminus S)\ge f(\uni)+f(\emptyset)=0$; by symmetry,
  the left-hand side is $2f(S)$.

  \emph{Upper bound} ($c_f \le 2$): For any $X, Y$ with $\ff{Y \setminus X} > 0$,
  write $Z=Y\setminus X$, so $Z\cap X=\emptyset$ and $X\cup Y=X\cup Z$.
  Submodularity applied to $X\cup Z$ and $\uni\setminus Z$ gives
  $\ff{X\cup Z}+\ff{\uni\setminus Z}\ge \ff{\uni}+\ff{X}$.
  By symmetry and normalization, $\ff{\uni\setminus Z}=\ff{Z}$ and
  $\ff{\uni}=0$, hence
  $\ff{X \cup Y} - \ff{X} \ge -\ff{Y \setminus X}$,
  giving ratio $\ge -1$, hence $c_f \le 2$.

  \emph{Lower bound} ($c_f \ge 2$): choose a singleton $Y=\{e\}$ with
  $\ff{Y}>0$, which exists by hypothesis, and set $X = \uni \setminus Y$.
  Then
  $(\ff{X \cup Y} - \ff{X})/\ff{Y} = (\ff{\uni} - \ff{X})/\ff{Y}
  = (\ff{\uni} - \ff{Y})/\ff{Y}$.
  Since $\ff{\uni} = \ff{\emptyset} = 0$ (by symmetry and normalization),
  this equals $-1$, giving $c_f \ge 2$.
\end{proof}

\noindent\textbf{Restatement of Proposition~\ref{prop:cF-infty}.}
Let $f$ be normalized submodular with $\ff{e}>0$ for every singleton.
If $f$ takes a negative value, then the multilinear-extension curvature
$c_F$ is infinite.

\begin{proof}[Proof of Proposition~\ref{prop:cF-infty}]
  Let $\uni' \subseteq \uni$ be inclusion-minimal with $\ff{\uni'} < 0$;
  minimality gives $\ff{S} \ge 0$ for all $S \subsetneq \uni'$,
  and $|\uni'| \ge 2$ since $\ff{e} > 0$ for every singleton.
  Pick any $e' \in \uni'$ and define
  $\mathbf{y} = \mathbf{1}_{\uni'}$,\;
  $\mathbf{x}_\varepsilon = (1-\varepsilon)\,\mathbf{1}_{\uni' \setminus \{e'\}}$
  for $\varepsilon \in (0,1)$.
  Then
  $\mathbf{z}_\varepsilon
  \coloneqq \mathbf{y} - \mathbf{x}_\varepsilon \land \mathbf{y}
  = \varepsilon\,\mathbf{1}_{\uni' \setminus \{e'\}} + \mathbf{1}_{\{e'\}}$
  and
  $\mathbf{x}_\varepsilon \lor \mathbf{y} = \mathbf{1}_{\uni'}$.

  The denominator $D(\varepsilon) = F(\mathbf{z}_\varepsilon)$ is continuous
  with $D(0) = \ff{e'} > 0$ and $D(1) = \ff{\uni'} < 0$.
  Let $\varepsilon_0 = \inf\{\varepsilon \in (0,1) : D(\varepsilon) \le 0\}$;
  by continuity, $D(\varepsilon_0) = 0$ and
  $D(\varepsilon) > 0$ for all $\varepsilon < \varepsilon_0$.

  The numerator satisfies
  $F(\mathbf{x}_\varepsilon \lor \mathbf{y}) - F(\mathbf{x}_\varepsilon)
  = \ff{\uni'} - F\bigl((1{-}\varepsilon)\,\mathbf{1}_{\uni' \setminus \{e'\}}\bigr)
  \le \ff{\uni'} < 0$
  for all~$\varepsilon$,
  since $(1{-}\varepsilon)\,\mathbf{1}_{\uni' \setminus \{e'\}}$ samples
  only subsets of $\uni' \setminus \{e'\}$, all having $f \ge 0$ by minimality.

  As $\varepsilon \to \varepsilon_0^-$, the denominator $D(\varepsilon) \to 0^+$
  while the numerator stays at most $\ff{\uni'} < 0$,
  so the ratio diverges to $-\infty$ and $c_F = \infty$.
\end{proof}

\noindent\textbf{Restatement of Theorem~\ref{thm:nn-characterization}.}
Let $f$ be normalized submodular with positive singleton values.  If
$f$ is strictly positive on every nonempty set, then $c_F=c_f$; if
$f$ takes a negative value, then $c_F=\infty$.

\begin{proof}[Proof of Theorem~\ref{thm:nn-characterization}]
  Part~(b) is Proposition~\ref{prop:cF-infty}.

  For part~(a), we show $c_F \le c_f$
  (equality then follows from Proposition~\ref{prop:c-leq-cF}).
  For $\mathbf{x}, \mathbf{y} \in [0,1]^\uni$, construct a coupling
  $(R_1, R_2) \sim \mathcal{D}(\mathbf{x}, \mathbf{y})$ where
  each element $i$ satisfies:
  $\Pr[i \in R_1] = x_i$, \;
  $\Pr[i \in R_2 \mid i \in R_1] = 0$, \;
  $\Pr[i \in R_2 \mid i \notin R_1] = (y_i - \min(x_i, y_i))/(1 - x_i)$.
  When $x_i=1$, the event $i\notin R_1$ has probability zero; define the
  last conditional probability arbitrarily, say as~$0$.

  Then $R_1 \sim \mathcal{R}(\mathbf{x})$,
  $R_2 \sim \mathcal{R}(\mathbf{y} - \mathbf{x} \land \mathbf{y})$,
  $R_1 \cup R_2 \sim \mathcal{R}(\mathbf{x} \lor \mathbf{y})$,
  and $R_1 \cap R_2 = \emptyset$.
  When $R_2 \ne \emptyset$, strict positivity gives $\ff{R_2} > 0$,
  so the curvature inequality
  $\ff{R_1 \cup R_2} - \ff{R_1} \ge (1-c_f)\,\ff{R_2}$ applies by
  Definition~\ref{def:curv}.
  When $R_2 = \emptyset$, both sides equal zero.
  Thus the inequality holds for all realizations of $(R_1, R_2)$.
  Taking expectations:
  $F(\mathbf{x} \lor \mathbf{y}) - F(\mathbf{x}) \ge (1-c_f)\,F(\mathbf{y} - \mathbf{x} \land \mathbf{y})$,
  giving $c_F \le c_f$.
\end{proof}

\begin{proposition}[OPT-aware decomposable curvature bound]\label{prop:greedy-curv-decomp}
For $f=g-\ell$ with $g$ monotone submodular of CC curvature $\alpha_g$
and $\ell$ non-negative modular, let
$A_0,A_1,\dots$ be the active-set trajectory of greedy with pruning and
$\mathcal O$ the set of optimal cardinality-feasible solutions.  If
\[
  r=\max_{O^*\in\mathcal O}\max_{i:\ff{A_i\setminus O^*}>0}
    \frac{\ell(A_i\setminus O^*)}{g(A_i\setminus O^*)},
\]
then \(c_g\le \alpha_g/(1-r)\).
\end{proposition}

\begin{proof}[Proof of Proposition~\ref{prop:greedy-curv-decomp}]
  Fix any $O^* \in \mathcal{O}$ and any trajectory step~$i$ with
  $\ff{A_i\setminus O^*}>0$.
  Apply the decomposition $f = g - \ell$ with $X = O^*$, $Y = A_i$:
  \begin{align*}
    \ff{O^* \cup A_i} - \ff{O^*}
    &= [g(O^* \cup A_i) - g(O^*)] - \ell(A_i \setminus O^*) \\
    &\ge (1 - \alpha_g)\,g(A_i \setminus O^*) - \ell(A_i \setminus O^*),
  \end{align*}
  where the inequality uses the CC curvature bound on~$g$.
  The denominator $g(A_i\setminus O^*)$ is positive on every qualifying
  step, because
  $g(A_i\setminus O^*)=\ff{A_i\setminus O^*}+\ell(A_i\setminus O^*)$
  and both terms on the right are non-negative with
  $\ff{A_i\setminus O^*}>0$.
  Writing $r_i = \ell(A_i \setminus O^*)/g(A_i \setminus O^*)$ and
  $\ff{A_i \setminus O^*} = (1 - r_i)\,g(A_i \setminus O^*)$:
  \[
    \frac{\ff{O^* \cup A_i} - \ff{O^*}}{\ff{A_i \setminus O^*}}
    \;\ge\; \frac{(1-\alpha_g) - r_i}{1 - r_i}
    \;=\; 1 - \frac{\alpha_g}{1 - r_i}
    \;\ge\; 1 - \frac{\alpha_g}{1 - r},
  \]
  where the last inequality uses $r_i \le r$ by definition of~$r$
  (which maximizes over \emph{all} $O^* \in \mathcal{O}$ and all
  trajectory steps).
  Since this holds for every $O^*$ and every qualifying step~$i$,
  taking the double minimum in Definition~\ref{def:greedy-curv}:
  \[
    c_g \;=\; 1 - \min_{O^* \in \mathcal{O}}\;\min_{i:\,\ff{A_i\setminus O^*}>0}
    \frac{\ff{O^*\cup A_i}-\ff{O^*}}{\ff{A_i\setminus O^*}}
    \;\le\; \frac{\alpha_g}{1 - r}.\qedhere
  \]
\end{proof}

\noindent\textbf{Restatement of Proposition~\ref{prop:movie-curv}.}
For the GCLin objective
\[
  \ff{S}=R(S)-\lambda D(S),
  \qquad
  R(S)=\sum_{i\in\uni}\sum_{j\in S}s_{i,j},
  \qquad
  D(S)=\sum_{\substack{i,j\in S\\i\neq j}}s_{i,j},
\]
with symmetric non-negative similarities $s_{i,j}$, uniform query weights ($w_i=1$),
and $0\le\lambda\le1$, the curvature satisfies $c_f\le2\lambda$.

\begin{proof}[Proof of Proposition~\ref{prop:movie-curv}]
  Let $B=T\setminus S$, and consider only pairs with $\ff{B}>0$, as in
  Definition~\ref{def:curv}.  By symmetry of the similarities,
  for any $S,T\subseteq\uni$:
  \begin{align*}
    \ff{S \cup T} - \ff{S}
    &= \ff{B} - 2\lambda \sum_{i \in S}\sum_{j \in B} s_{i,j} \\
    &\ge \ff{B} - 2\lambda \sum_{i \in \uni \setminus B}\sum_{j \in B} s_{i,j}.
  \end{align*}
  The remaining cross-similarity term is controlled by the standalone
  GCLin value of~$B$ when $\lambda\le1$:
  \[
    \sum_{i\in\uni\setminus B}\sum_{j\in B}s_{i,j}
    =
    \sum_{i\in\uni}\sum_{j\in B}s_{i,j}
    -
    \sum_{i\in B}\sum_{j\in B}s_{i,j}
    \le
    \sum_{i\in\uni}\sum_{j\in B}s_{i,j}
    -
    \lambda\sum_{i\in B}\sum_{j\in B}s_{i,j}
    \le
    \ff{B}.
  \]
  Hence
  \[
    \ff{S\cup T}-\ff{S}\ge (1-2\lambda)\ff{B}
    =(1-2\lambda)\ff{T\setminus S}.
  \]
  Therefore $c_f \le 2\lambda$.

  The monotonicity claim used in Section~\ref{sec:gclin-diversity} for
  $\lambda\le1/2$ follows from the corresponding singleton marginal:
  for $e\notin S$,
  \[
    \ff{S\cup\{e\}}-\ff{S}
    =
    s_{e,e}
    +(1-2\lambda)\sum_{i\in S}s_{i,e}
    +\sum_{i\in\uni\setminus(S\cup\{e\})}s_{i,e}
    \ge0.
  \]
\end{proof}

\begin{remark}
  The restriction $\lambda\le1$ is not a proof artifact.  For
  $\lambda>1$, the positive standalone value $\ff{B}=
  R(B)-\lambda D(B)$ can be arbitrarily small relative to the
  cross-similarity $\sum_{i\in S,j\in B}s_{i,j}$, so the argument above
  gives no uniform $2\lambda$ curvature bound without an additional
  cross-redundancy condition.  The $\lambda=1.5$ experiments in
  Section~\ref{sec:gclin-diversity} are therefore reported only as
  empirical trajectory diagnostics.
\end{remark}

\noindent\textbf{Restatement of Proposition~\ref{prop:qp-curv}.}
For the DR-submodular quadratic objective
\[
  F(\mathbf{x})=\frac12\mathbf{x}^{\top}\mathbf H\mathbf x
  -\beta\,\mathbf u^{\top}\mathbf H\mathbf x,
\]
where $\mathbf H$ is symmetric entry-wise non-positive, $\mathbf x\le\mathbf u$
on the polytope, and $\beta\in(1/2,1)$, the instance is normalized,
non-negative, DR-submodular, and has $c_F\le2/(2\beta-1)$.  It is
non-monotone whenever $(\mathbf H\mathbf u)_j<0$ for some feasible
coordinate~$j$.

\begin{proof}[Proof of Proposition~\ref{prop:qp-curv}]
  \emph{Normalization and DR-submodularity.}
  We have $F(\mathbf0)=0$.  The Hessian of $F$ is exactly
  $\mathbf H$, whose off-diagonal entries are non-positive by assumption,
  so every first partial derivative is coordinate-wise non-increasing in
  the other coordinates.  This is precisely DR-submodularity for smooth
  functions.

  \emph{Non-monotonicity in nondegenerate instances.}
  The gradient is $\nabla F(\mathbf x)=\mathbf H\mathbf x-\beta\mathbf H\mathbf u$.
  At $\mathbf x=\mathbf u$,
  $\nabla F(\mathbf u)=(1-\beta)\mathbf H\mathbf u$.  Thus, if
  $(\mathbf H\mathbf u)_j<0$ for some feasible coordinate~$j$, the
  marginal in that coordinate is negative at~$\mathbf u$, and the
  objective is not monotone on the feasible box.

  \emph{Non-negativity.}
  For $\mathbf{x} \in \polytope$, since $\mathbf{x} \le \mathbf{u}$ and $\mathbf{H}$
  is entry-wise non-positive, we have
  $\mathbf{x}^\top \mathbf{H}\mathbf{x} \ge \mathbf{u}^\top \mathbf{H}\mathbf{x}$
  and $\mathbf{u}^\top \mathbf{H}\mathbf{x} \le 0$.
  Therefore
  $F(\mathbf{x}) = \tfrac{1}{2}\mathbf{x}^\top\mathbf{H}\mathbf{x}
  - \beta\,\mathbf{u}^\top\mathbf{H}\mathbf{x}
  \ge (\tfrac{1}{2} - \beta)\,\mathbf{u}^\top\mathbf{H}\mathbf{x} \ge 0$.

  \emph{Curvature.}
  For $\mathbf{x}, \mathbf{y} \in \polytope$, let $\mathbf{z} = \mathbf{y} - \mathbf{x} \land \mathbf{y}$.
  Then $F(\mathbf{z}) = \tfrac{1}{2}\mathbf{z}^\top\mathbf{H}\mathbf{z}
  - \beta\,\mathbf{u}^\top\mathbf{H}\mathbf{z}
  \ge (\tfrac{1}{2} - \beta)\,\mathbf{u}^\top\mathbf{H}\mathbf{z}
  \ge (\tfrac{1}{2} - \beta)\,\mathbf{x}^\top\mathbf{H}\mathbf{z}$,
  where the last step uses $\mathbf{x} \le \mathbf{u}$ and $\mathbf{H}\mathbf{z} \le \mathbf{0}$.
  Also $F(\mathbf{x} \lor \mathbf{y}) - F(\mathbf{x})
  = F(\mathbf{z}) + \mathbf{z}^\top\mathbf{H}\mathbf{x}
  \ge F(\mathbf{z}) + \tfrac{1}{\tfrac{1}{2}-\beta}\,F(\mathbf{z})
  = (1 + \tfrac{1}{\tfrac{1}{2}-\beta})\,F(\mathbf{z})$.
  Since $\tfrac{1}{2} - \beta < 0$, this gives
  $F(\mathbf{x} \lor \mathbf{y}) - F(\mathbf{x})
  \ge (1 - \tfrac{2}{2\beta - 1})\,F(\mathbf{z})$,
  hence $c_F \le \tfrac{2}{2\beta - 1}$.
\end{proof}

\end{document}

%% file: util.tex
\usepackage[utf8]{inputenc}
\usepackage[T1]{fontenc}
\usepackage{url}
\usepackage{booktabs}
\usepackage{amsfonts}
\usepackage{nicefrac}
\usepackage{microtype}
\pdfoutput=1

\usepackage{amsmath}
 
\usepackage{amsthm}
\usepackage[ruled,linesnumbered,noend]{algorithm2e}
\usepackage{graphicx}
\usepackage{subcaption}
\usepackage{wrapfig}
\usepackage{tikz}
\usetikzlibrary{arrows}
\usepackage{mathrsfs}
\usepackage{amssymb}
\usepackage{changepage}
\usepackage{multirow}
\usepackage{xspace}

\usepackage{xcolor}
\usepackage{hyperref}
\hypersetup{
    colorlinks,
    linkcolor={red!80!black},
    citecolor={blue!60!black},
    urlcolor={blue!80!black}
}

\newtheorem{theorem}{Theorem}
\newtheorem*{theorem*}{Theorem}
\newtheorem{lemma}{Lemma}
\newtheorem{definition}{Definition}

\newtheorem{proposition}{Proposition}

\newtheorem{assumption}{Assumption}
\theoremstyle{definition}
\newtheorem{remark}{Remark}

\newtheorem{example}{Example}

\newcommand{\ex}[1]{ \mathbb{E} \left[ #1 \right] }
\newcommand{\reals}{\mathbb{R}}
\newcommand{\ff}[1]{ f \left( #1 \right) }
\renewcommand{\vec}[1]{ \mathbf{ #1 } }
\newcommand{\marge}[2]{ \Delta \left( #1 | #2 \right) }
\DeclareMathOperator*{\argmax}{arg\,max}

\DeclareMathOperator{\diag}{diag}
\DeclareMathOperator{\tr}{tr}

\newcommand{\uni}{\mathcal{N}}
\newcommand{\greedy}{\textsc{Greedy}\xspace}

\newcommand{\mcg}{\textsc{MeasuredContinuousGreedy}\xspace}
\newcommand{\polytope}{\mathcal{P}}
\providecommand{\dmcgp}{\textsc{DMCG-P}\xspace}
\providecommand{\wdmcgp}{\textsc{wDMCG-P}\xspace}
\providecommand{\coloneqq}{\mathrel{:=}}
\providecommand{\conv}{\operatorname{conv}}
\providecommand{\supp}{\operatorname{supp}}
\providecommand{\bR}{\mathbb{R}}
\providecommand{\bE}{\mathbb{E}}
\providecommand{\bN}{\mathbb{N}}

\providecommand{\cI}{\mathcal{I}}
\providecommand{\cO}{\mathcal{O}}
\providecommand{\cN}{\mathcal{N}}
\providecommand{\vzero}{\mathbf{0}}
\providecommand{\vone}{\mathbf{1}}
\providecommand{\vx}{\mathbf{x}}
\providecommand{\vy}{\mathbf{y}}
\providecommand{\vz}{\mathbf{z}}
\providecommand{\vw}{\mathbf{w}}
\providecommand{\vd}{\mathbf{d}}
\providecommand{\ve}{\mathbf{e}}
\providecommand{\one}{\mathbf{1}}

%% file: cdgCurvature.tex
\section{DMCG-P: Full Proofs}
\label{apx:dmcgp}

This appendix contains the full proofs for
Section~\ref{sec:dmcgp}.  The headline result is
Theorem~\ref{thm:dmcgp-general}: DMCG-P achieves a
$(1-e^{-\bar c_g^F})/\bar c_g^F$ multiplicative guarantee for
\emph{arbitrary} submodular~$f$ with finite trajectory curvature~$c_g^F$,
including functions that take negative values.
For decomposable $f = g - \ell$,
Proposition~\ref{prop:dmcgp-decomp-cert} provides a closed-form
certificate $c_g^F \le \alpha_g/(1-\hat r_F)$ from the pruned
trajectory, making the guarantee concrete without
access to OPT.  This is the fractional
analogue of the discrete decomposable certificate used for the
applications.

We first recall the setup and algorithm reference
(\S\ref{apx:dmcgp-alg}), then establish the structural lemmas needed
for the analysis (\S\ref{apx:dmcgp-lemmas}),
and prove the headline DMCG-P guarantee
$(1-e^{-\bar c_g^F})/\bar c_g^F \cdot f(O^*)$ for \emph{arbitrary} submodular~$f$
with finite trajectory curvature $c_g^F$ via the \emph{fractional
greedy curvature} $c_g^F$ tailored to the
DMCG-P trajectory (\S\ref{apx:dmcgp-general}).  No monotonicity or
non-negativity of~$f$ is required: the analysis is a direct
line-by-line lift of the discrete greedy proof
(Theorem~\ref{thm:greedy}) to the multilinear setting, and each step
uses only submodularity or the pruning slope invariant
(Lemma~\ref{lem:dmcgp-slope}).
Section~\ref{apx:dmcgp-fractional} then establishes the
decomposable-structure \emph{certificate} $c_g^F \le
\alpha_g/(1 - \hat r_F)$, making
\S\ref{apx:dmcgp-general} concrete on negative-valued~$f$, and
Section~\ref{apx:dmcgp-nonneg} handles the complementary
non-negative non-monotone regime via damped weighted DMCG-P.
Throughout this appendix, ratios of the form $(1-e^{-c})/c$ and
$(1-(1-c/T)^T)/c$ are interpreted by continuous extension at $c=0$
when the small-curvature monotone analysis is invoked.

\subsection{Setup and Algorithm Reference}\label{apx:dmcgp-alg}

Throughout \S\S\ref{apx:dmcgp-alg}--\ref{apx:dmcgp-general} we work with
a downward-closed family of feasible sets $\cI \subseteq 2^\cN$ and its
integral hull
$P = \conv\{\vone_B : B \in \cI\}\subseteq[0,1]^\cN$.  We assume exact
linear optimization over $\cI$ (equivalently over $P$) is available.
This includes matroid independent-set polytopes and other explicitly
integral feasible-set hulls, but not an arbitrary fractional
downward-closed relaxation unless it is represented as such a hull.  We
write $b=\max_{B\in\cI}|B|$.
For $f: 2^\cN \to \bR$ submodular, the \emph{multilinear extension}
$F: [0,1]^\cN \to \bR$ is
\[
  F(\vx) \;=\; \bE_{R \sim \vx}[f(R)]
  \;=\; \sum_{R \subseteq \cN} f(R) \prod_{i \in R} x_i \prod_{j \notin R}(1 - x_j),
\]
where $R$ includes each $i$ independently with probability $x_i$.
For a coordinate $j \in \cN$ let
$\partial_j F(\vx) = F(\vx \mid x_j = 1) - F(\vx \mid x_j = 0)$;
by multilinearity, $\partial_j F$ is independent of $x_j$.

Algorithm~\ref{alg:dmcgp-main} gives the DMCG-P procedure used
throughout this appendix.
Three details of Algorithm~\ref{alg:dmcgp-main} are used repeatedly.
\begin{enumerate}
  \item[(i)] The selection rule on Line~\ref{line:dmcgp-select} uses the
    \emph{unweighted} gradient: it maximises
    $\langle \nabla F(\tilde S_i),\,\vone_B\rangle$ over $B\in\cI$, not the
    Measured-Continuous-Greedy weighted form with multiplier
    $(1-(\tilde S_i)_j)$ on each coordinate.  Since
    $P=\conv\{\vone_B:B\in\cI\}$, this is equivalent to linear
    optimization over $P$, and the selected direction is the integral
    vector $v_i=\vone_{B_i}$ used in the analysis.
    The analysis in \S\ref{apx:dmcgp-general} shows that the unweighted
    form suffices: Lemma~\ref{lem:dmcgp-slope} (non-negative surviving
    slopes) is exactly what lets one turn a weighted swap bound into an
    unweighted one.
  \item[(ii)] After the final iteration, $\tilde S_T$ lies in the
    downward-closed polytope $P$ but need not itself be integral;
    pruning may have reduced total mass below $b$.
    If an integral output is required, one may apply a rounding method
    for points in integral hulls, such as swap rounding or pipage rounding,
    to $\tilde S_T \in P$.
  \item[(iii)] The pruning loop is written with a nonstrict inequality
    $\partial_j F \le 0$ on Line~\ref{line:dmcgp-prune}.  As in
    Section~\ref{sec:greedy}, this ensures
    that after pruning every surviving coordinate has strictly positive
    multilinear slope (Lemma~\ref{lem:dmcgp-slope} below), which is what
    drives the multilinear curvature argument.
\end{enumerate}

\subsection{Structural Lemmas}\label{apx:dmcgp-lemmas}

We collect here the facts about DMCG-P that underlie the analysis.
All lemmas hold for any submodular $f$ (monotone or not), unless noted.

\begin{lemma}[Feasibility]\label{lem:dmcgp-feas}
  For every $i \in \{0, 1, \dots, T\}$, the iterate $\tilde S_i$ lies in
  the downward-closed polytope $P$ and satisfies
  $\tilde S_i \in [0,1]^\cN$.  In particular, $(\tilde S_i)_j \le i/T$
  for every coordinate $j$.
\end{lemma}
\begin{proof}
  We show $\tilde S_i \le \frac1T\sum_{\tau<i} v_\tau$
  coordinate-wise by induction.  The base case is $\tilde S_0=\vzero$.
  If the bound holds at step~$i$, then before pruning
  $S_{i+1}=\tilde S_i+\frac1T v_i \le \frac1T\sum_{\tau\le i}v_\tau$.
  Pruning only zeroes coordinates, so $\tilde S_{i+1}\le S_{i+1}$
  coordinate-wise, preserving the invariant.
  Since each $v_\tau\in P$ and $P$ is downward-closed, the average
  $\frac1T\sum_{\tau\le i}v_\tau = \frac{i+1}{T}\bigl(\frac{1}{i+1}\sum_{\tau\le i}v_\tau\bigr)\in P$,
  whence $\tilde S_{i+1}\in P$.
  The coordinate bound $(\tilde S_i)_j\le i/T$ follows because
  $\sum_{\tau<i}(v_\tau)_j\le i$ (each vertex is in $\{0,1\}^\cN$),
  and pruning can only decrease coordinates further.
\end{proof}

\begin{lemma}[Pruning is value-improving]\label{lem:dmcgp-prune-up}
  At every iteration, $F(\tilde S_{i+1}) \ge F(S_{i+1})$, with strict
  inequality if any coordinate is pruned.
\end{lemma}
\begin{proof}
  Consider one pass of the pruning loop, and let $\vx$ be the current
  vector just before coordinate~$j$ is set to zero.  Multilinearity makes
  $F$ affine in coordinate~$j$, and $\partial_jF(\vx)$ is independent of
  $x_j$.  Therefore replacing $x_j>0$ by~$0$ changes the value by
  \[
    F(\vx-x_j\ve_j)-F(\vx)=-x_j\,\partial_jF(\vx).
  \]
  The pruning rule fires only when $\partial_jF(\vx)\le0$, so the displayed
  change is non-negative, and it is strict when the slope is strictly
  negative.  Applying this argument to each pass of the while-loop gives
  $F(\tilde S_{i+1})\ge F(S_{i+1})$.
\end{proof}

\begin{lemma}[Nonneg.\ slope after pruning]\label{lem:dmcgp-slope}
  At every iterate $\tilde S_i$ (for $i \ge 1$) and every $j \in \cN$,
  either $(\tilde S_i)_j = 0$ or $\partial_j F(\tilde S_i) > 0$.
\end{lemma}
\begin{proof}
  Direct from the pruning loop termination condition
  (Line~\ref{line:dmcgp-prune}): the loop exits exactly when no coordinate
  has both positive value and non-positive slope.
\end{proof}

Lemma~\ref{lem:dmcgp-slope} is the structural property that replaces the
pointwise curvature argument when $f$ is allowed to be negative.  In
ordinary MCG, $c_F$ can diverge because the definition of $c_F$ ranges
over \emph{all} $\vx \in [0,1]^\cN$, in particular over witnesses of
the form $\vx_\varepsilon = (1-\varepsilon)\vone_{\cN' \setminus \{e'\}}$ from
Proposition~\ref{prop:cF-infty}, which
concentrate mass on inclusion-minimal sets with $f < 0$.
Lemma~\ref{lem:dmcgp-slope} shows that DMCG-P's iterates never lie at such
witnesses: at each $\tilde S_i$, every active coordinate has positive
multilinear slope, so the iterate avoids the neighbourhood where
multilinear curvature blows up.

\begin{remark}[Slope signs after pruning]
  \label{rem:dmcgp-stable}
  A coordinate whose slope is non-positive at one iterate need not keep a
  non-positive slope forever.  Pruning may reduce
  $(\tilde S_i)_k$ for some $k \ne j$, which by submodularity can
  \emph{increase} $\partial_j F$ and flip its sign.  The useful invariant is
  the weaker one stated in Lemma~\ref{lem:dmcgp-slope}: after each pruning
  loop terminates, every coordinate with positive mass has positive slope.
  Within a stretch of iterations without pruning, slopes are monotone
  non-increasing, but the proof below uses only the post-pruning invariant.
\end{remark}

\subsection{General DMCG-P Guarantee}\label{apx:dmcgp-general}

We now prove the headline analytic result: unweighted DMCG-P
(Algorithm~\ref{alg:dmcgp-main}) on an \emph{arbitrary} submodular
objective $f$ over the integral feasible-set hull
$P=\conv\{\vone_B:B\in\cI\}$ matches the classical continuous-greedy guarantee
$(1 - e^{-\bar c_g^F})/\bar c_g^F \cdot f(O^*)$ in the $T \to \infty$ limit,
with $c_g^F$ a trajectory-restricted multilinear analogue of the
greedy curvature $c_g$ of Definition~\ref{def:greedy-curv}.  No
monotonicity or non-negativity of $f$ is assumed; the only
hypothesis is that the trajectory admits a finite curvature
certificate.  Lemma~\ref{lem:dmcgp-slope} supplies the pruning
invariant used in the descent proof, but it is not by itself a
global finiteness theorem for all possible denominators.
The proof is a direct line-by-line lift of the discrete argument for
Theorem~\ref{thm:greedy}: each discrete inequality---pruning, greedy
selection, submodularity, curvature, and the \emph{slope invariant}
on the active set (Lemma~\ref{lem:dmcgp-slope})---has a multilinear
analogue spelled out below, requiring only submodularity of~$f$ or
finite on-trajectory slope, never monotonicity of~$f$.  The
per-step recurrence closes in exactly the same form, modulo a
discretisation correction of size $O(1/T)$.  We state the finite-$T$
bound in the regime $\bar c_g^F\le T$ (equivalently $\delta\bar c_g^F\le 1$
for step size $\delta=1/T$), which is always achievable by choosing
$T$ large enough for any fixed finite trajectory curvature.
This is deliberately different from the conditional MCG comparison in
Appendix~\ref{apx:mcg}: DMCG-P does not assume a global coordinate-zeroing
monotonicity lemma, but enforces the local positive-slope condition needed
for Step~7 by pruning.

\noindent\textbf{Key inequalities.}
The analysis relies on three elementary properties of the multilinear
extension of a submodular $f$.  Let $\vx, \vy \in [0,1]^\cN$ and
$\vd \in \bR_{\ge 0}^\cN$ with $\vx, \vx+\vd \in [0,1]^\cN$.

\medskip
\noindent \emph{(P1) Concavity along non-negative directions.}
The map $t \mapsto F(\vx + t\vd)$ is concave on
$[0, \min_j \{(1-x_j)/d_j : d_j > 0\}]$~\cite[\S2.3]{calinescu2011maximizing}.

\medskip
\noindent \emph{(P2) Tangent bounds.}
By (P1), $F(\vx + \vd) - F(\vx) \le \langle \nabla F(\vx),\,\vd\rangle$
(upper tangent at the left endpoint), and equivalently
$F(\vx + \vd) - F(\vx) \ge \langle \nabla F(\vx + \vd),\,\vd\rangle$
(lower tangent at the right endpoint).

\medskip
\noindent \emph{(P3) Submodular union bound.}
$F(\vx \lor \vy) - F(\vx) \le \sum_{j}(1 - x_j)\,y_j \,\partial_j F(\vx).$
\emph{Proof.}  Telescope coordinate-wise from $\vx$ to $\vx \lor \vy$,
changing coordinate $j$ (for which $y_j > x_j$) from $x_j$ to
$y_j \lor x_j$; the per-coordinate gain is
$(y_j - x_j)_+ \,\partial_j F(\cdot)$ evaluated at the intermediate
point, bounded above by $(y_j - x_j)_+ \,\partial_j F(\vx)$ by
submodularity (cross-partials $\le 0$).  Summing and using
$(y_j - x_j)_+ \le (1 - x_j) y_j$ for $y_j \in [0,1]$ gives (P3).

\noindent\textbf{Fractional greedy curvature.}
The relevant multilinear curvature is not $c_F$ (which can be $+\infty$
by Proposition~\ref{prop:cF-infty})
but its trajectory-restricted analogue, mirroring the discrete greedy
curvature $c_g$ of Definition~\ref{def:greedy-curv}.

\begin{definition}[Fractional greedy curvature along a trajectory;
  restated from Definition~\ref{def:cgF}]
  \label{defn:fractional-curvature}
  Let $F$ be the multilinear extension of a submodular $f$ and let
  $\cO = \argmax_{B \in \cI} f(B)$ denote the set of optimal feasible
  sets.
  For the sequence of iterates $\tilde S_0, \dots, \tilde S_{T-1} \in
  P$ produced by $\dmcgp(F, P, T)$, the \emph{fractional greedy
  curvature} is
  \[
    c_g^F \;=\; 1 \;-\; \min_{O^* \in \cO}\;
      \min_{\substack{0 \le i < T \\ F(\tilde S_i - \tilde S_i \land \vone_{O^*}) > 0}}\;
      \frac{F(\tilde S_i \lor \vone_{O^*}) - F(\vone_{O^*})}
           {F(\tilde S_i - \tilde S_i \land \vone_{O^*})}.
  \]
\end{definition}

\begin{remark}[Finiteness on DMCG-P trajectories]\label{rem:cgF-finite}
  Definition~\ref{defn:fractional-curvature} restricts the minimisation
  to the $T$ trajectory points $\{\tilde S_i\}_{i=0}^{T-1}$, not the
  whole cube, and only positive-denominator pairs are admissible in the
  inner minimum.  Lemma~\ref{lem:dmcgp-slope} ensures that active
  coordinates have strictly positive multilinear slope after pruning,
  which is the structural reason the algorithm avoids the particular
  concentration witnesses used to prove $c_F = \infty$ in
  Proposition~\ref{prop:cF-infty}.  The theorem below therefore keeps
  $c_g^F<\infty$ as an explicit trajectory hypothesis.  In applications
  one proves this hypothesis, or upper-bounds $c_g^F$, using a certificate
  such as Proposition~\ref{prop:dmcgp-decomp-cert}.
\end{remark}

\noindent\textbf{Per-step descent lemma.}
\begin{lemma}[Per-step descent]\label{lem:dmcgp-descent}
  Let $f$ be submodular on $P$ (not necessarily monotone or
  non-negative), fix $O^* \in \cO$, and let $c_g^F \in [0, \infty)$
  be a finite fractional greedy curvature value from
  Definition~\ref{defn:fractional-curvature}. Put
  $\bar c_g^F=\max\{1,c_g^F\}$. Then for every
  $i \in \{0, \dots, T-1\}$,
  \begin{equation}\label{eq:dmcgp-descent}
    F(\tilde S_{i+1}) - F(\tilde S_i)
    \;\ge\; \tfrac{1}{T}\bigl[\,F(\vone_{O^*}) \;-\; \bar c_g^F\,F(\tilde S_i)\,\bigr]
         \;-\; \tfrac{1}{T^2}\,E_i,
  \end{equation}
  where the per-step discretisation error satisfies
  $0 \le E_i \le n(n-1)\,C_F$ and
  $\sum_{i=0}^{T-1}(1/T^2) E_i \le n(n-1) C_F / T$.
\end{lemma}

\begin{proof}
  We mirror the five-line discrete proof of
  Theorem~\ref{thm:greedy}, with each
  step justified by the multilinear analogue.  Throughout, write
  $v_i \in P$ for the LP maximiser selected on
  Line~\ref{line:dmcgp-select}; by construction
  $v_i = \vone_{B_i}$ for some feasible set $B_i \in \cI$, and we use
  this integral representation in Steps~2--3.

  \smallskip
  \noindent \emph{Step 1 (pruning).}
  By Lemma~\ref{lem:dmcgp-prune-up}, $F(\tilde S_{i+1}) \ge F(S_{i+1}) =
  F(\tilde S_i + (1/T)\,v_i)$.

  \smallskip
  \noindent \emph{Step 2 (concavity, right-endpoint tangent).}
  By (P2) applied with $\vx = \tilde S_i$, $\vd = (1/T)\,v_i$
  (both endpoints in $[0,1]^\cN$ by Lemma~\ref{lem:dmcgp-feas}):
  \[
    F\bigl(\tilde S_i + \tfrac{1}{T}\,v_i\bigr) - F(\tilde S_i)
    \;\ge\; \tfrac{1}{T}\,\langle \nabla F(S_{i+1}),\,v_i\rangle.
  \]

  \smallskip
  \noindent \emph{Step 3 (left-to-right gradient substitution).}
  Define the per-step discretisation error
  \begin{equation}\label{eq:dmcgp-Ei}
    E_i \;\triangleq\; T\!\sum_{j \in B_i}
      \bigl[\partial_j F(\tilde S_i) - \partial_j F(S_{i+1})\bigr].
  \end{equation}
  Since $v_i = \vone_{B_i}$, this gives the \emph{exact} identity
  \begin{equation}\label{eq:dmcgp-cross-term}
    \langle \nabla F(S_{i+1}),\,v_i\rangle
    \;=\; \sum_{j \in B_i}\partial_j F(S_{i+1})
    \;=\; \sum_{j \in B_i}\partial_j F(\tilde S_i) \;-\; \tfrac{1}{T} E_i
    \;=\; \langle \nabla F(\tilde S_i),\,v_i\rangle \;-\; \tfrac{1}{T} E_i.
  \end{equation}
  Non-negativity $E_i \ge 0$ holds because submodularity
  ($\partial_{j\ell}F \le 0$ for $j \ne \ell$) makes each
  $\partial_j F$ non-increasing in every other coordinate, so
  $\partial_j F(S_{i+1}) \le \partial_j F(\tilde S_i)$ for every
  $j \in B_i$.  For the upper bound, write
  $S_{i+1} = \tilde S_i + (1/T)\vone_{B_i}$ and apply the
  mean-value theorem coordinate-by-coordinate:
  \[
    \partial_j F(\tilde S_i) - \partial_j F(S_{i+1})
    \;=\; -\sum_{\ell \in B_i \setminus j}
      \frac{1}{T}\,\partial_{j\ell}F(\xi_{j\ell})
    \;\le\; \frac{|B_i|-1}{T}\,C_F,
  \]
  where $\xi_{j\ell}$ lies on the segment $[\tilde S_i, S_{i+1}]$ and
  $C_F \triangleq \max_{j \ne \ell,\,\vx \in [0,1]^\cN}
  |\partial_{j\ell}F(\vx)|$ is finite for any set function on~$\cN$.
  Multiplying by $T$ and summing over $j \in B_i$:
  \[
    E_i \;\le\; |B_i|(|B_i|-1)\,C_F \;\le\; n(n-1)\,C_F.
  \]
  Since $|B_i| \le n$, the total discretisation error is
  $\sum_{i=0}^{T-1}(1/T^2)\,E_i \le n(n-1)\,C_F / T$.

  \smallskip
  \noindent \emph{Step 4 (linear selection, unweighted).}
  Since Line~\ref{line:dmcgp-select} picks $v_i$ to maximise
  $\langle \nabla F(\tilde S_i),\,x\rangle$ over all $x \in P$ and
  $\vone_{O^*} \in P$ (as $O^* \in \cO \subseteq \cI$),
  \[
    \langle \nabla F(\tilde S_i),\,v_i\rangle
    \;\ge\; \langle \nabla F(\tilde S_i),\,\vone_{O^*}\rangle
    \;=\; \sum_{j \in O^*}\partial_j F(\tilde S_i).
  \]
  This is a direct consequence of LP optimality and does not require
  matroid structure.

  \smallskip
  \noindent \emph{Step 5 (submodular union bound, with slope monotonicity).}
  By (P3) applied with $\vy = \vone_{O^*}$:
  \[
    F(\tilde S_i \lor \vone_{O^*}) - F(\tilde S_i)
    \;\le\; \sum_{j \in O^*}(1 - (\tilde S_i)_j)\,\partial_j F(\tilde S_i)
    \;\le\; \sum_{j \in O^*}\partial_j F(\tilde S_i),
  \]
  where the second inequality uses $(\tilde S_i)_j \in [0,1]$ together
  with Lemma~\ref{lem:dmcgp-slope}, which guarantees
  $\partial_j F(\tilde S_i) \ge 0$ whenever $(\tilde S_i)_j > 0$
  (and trivially whenever $(\tilde S_i)_j = 0$ then
  $1 - (\tilde S_i)_j = 1$).  Combining Steps~4 and~5:
  \begin{equation}\label{eq:dmcgp-swap-bound}
    \langle \nabla F(\tilde S_i),\,v_i\rangle
    \;\ge\; F(\tilde S_i \lor \vone_{O^*}) - F(\tilde S_i).
  \end{equation}

  \smallskip
  \noindent \emph{Step 6 (fractional curvature).}
  Let
  \[
    \vw_i \;\triangleq\; \tilde S_i-\tilde S_i\land \vone_{O^*}.
  \]
  If $\vw_i=0$, then
  $\tilde S_i\lor \vone_{O^*}=\vone_{O^*}$ and the displayed inequality
  below is an equality.  Otherwise, every coordinate in $\supp(\vw_i)$
  is active in $\tilde S_i$, so Lemma~\ref{lem:dmcgp-slope} gives
  $\partial_jF(\tilde S_i)>0$ for each $j\in\supp(\vw_i)$.  Since
  $t\vw_i\le \tilde S_i$ for $t\in[0,1]$, submodularity of the
  multilinear extension gives
  $\partial_jF(t\vw_i)\ge \partial_jF(\tilde S_i)>0$ on this support.
  With $F(\mathbf 0)=0$,
  \[
    F(\vw_i)
    = \int_0^1 \langle \nabla F(t\vw_i),\vw_i\rangle\,dt
    >0.
  \]
  Thus, in the nonzero case, the present pair appears in
  Definition~\ref{defn:fractional-curvature}.  Since
  $\bar c_g^F\ge c_g^F$, the curvature inequality is also valid with
  $\bar c_g^F$ in place of $c_g^F$; in the zero case, the same display
  holds trivially.  Hence
  \[
    F(\tilde S_i \lor \vone_{O^*}) - F(\vone_{O^*})
    \;\ge\; (1 - \bar c_g^F)\, F(\vw_i),
  \]
  so
  \[
    F(\tilde S_i \lor \vone_{O^*}) - F(\tilde S_i)
    \;\ge\; F(\vone_{O^*}) + (1 - \bar c_g^F)\,F(\tilde S_i - \tilde S_i \land \vone_{O^*}) - F(\tilde S_i).
  \]

  \smallskip
  \noindent \emph{Step 7 (zeroing the $O^*$-coordinates of $\tilde S_i$ only decreases $F$).}
  The vector $\tilde S_i - \tilde S_i \land \vone_{O^*}$ is obtained
  from $\tilde S_i$ by zeroing every coordinate $j \in O^*$.  Fix any
  order for this sequential zeroing.  When coordinate $j$ is reached, all
  other changes made so far have only decreased coordinates.  By
  submodularity of the multilinear extension, decreasing other coordinates
  can only increase the remaining slope $\partial_jF$; hence
  Lemma~\ref{lem:dmcgp-slope} implies that either the current value of
  coordinate $j$ is zero, or the current slope in direction~$j$ is still
  nonnegative.  Multilinearity in coordinate~$j$ therefore implies that
  zeroing $j$ changes $F$ by
  $-(\text{current }j\text{-mass})\,\partial_jF(\cdot)\le0$.  Iterating
  over $j \in O^*$ gives
  \[
    F(\tilde S_i - \tilde S_i \land \vone_{O^*}) \;\le\; F(\tilde S_i).
  \]
  Substituting this into the previous display and expanding
  $(1 - \bar c_g^F)$:
  \[
    F(\tilde S_i \lor \vone_{O^*}) - F(\tilde S_i)
    \;\ge\; F(\vone_{O^*}) \;-\; \bar c_g^F\,F(\tilde S_i).
  \]

  \smallskip
  \noindent \emph{Assembly.}
  Chaining Steps 1--3, 4--5 (as~\eqref{eq:dmcgp-swap-bound}), and 6--7:
  \begin{align*}
    F(\tilde S_{i+1}) - F(\tilde S_i)
    &\ge\; \tfrac{1}{T}\,\langle \nabla F(S_{i+1}),\,v_i\rangle
      &&\text{(Steps 1--2)}\\
    &\ge\; \tfrac{1}{T}\,\langle \nabla F(\tilde S_i),\,v_i\rangle \;-\; \tfrac{1}{T^2} E_i
      &&\text{(Step 3)}\\
    &\ge\; \tfrac{1}{T}\bigl[F(\tilde S_i \lor \vone_{O^*}) - F(\tilde S_i)\bigr] \;-\; \tfrac{1}{T^2} E_i
      &&\text{(Steps 4--5)}\\
    &\ge\; \tfrac{1}{T}\bigl[F(\vone_{O^*}) - \bar c_g^F\,F(\tilde S_i)\bigr] \;-\; \tfrac{1}{T^2} E_i
      &&\text{(Steps 6--7)}.
  \end{align*}
  This is~\eqref{eq:dmcgp-descent}.
  Finally, $\sum_{i=0}^{T-1}(1/T^2) E_i \le n(n-1)\,C_F/T$
  by the bound on $E_i$ from~\eqref{eq:dmcgp-Ei}; hence the aggregate
  discretisation error is $O(1/T)$ and vanishes as $T \to \infty$.
\end{proof}

\begin{remark}[On the discretisation error]\label{rem:discretisation}
  It is tempting to claim that the per-step equality
  \[
    F\bigl(\tilde S_i + \tfrac{1}{T}\,v_i\bigr) - F(\tilde S_i)
    \;=\; \tfrac{1}{T}\,\langle \nabla F(\tilde S_i),\,v_i\rangle
  \]
  holds exactly, ``because $F$ is multilinear.''  This is false in
  general: multilinearity makes $F$ linear in \emph{each coordinate
  separately}, but along a direction $v_i$ with multiple non-zero
  entries, $t \mapsto F(\vx + t\,v_i)$ is a polynomial of degree up to
  $|\supp(v_i)|$ in $t$, with quadratic coefficients
  $\partial_{j\ell}F \le 0$ (submodularity) and bounded higher-order
  cross-terms.
  The cleanest exact statement is (P2): the secant slope exceeds the
  right-endpoint gradient slope, so
  $F(\vx + t\,v) - F(\vx) \ge t\,\langle \nabla F(\vx + t\,v),\,v\rangle$.
  Replacing the right-endpoint gradient with the left-endpoint one costs
  $O(t) = O(1/T)$ per step via~\eqref{eq:dmcgp-cross-term}, which sums to
  $O(1/T)$ total---vanishing in the limit but not zero at finite $T$.
  The $T \to \infty$ bound is unaffected; the finite-$T$ bound carries
  this additive $O(1/T)$ penalty.
\end{remark}

\begin{theorem}[General DMCG-P guarantee]
  \label{thm:dmcgp-general}
  Let $\cI\subseteq2^\cN$ be downward-closed, let
  $P=\conv\{\vone_B:B\in\cI\}$, and assume exact linear optimization
  over $\cI$.  Let $f : 2^\cN \to \bR$ be submodular (not necessarily
  monotone or non-negative) and let $O^* \in \argmax_{B \in \cI} f(B)$.
  Let $n=|\cN|$, and let $c_g^F \in [0,\infty)$ be a finite fractional greedy curvature
  value for the trajectory of $\dmcgp(F,P,T)$ from
  Definition~\ref{defn:fractional-curvature}.
  Put $\bar c_g^F=\max\{1,c_g^F\}$. If $\bar c_g^F \le T$
  (equivalently $\delta\bar c_g^F\le 1$ for $\delta=1/T$),
  then the output $\tilde S_T$ satisfies
      \[
        F(\tilde S_T)
        \;\ge\; \frac{1 - (1 - \bar c_g^F/T)^T}{\bar c_g^F}\,f(O^*) \;-\; \frac{n(n-1)\,C_F}{T}
        \;\xrightarrow[T \to \infty]{}\; \frac{1 - e^{-\bar c_g^F}}{\bar c_g^F}\,f(O^*).
      \]
  The step count $T$ is independent of the cardinality or matroid-rank
  constraint; choosing $T$ large reduces the discretization error
  at the cost of $O(Tn)$ oracle calls.
  The sharper small-curvature formula is recovered in
  the monotone setting by the classical analysis.
\end{theorem}

\begin{proof}
  Let $a_i \triangleq F(\tilde S_i)$, $A \triangleq F(\vone_{O^*}) = f(O^*)$,
  put $\bar c_g^F=\max\{1,c_g^F\}$, and let
  $\delta_i \triangleq (1/T^2) E_i$ from
  Lemma~\ref{lem:dmcgp-descent}.  Equation~\eqref{eq:dmcgp-descent}
  (which holds for arbitrary submodular $f$, without monotonicity) gives
  the linear recurrence
  \[
    a_{i+1} \;\ge\; \tfrac{1}{T} A \;+\; \bigl(1 - \bar c_g^F/T\bigr)\,a_i \;-\; \delta_i,
    \qquad a_0 = 0.
  \]

  Because $\bar c_g^F \le T$, the recurrence coefficient
  $1 - \bar c_g^F/T \ge 0$.  Iterating:
  \[
    a_T \;\ge\; \tfrac{A}{T}\sum_{i=0}^{T-1}\bigl(1 - \bar c_g^F/T\bigr)^i
          \;-\; \sum_{i=0}^{T-1}\bigl(1 - \bar c_g^F/T\bigr)^{T-1-i}\,\delta_i.
  \]
  The geometric sum evaluates to
  $\tfrac{A}{T}\cdot \tfrac{1 - (1 - \bar c_g^F/T)^T}{\bar c_g^F/T}
   = \tfrac{A}{\bar c_g^F}\bigl(1 - (1 - \bar c_g^F/T)^T\bigr)$,
  and each factor $(1 - \bar c_g^F/T)^{T-1-i} \in [0, 1]$ (since
  $\bar c_g^F/T \in [0, 1]$), so the error term is bounded by
  $\sum_i \delta_i = O(1/T)$.  This gives the stated finite-$T$
  bound with $\sum_i\delta_i\le n(n-1)\,C_F/T$.  Sending $T \to \infty$ for
  fixed finite $\bar c_g^F$,
  $(1 - \bar c_g^F/T)^T \to e^{-\bar c_g^F}$ by the elementary exponential limit,
  the error vanishes, and the limit is
  $(1 - e^{-\bar c_g^F})/\bar c_g^F \cdot f(O^*)$.
\end{proof}

\begin{remark}[Reduction to classical CC curvature]\label{rem:classical-reduction}
  When $f$ is monotone submodular (so $c_g^F \le c_F = c_f = \alpha_g
  \in [0,1]$ by Theorem~\ref{thm:nn-characterization}(a), with
  $\alpha_g$ the classical Conforti--Cornu\'ejols curvature),
  the classical continuous-greedy analysis yields
  $F(\tilde S_T) \ge (1 - e^{-\alpha_g})/\alpha_g \cdot f(O^*)$ in the
  limit $T \to \infty$.  The non-monotone DMCG-P recurrence above uses
  $\bar c_g^F=\max\{1,c_g^F\}$; the sharper small-curvature formula
  comes from monotonicity, not from Step~7 of the non-monotone proof.
  When $f$ takes negative values, $c_F = \infty$ globally by
  Proposition~\ref{prop:cF-infty}; Theorem~\ref{thm:dmcgp-general}
  still applies whenever its finite trajectory-curvature
  hypothesis is verified.  For decomposable objectives,
  \S\ref{apx:dmcgp-fractional} supplies a closed-form certificate for
  that hypothesis.
\end{remark}

\noindent\textbf{Why this works for unweighted greedy.}
The unweighted selection rule is sufficient because the proof needs an
upper bound on the union gain, not a lower bound on a weighted gradient.
(P3) gives
$F(\tilde S_i \lor \vone_{O^*}) - F(\tilde S_i) \le \sum_{j \in O^*}(1 - (\tilde S_i)_j)\partial_j F(\tilde S_i)$,
and the analysis then upper-bounds this weighted sum by
the unweighted sum $\sum_{j \in O^*}\partial_j F(\tilde S_i) = \langle
\nabla F(\tilde S_i),\,\vone_{O^*}\rangle$.  The latter is free when
$\partial_j F(\tilde S_i) \ge 0$ for every $j \in O^*$---which is
exactly what Lemma~\ref{lem:dmcgp-slope} plus the $O^* \subseteq \cN$
case analysis in Step 5 of the proof provides (coordinates $j \in O^*$
with $(\tilde S_i)_j = 0$ have $1 - (\tilde S_i)_j = 1$ trivially;
coordinates with $(\tilde S_i)_j > 0$ have $\partial_j F \ge 0$ by
Lemma~\ref{lem:dmcgp-slope}, so
$(1 - (\tilde S_i)_j)\partial_j F \le \partial_j F$ regardless).
Thus unweighted DMCG-P---the feasible-set form
$B_i \in \arg\max_{B\in\cI}\sum_{j\in B}\partial_jF(\tilde S_i)$ with
$v_i=\vone_{B_i}$---is the correct clean algorithm and the discrete
proof lifts directly.

\subsection{Decomposable Certificate for \texorpdfstring{$c_g^F$}{cgF}}\label{apx:dmcgp-fractional}

Theorem~\ref{thm:dmcgp-general} delivers the headline DMCG-P
guarantee $(1 - e^{-\bar c_g^F})/\bar c_g^F \cdot f(O^*)$ conditional on
$c_g^F < \infty$ and $\delta\bar c_g^F \le 1$.  Lemma~\ref{lem:dmcgp-slope}
provides the positive-slope invariant used in the proof, but the
theorem is most useful when
$c_g^F$ can be \emph{bounded} in closed form.  This subsection
provides such a bound for the \emph{decomposable} case $f = g -
\ell$, mirroring the discrete removal-marginal certificate
$c_g \le \alpha_g/(1 - \hat r)$
of Proposition~\ref{prop:opt-free}.  The content here is a
\emph{computational certificate for $c_g^F$}, not a separate
approximation guarantee: plugging the certificate into
Theorem~\ref{thm:dmcgp-general} recovers the $(1-e^{-\bar c_g^F})/\bar c_g^F$
bound for negative-valued~$f$.
We first fix the setting
(\S\ref{apx:dmcgp-decomp-setup}), then prove the multilinear lift of
the CC-curvature inequality that the analysis requires
(\S\ref{apx:dmcgp-multilinear-cc}), state the OPT-aware comparison bound
(\S\ref{apx:dmcgp-opt-aware}), obtain the OPT-free certificate
$c_g^F \le \alpha_g/(1 - \hat r_F)$ from the pruned trajectory
(\S\ref{apx:dmcgp-opt-free}), and assemble the decomposable-certificate
proposition (\S\ref{apx:dmcgp-frac-main}).

\subsubsection{Setting}\label{apx:dmcgp-decomp-setup}

Throughout this subsection, $f = g - \ell$ where
\begin{itemize}
  \item $g : 2^\cN \to \bR_{\ge 0}$ is monotone submodular with
    Conforti--Cornu\'ejols curvature $\alpha_g \in [0, 1]$,
    characterised
    (Definition~\ref{def:curv} and equation~\eqref{eq:curvature-ineq})
    by
    \begin{equation}\label{eq:dmcgp-cc-set}
      g(X \cup Y) - g(Y) \;\ge\; (1 - \alpha_g)\,g(X \setminus Y)
      \qquad \forall\, X, Y \subseteq \cN;
    \end{equation}
  \item $\ell : 2^\cN \to \bR_{\ge 0}$ is non-negative modular with
    per-element costs $\ell_j \triangleq \ell(\{j\})$, so
    $\ell(A) = \sum_{j \in A}\ell_j$ for every $A \subseteq \cN$.
\end{itemize}
The multilinear extensions decompose as $F = G - L$, with $G$ the
multilinear extension of~$g$ (monotone submodular) and
\begin{equation}\label{eq:dmcgp-Ldef}
  L(\vx) \;=\; \sum_{j \in \cN}\ell_j x_j
\end{equation}
exactly linear in $\vx$.  Note that $f$ itself may take negative
values, so $c_F = \infty$ by
Proposition~\ref{prop:cF-infty}; the
analysis below operates entirely through the trajectory-restricted
curvature $c_g^F$ of Definition~\ref{defn:fractional-curvature}, which
is bounded by the decomposable certificates proved below.

\subsubsection{Multilinear CC-curvature inequality for \texorpdfstring{$G$}{G}}
\label{apx:dmcgp-multilinear-cc}

The multilinear extension $G$ of a monotone submodular $g$ with CC
curvature $\alpha_g$ inherits a fractional form of the set
inequality~\eqref{eq:dmcgp-cc-set}.  We state and prove the precise
form required for the decomposable analysis.  Related multilinear
inequalities are standard
(e.g.~\cite{calinescu2011maximizing}),
but the statement below---comparing $G(\vx \lor \vy)$ with
$G(\vx - \vx \land \vy)$---does not seem to appear in this form in
the literature, so we include a self-contained proof.

\begin{lemma}[Multilinear CC-curvature inequality, elementwise form]
  \label{lem:dmcgp-multilinear-cc}
  Let $g : 2^\cN \to \bR_{\ge 0}$ be monotone submodular with CC
  curvature $\alpha_g \in [0,1]$, and let $G$ be its multilinear
  extension.  For all $\vx, \vy \in [0,1]^\cN$, writing
  $\vw \triangleq \vx - \vx \land \vy$ (so $w_j = (x_j - y_j)_+$),
  \begin{equation}\label{eq:dmcgp-mlcc-elem}
    G(\vx \lor \vy) - G(\vy)
    \;\ge\; (1 - \alpha_g)\sum_{j \in \cN} w_j\,g(\{j\}).
  \end{equation}
  In particular, by $g$-subadditivity
  $G(\vw) \le \sum_j w_j\,g(\{j\})$, the weaker set-style inequality
  \begin{equation}\label{eq:dmcgp-mlcc}
    G(\vx \lor \vy) - G(\vy) \;\ge\; (1 - \alpha_g)\,G(\vw)
  \end{equation}
  follows by combining the elementwise bound with subadditivity.
\end{lemma}

\begin{proof}
  The elementwise bound~\eqref{eq:dmcgp-mlcc-elem} follows from the
  pointwise gradient inequality
  \begin{equation}\label{eq:dmcgp-partial-lb-global}
    \partial_j G(\vz) \;\ge\; (1 - \alpha_g)\,g(\{j\})
    \qquad \text{for every } \vz \in [0,1]^\cN \text{ and every } j \in \cN,
  \end{equation}
  integrated along the straight-line interpolation from $\vy$ to
  $\vx \lor \vy$.

  \emph{Step 1 (pointwise gradient lower bound).}
  Submodularity of $g$ is equivalent to the mixed-partial condition
  $\partial_{j\ell}G \le 0$ for $\ell \ne j$, i.e.\ $\partial_j G$ is
  coordinate-wise non-increasing in every other coordinate.  Hence for
  any $\vz \in [0,1]^\cN$,
  \[
    \partial_j G(\vz)
    \;\ge\; \partial_j G(\vone_\cN)
    \;=\; g(\cN) - g(\cN \setminus \{j\})
    \;\ge\; (1 - \alpha_g)\,g(\{j\}),
  \]
  where the final inequality is the CC-curvature definition
  $\alpha_g = 1 - \min_{e, S\not\ni e}\Delta_g(e \mid S)/g(\{e\})$
  specialised to $S = \cN \setminus \{j\}$.  This
  gives~\eqref{eq:dmcgp-partial-lb-global} (normalisation of $g$ is not
  needed).

  \emph{Step 2 (telescoping via multilinearity).}
  Order $\cN = \{1, \dots, n\}$ arbitrarily and interpolate from $\vy$
  to $\vx \lor \vy$ one coordinate at a time: set $\vz^{(0)} = \vy$
  and $\vz^{(j)} = \vz^{(j-1)} + w_j\,\ve_j$, so
  $\vz^{(n)} = \vy + \vw = \vx \lor \vy$ (using
  $\max(x_j, y_j) = y_j + (x_j-y_j)_+$).  By multilinearity of $G$ in
  each coordinate separately,
  \[
    G(\vz^{(j)}) - G(\vz^{(j-1)}) \;=\; w_j\,\partial_j G(\vz^{(j-1)}),
  \]
  and summing,
  \begin{equation}\label{eq:dmcgp-telescope}
    G(\vx \lor \vy) - G(\vy) \;=\; \sum_{j=1}^{n} w_j\,\partial_j G(\vz^{(j-1)}).
  \end{equation}
  Applying~\eqref{eq:dmcgp-partial-lb-global} at each
  $\vz^{(j-1)} \in [0,1]^\cN$ gives~\eqref{eq:dmcgp-mlcc-elem}.

  \emph{Step 3 (set-style corollary).}
  By $g \ge 0$ and submodularity, $g$ is subadditive:
  $g(A) \le \sum_{e \in A} g(\{e\})$.  Taking expectations under
  $R \sim \vw$ (coordinate-wise independent Bernoullis),
  $G(\vw) = \bE[g(R)] \le \bE\bigl[\sum_{j \in R}g(\{j\})\bigr]
  = \sum_j w_j\,g(\{j\})$.  Combining with~\eqref{eq:dmcgp-mlcc-elem}
  yields~\eqref{eq:dmcgp-mlcc}.
\end{proof}

\begin{remark}[Strength of the elementwise form]\label{rem:dmcgp-mlcc-tight}
  The elementwise inequality~\eqref{eq:dmcgp-mlcc-elem} is strictly
  stronger than the set-style form~\eqref{eq:dmcgp-mlcc}, since the
  subadditivity step $G(\vw) \le \sum_j w_j\,g(\{j\})$ can be strict.
  Both forms are used downstream: \eqref{eq:dmcgp-mlcc} is the clean
  multilinear analogue of~\eqref{eq:dmcgp-cc-set}, while
  \eqref{eq:dmcgp-mlcc-elem} is what drives the fractional OPT-free
  certificate in~\S\ref{apx:dmcgp-opt-free}.
\end{remark}

\subsubsection{OPT-aware decomposable bound}\label{apx:dmcgp-opt-aware}

For each iterate write
\begin{equation}\label{eq:dmcgp-wi}
  \vw_i \;\triangleq\; \tilde S_i - \tilde S_i \land \vone_{O^*}
  \qquad \text{(coordinate $j$: $(\vw_i)_j = (\tilde S_i)_j\,\one[j \notin O^*]$),}
\end{equation}
the ``outside-OPT'' fractional mass.  Since $O^*$ is a set and
$\vone_{O^*}$ is $\{0,1\}$-valued, $\tilde S_i \land \vone_{O^*}$ zeros
out the coordinates $j \in \cN \setminus O^*$, so
$\vw_i = \tilde S_i \odot \one[\cdot \notin O^*]$ is literally
$\tilde S_i$ restricted to $\cN \setminus O^*$.  Define the
trajectory-level cost ratio
\begin{equation}\label{eq:dmcgp-rhoF}
  \rho_F \;\triangleq\; \max_{\substack{O^* \in \cO \\ 0 \le i < T \\ G(\vw_i) > 0}}\,
    \frac{L(\vw_i)}{G(\vw_i)}.
\end{equation}
(If $G(\vw_i) = 0$ for every $i, O^*$, set $\rho_F = 0$; the curvature
definition then does not involve any active constraint from the
decomposable side.)

\begin{proposition}[OPT-aware fractional curvature bound]
  \label{prop:dmcgp-frac-opt-aware}
  Let $f = g - \ell$ be decomposable as
  in~\S\ref{apx:dmcgp-decomp-setup}, and consider the trajectory of
  $\dmcgp(F, P, T)$.  If $\rho_F < 1$, then the fractional greedy
  curvature of Definition~\ref{defn:fractional-curvature} satisfies
  \begin{equation}\label{eq:dmcgp-frac-opt-aware}
    c_g^F \;\le\; \frac{\alpha_g}{1 - \rho_F}.
  \end{equation}
\end{proposition}

\begin{proof}
  Fix $O^* \in \cO$ and $i \in \{0, \dots, T-1\}$ with
  $F(\vw_i) > 0$ (the only indices active in
  Definition~\ref{defn:fractional-curvature}); write $\vx = \tilde S_i$
  and $\vy = \vone_{O^*}$, so $\vw_i = \vx - \vx \land \vy$.  Decompose
  $F = G - L$ and expand the numerator of the curvature ratio:
  \begin{align*}
    F(\vx \lor \vy) - F(\vy)
    &\;=\; \bigl[\,G(\vx \lor \vy) - G(\vy)\,\bigr]
         \;-\; \bigl[\,L(\vx \lor \vy) - L(\vy)\,\bigr]\\
    &\;\ge\; (1 - \alpha_g)\,G(\vw_i)
         \;-\; \bigl[\,L(\vx \lor \vy) - L(\vy)\,\bigr]
      && \text{by Lemma~\ref{lem:dmcgp-multilinear-cc}.}
  \end{align*}
  For the $L$ side, linearity of $L$ gives
  $L(\vx \lor \vy) + L(\vx \land \vy) = L(\vx) + L(\vy)$
  (coordinate-wise: $\max(x_j, y_j) + \min(x_j, y_j) = x_j + y_j$), so
  \begin{equation}\label{eq:dmcgp-L-gap}
    L(\vx \lor \vy) - L(\vy) \;=\; L(\vx) - L(\vx \land \vy)
    \;=\; L(\vx - \vx \land \vy) \;=\; L(\vw_i).
  \end{equation}
  Let $\rho_i \triangleq L(\vw_i)/G(\vw_i)$ (finite by
  $G(\vw_i) > 0$, since $F(\vw_i) > 0$ and $L \ge 0$ together imply
  $G(\vw_i) > 0$).  Then
  $F(\vw_i) = G(\vw_i) - L(\vw_i) = (1 - \rho_i)\,G(\vw_i)$, and
  \[
    F(\vx \lor \vy) - F(\vy)
    \;\ge\; \bigl[(1-\alpha_g) - \rho_i\bigr]\,G(\vw_i).
  \]
  Dividing by $F(\vw_i) = (1-\rho_i)\,G(\vw_i) > 0$,
  \begin{equation}\label{eq:dmcgp-ratio-per-i}
    \frac{F(\vx \lor \vy) - F(\vy)}{F(\vw_i)}
    \;\ge\; \frac{(1-\alpha_g) - \rho_i}{1 - \rho_i}
    \;=\; 1 - \frac{\alpha_g}{1 - \rho_i}
    \;\ge\; 1 - \frac{\alpha_g}{1 - \rho_F},
  \end{equation}
  using $\rho_i \le \rho_F < 1$ and monotonicity of
  $t \mapsto \alpha_g/(1-t)$ on $[0, 1)$.  Taking the infimum over
  $i$ and $O^* \in \cO$ and applying
  Definition~\ref{defn:fractional-curvature} gives
  $c_g^F \le \alpha_g/(1 - \rho_F)$.
\end{proof}

\subsubsection{OPT-free fractional certificate}
\label{apx:dmcgp-opt-free}

Proposition~\ref{prop:dmcgp-frac-opt-aware} bounds $c_g^F$ in terms of
the OPT-dependent vectors $\vw_i$.  For a deployable certificate we
want an OPT-free quantity computed directly from the DMCG-P trajectory.
The right fractional analogue of the discrete removal marginal
\(\Delta_g(e\mid A_i\setminus\{e\})\) is the current multilinear
slope \(\partial_jG(\tilde S_i)\).

Define the removal-slope ratio
\begin{equation}\label{eq:dmcgp-rhatF}
  \hat r_F \;\triangleq\;
  \max_{\substack{0\le i<T\\ j:\,(\tilde S_i)_j>0}}
    \frac{\ell_j}{\partial_jG(\tilde S_i)},
\end{equation}
with value~$0$ if the index set is empty.
Lemma~\ref{lem:dmcgp-slope} gives
\(\partial_jF(\tilde S_i)=\partial_jG(\tilde S_i)-\ell_j>0\)
on every active coordinate, so every denominator in
\eqref{eq:dmcgp-rhatF} is positive and \(\hat r_F<1\).

\begin{lemma}[Removal-slope domination]\label{lem:dmcgp-removal-domination}
  Let $\vx=\tilde S_i$ be a DMCG-P iterate and let
  $\vw\in[0,1]^\cN$ satisfy \(0\le \vw\le \vx\) coordinate-wise and
  \(\supp(\vw)\subseteq\supp(\vx)\).  Then
  \[
    L(\vw) \;\le\; \hat r_F\,G(\vw).
  \]
\end{lemma}

\begin{proof}
  By definition of \(\hat r_F\), for every coordinate with \(w_j>0\),
  \(\ell_j\le \hat r_F\,\partial_jG(\vx)\).  Hence
  \[
    L(\vw)=\sum_j w_j\ell_j
    \;\le\; \hat r_F\sum_j w_j\partial_jG(\vx).
  \]
  It remains to compare the weighted slope sum with \(G(\vw)\).
  Since \(G\) is the multilinear extension of a monotone submodular
  function, it is DR-submodular: each partial derivative is
  coordinate-wise non-increasing.  For \(t\in[0,1]\), we have
  \(t\vw\le \vw\le \vx\), so
  \(\partial_jG(t\vw)\ge \partial_jG(\vx)\).  Using normalization
  \(G(\vzero)=0\) and integrating along the ray \(t\vw\),
  \[
    G(\vw)
    =\int_0^1 \langle\nabla G(t\vw),\vw\rangle\,dt
    \;\ge\; \sum_j w_j\partial_jG(\vx).
  \]
  Combining the two displays gives the claim.
\end{proof}

\begin{theorem}[OPT-free fractional curvature certificate]
  \label{thm:dmcgp-frac-opt-free}
  Let $f = g - \ell$ be decomposable as in~\S\ref{apx:dmcgp-decomp-setup}
  and consider the trajectory of $\dmcgp(F, P, T)$.  The fractional greedy curvature of
  Definition~\ref{defn:fractional-curvature} satisfies
  \begin{equation}\label{eq:dmcgp-frac-opt-free}
    c_g^F \;\le\; \frac{\alpha_g}{1 - \hat r_F}.
  \end{equation}
\end{theorem}

\begin{proof}
  Fix $O^* \in \cO$ and $i \in \{0, \dots, T-1\}$ with $F(\vw_i) > 0$,
  and write $\vx = \tilde S_i$, $\vy = \vone_{O^*}$, so
  $\vw_i = \vx - \vx \land \vy$.  Since \(0\le \vw_i\le \vx\) and
  \(\supp(\vw_i)\subseteq\supp(\vx)\),
  Lemma~\ref{lem:dmcgp-removal-domination} gives
  \begin{equation}\label{eq:dmcgp-L-removal-bound}
    L(\vw_i)\le \hat r_F\,G(\vw_i).
  \end{equation}
  The set-style multilinear CC
  inequality~\eqref{eq:dmcgp-mlcc} applied to $\vx, \vy$ gives
  \begin{equation}\label{eq:dmcgp-G-removal-lb}
    G(\vx \lor \vy) - G(\vy)
    \;\ge\; (1 - \alpha_g)G(\vw_i),
  \end{equation}
  and modularity of $\ell$ (as in~\eqref{eq:dmcgp-L-gap}) gives
  $L(\vx \lor \vy) - L(\vy) = L(\vw_i)$.
  Let \(\tau \triangleq L(\vw_i)/G(\vw_i)\).  The positive-denominator
  assumption \(F(\vw_i)>0\) implies \(\tau<1\), and
  \eqref{eq:dmcgp-L-removal-bound} implies \(0\le\tau\le\hat r_F\).
  Then
  \begin{align*}
    F(\vx \lor \vy) - F(\vy)
    &\;=\; \bigl[G(\vx \lor \vy) - G(\vy)\bigr] - L(\vw_i) \\
    &\;\ge\; (1 - \alpha_g)\,G(\vw_i) - L(\vw_i)
    \;=\; \bigl[(1-\alpha_g) - \tau\bigr]\,G(\vw_i),
  \end{align*}
  while $F(\vw_i) = G(\vw_i) - L(\vw_i) = (1-\tau)G(\vw_i)$.
  Dividing,
  \[
    \frac{F(\vx \lor \vy) - F(\vy)}{F(\vw_i)}
    \;\ge\; \frac{(1-\alpha_g) - \tau}{1 - \tau}.
  \]
  The right-hand side is decreasing in \(\tau\), since its derivative is
  \(-\alpha_g/(1-\tau)^2\).  Because \(\tau\le\hat r_F\),
  \[
    \frac{F(\vx \lor \vy) - F(\vy)}{F(\vw_i)}
    \;\ge\; 1 - \frac{\alpha_g}{1 - \hat r_F}.
  \]
  Taking the infimum over $i$ and $O^* \in \cO$ and applying
  Definition~\ref{defn:fractional-curvature}
  yields~\eqref{eq:dmcgp-frac-opt-free}.
\end{proof}

\begin{remark}[Relationship to the discrete certificate]
  \label{rem:dmcgp-discrete-parallel}
  Theorem~\ref{thm:dmcgp-frac-opt-free} is the exact multilinear
  analogue of the discrete removal-marginal certificate for decomposable
  objectives.
  In the discrete-greedy limit
  ($\tilde S_i \to \vone_{A_i}$ integer-valued),
  $\partial_jG(\tilde S_i)$ reduces to
  \(\Delta_g(j\mid A_i\setminus\{j\})\), and
  \eqref{eq:dmcgp-frac-opt-free} reduces to
  \(c_g \le \alpha_g/(1-\hat r)\).
\end{remark}

\begin{remark}[Verifiability of $\hat r_F$]
  \label{rem:dmcgp-frac-verifiable}
  The certificate is computed after the run from the active coordinates
  and the component gradients \(\partial_jG(\tilde S_i)\).  No optimum
  is needed.  The curvature parameter \(\alpha_g\) admits closed-form
  expressions for
  the application classes in this paper
  (Appendix~\ref{apx:certificate}).  When $\alpha_g$ is not known
  analytically, the total curvature
  $\alpha_g^{\mathrm{total}} = 1 - \min_e\,\Delta_g(e \mid N \setminus \{e\})/g(\{e\})$
  is a computable upper bound that can be substituted for \(\alpha_g\).
\end{remark}

\subsubsection{Decomposable certificate proposition}\label{apx:dmcgp-frac-main}

Combining Theorem~\ref{thm:dmcgp-frac-opt-free} with the headline
Theorem~\ref{thm:dmcgp-general} yields the main deployable result of
this section.  Proposition~\ref{prop:dmcgp-frac-opt-aware} remains a
sharper OPT-aware comparison when the optimum is available.

\begin{proposition}[Decomposable certificate for $c_g^F$]
  \label{prop:dmcgp-decomp-cert}
  Let $f = g - \ell$ be decomposable as
  in~\S\ref{apx:dmcgp-decomp-setup}, let $O^* \in \cO$, and consider
  the trajectory of $\dmcgp(F, P, T)$.  Then the OPT-free certificate
  \[
    c_g^F \le \frac{\alpha_g}{1-\hat r_F}
  \]
  holds, with \(\hat r_F<1\) computed from
  \eqref{eq:dmcgp-rhatF}.  Plugging this certificate into
  Theorem~\ref{thm:dmcgp-general} and using monotonicity of
  $c \mapsto (1 - e^{-c})/c$ on $c > 0$ yields the approximation
  bound
  \begin{equation}\label{eq:dmcgp-frac-main}
    F(\tilde S_T)
    \;\ge\; \frac{1 - (1 - \bar c_g^F/T)^T}{\bar c_g^F}\,f(O^*) \;-\; O(1/T)
    \;\xrightarrow[T\to\infty]{}\;
    \frac{1 - e^{-\bar c_g^F}}{\bar c_g^F}\,f(O^*),
  \end{equation}
  and in particular
  \begin{equation}\label{eq:dmcgp-frac-cert}
    F(\tilde S_T)
    \;\ge\; \frac{1 - e^{-\bar c_{\mathrm{cert}}}}{\bar c_{\mathrm{cert}}}\,f(O^*) \;-\; o(1),
    \qquad
    \hat c_g^F \;\triangleq\; \frac{\alpha_g}{1 - \hat r_F},\quad
    \bar c_{\mathrm{cert}}\triangleq\max\{1,\hat c_g^F\}.
  \end{equation}
  No monotonicity or non-negativity of $f$ itself is required; in
  particular, $f$ may take negative values (where $c_F = \infty$
  globally and the classical MCG analysis fails).
\end{proposition}

\begin{proof}
  The curvature bound is Theorem~\ref{thm:dmcgp-frac-opt-free}.
  The approximation bound~\eqref{eq:dmcgp-frac-main} is
  Theorem~\ref{thm:dmcgp-general} applied to $f$ with $c_g^F$
  upper-bounded by this certificate.
  Equation~\eqref{eq:dmcgp-frac-cert} follows from monotonicity of
  $c \mapsto (1 - e^{-c})/c$ on $c > 0$ (decreasing), so an
  \emph{upper} bound on $\bar c_g^F$ yields a \emph{lower} bound on
  $(1-e^{-\bar c_g^F})/\bar c_g^F$.  Finally, the hypothesis of
  Theorem~\ref{thm:dmcgp-general} is finite trajectory curvature,
  made concrete here by the closed-form certificate (and by taking
  $T$ in the stated discretization regime); no monotonicity of $f$ is
  invoked at any step of
  Theorem~\ref{thm:dmcgp-general}'s proof
  (see Lemma~\ref{lem:dmcgp-descent}, which is stated and proved for
  arbitrary submodular $f$).
\end{proof}

\begin{remark}[Reductions and comparisons]\label{rem:dmcgp-comparisons}
  \textbf{(i) Purely submodular case ($\ell \equiv 0$).}  Then
  $\rho_F = 0$, $\hat r_F = 0$, and
  Proposition~\ref{prop:dmcgp-decomp-cert}
  gives $c_g^F \le \alpha_g$, so the approximation
  bound~\eqref{eq:dmcgp-frac-main} collapses to
  $F(\tilde S_T) \ge (1 - e^{-\alpha_g})/\alpha_g\,f(O^*) - o(1)$,
  recovering both the classical CC-curvature guarantee
  Theorem~\ref{thm:greedy} and the monotone specialisation of the
  headline Theorem~\ref{thm:dmcgp-general} (with $c_g^F \le \alpha_g$
  from Theorem~\ref{thm:nn-characterization}(a)).

  \textbf{(ii) Negative $f$ is allowed.}  No non-negativity
  assumption on $f = g - \ell$ is made, and monotonicity is likewise
  not assumed.  The relevant finiteness of $c_g^F$ comes from
  the certificate \(\hat r_F<1\) (or the sharper OPT-aware condition
  \(\rho_F<1\)),
  together with the pruning slope invariant from
  Lemma~\ref{lem:dmcgp-slope}; plugging the
  certificate into Theorem~\ref{thm:dmcgp-general} yields the
  approximation bound.

  \textbf{(iii) Comparison to the Harshaw--Feldman--Ward--Karbasi
  additive bound.}  The additive guarantee of
  \cite{harshaw2019submodular} for $g - \ell$ (and its continuous
  analogue via Feldman--Naor--Schwartz) is, in multilinear form,
  $F(\tilde S_T) \ge (1 - 1/e)\,G(\vone_{O^*}) - L(\vone_{O^*})$.
  Writing $\rho^* \triangleq L(\vone_{O^*})/G(\vone_{O^*})$
  (the OPT-level cost ratio), this additive bound is non-vacuous
  only when $\rho^* < 1 - 1/e$ and collapses to $0$ at
	  $\rho^* = 1 - 1/e$.  The multiplicative bound
	  \eqref{eq:dmcgp-frac-cert} is strictly positive whenever
	  \(\hat r_F<1\) and \(f(O^*)>0\), and in particular it can remain
	  informative in regimes where the additive bound is vacuous.
\end{remark}

\subsection{Non-Negative Non-Monotone Regime via Damped wDMCG-P}
\label{apx:dmcgp-nonneg}

The headline DMCG-P guarantee (Theorem~\ref{thm:dmcgp-general}) and
the decomposable certificate
(Proposition~\ref{prop:dmcgp-decomp-cert}) together cover every
submodular $f$ with finite trajectory curvature $c_g^F$---including
$f$ that takes negative values.  In this section we treat the
complementary regime: \emph{non-negative} submodular $f$ with no
further structure (neither monotone nor decomposable), where
$c_g^F$ need not be bounded and the goal is to recover a
$1/e$-type multiplicative guarantee via a measured-continuous-greedy
style trajectory.  The algorithm is a weighted, damped variant of
DMCG-P (wDMCG-P-damped), structurally distinct from the unweighted
DMCG-P of \S\ref{apx:dmcgp-alg}: both the weighted LP selection and the
$(1 - \tilde S_{i,j})$ damping factor are new ingredients.

\subsubsection{Algorithm}

Fix step count $k \in \bN$ and damping step size $\delta = 1/k$.

\begin{algorithm}[h]
  \DontPrintSemicolon
  \caption{$\wdmcgp(F, P, k)$ --- wDMCG-P with damped update}
  \label{alg:wdmcgp-damped}
  \KwIn{multilinear oracle $F$ of non-negative submodular $f$,
    downward-closed solvable polytope $P \subseteq [0,1]^\cN$ with
    feasible-set family $\cI$, step count $k \in \bN$}
  $\tilde S_0 \gets \vzero$;\quad $\delta \gets 1/k$\;
  \For{$i \gets 0$ \textbf{to} $k-1$}{
    \tcp*[l]{Pruning: zero every coordinate with non-positive slope.}
    \While{$\exists\, j \in \cN$ with $(\tilde S_i)_j > 0$ and
           $\partial_j F(\tilde S_i) \le 0$\label{line:wdmcgp-prune}}{
      $(\tilde S_i)_j \gets 0$\;
    }
    \tcp*[l]{Weighted LP selection over feasible sets.}
    $B_i \gets \arg\max_{B \in \cI}
       \sum_{j \in \cN}(1 - (\tilde S_i)_j)\,\partial_j F(\tilde S_i)\,\mathbf{1}_{B,j}$\;
       \label{line:wdmcgp-select}
    \tcp*[l]{Damped update: residual-room step.}
    $\tilde S_{i+1,j} \gets \tilde S_{i,j}
       + \delta\,(1 - \tilde S_{i,j})\,\mathbf{1}_{B_i,j}
       \quad \forall j \in \cN$\;\label{line:wdmcgp-update}
  }
  \KwRet{$\tilde S_k$}\;
\end{algorithm}

Two remarks on Algorithm~\ref{alg:wdmcgp-damped}.
\begin{enumerate}
  \item[(i)] The pruning step is kept for algorithmic consistency
    with DMCG-P but plays no load-bearing role in the analysis below:
    the descent lemma uses only the weighted LP optimality of $B_i$
    and the FNS correlation-gap bound (Lemma~2.2
    of~\cite{feldman2011unified}), neither of which requires pruning.
    In the proof, $\tilde S_i$ denotes the point after this optional pruning
    at the start of iteration~$i$; pruning can only increase $F$ and preserves
    feasibility, so this convention is conservative for the value recurrence.
    One may drop the while-loop entirely without affecting
    Theorem~\ref{thm:wdmcgp-nonneg}.
  \item[(ii)] The damped update on Line~\ref{line:wdmcgp-update} is
    the defining structural change from DMCG-P: each coordinate
    $j \in B_i$ moves towards $1$ by a $\delta$-fraction of its
    \emph{remaining room} $(1 - \tilde S_{i,j})$, rather than by the
    fixed increment $1/T$ used by unweighted DMCG-P.  This is the
    multilinear analogue of the Feldman--Naor--Schwartz measured
    continuous greedy update.
\end{enumerate}

\subsubsection{Feasibility and coordinate bound}

\begin{lemma}[Feasibility]\label{lem:wdmcgp-damped-feasibility}
  For every $i \in \{0, 1, \dots, k\}$, $\tilde S_i \in P$ and
  $\tilde S_i \in [0, 1]^\cN$.
\end{lemma}
\begin{proof}
  For each coordinate $j$, let $m_{i,j} \triangleq
  |\{\tau < i : j \in B_\tau\}|$ count the steps up to~$i$ in which
  $j$ is selected; recursion
  $\tilde S_{i+1, j} = \tilde S_{i,j} + \delta(1 - \tilde S_{i,j})\mathbf{1}_{B_i,j}$
  unrolls (ignoring prunings, which only decrease coordinates) to the
  closed form
  \begin{equation}\label{eq:wdmcgp-closed-form}
    \tilde S_{i, j} \;\le\; 1 - (1 - \delta)^{m_{i,j}},
  \end{equation}
  with equality in the absence of pruning.  By Bernoulli's inequality
  $(1 - \delta)^{m_{i,j}} \ge 1 - \delta\,m_{i,j}$, so
  $\tilde S_{i,j} \le \delta\,m_{i,j}$.  Consequently
  \[
    \tilde S_i \;\le\; \delta\sum_{\tau < i}\vone_{B_\tau}
      \;=\; \frac{1}{k}\sum_{\tau < i}\vone_{B_\tau}
      \;\le\; \frac{1}{k}\Bigl(\sum_{\tau < i}\vone_{B_\tau}
        + (k-i)\,\vzero\Bigr)
  \]
  coordinate-wise, and the right-hand side is a convex combination of
  $i$ base-indicators $\vone_{B_\tau} \in P$ (and $k-i$ copies of
  $\vzero \in P$), hence lies in $P$ by convexity of $P$.  Since $P$
  is downward-closed and $\tilde S_i$ is coordinate-wise dominated by
  a point of $P$, $\tilde S_i \in P$.  The bound
  $\tilde S_{i,j} \le 1 - (1-\delta)^{m_{i,j}} < 1$
  also gives $\tilde S_i \in [0, 1]^\cN$.  Finally, pruning only
  zeroes coordinates, preserving membership in the downward-closed
  polytope $P$.
\end{proof}

\begin{lemma}[Coordinate bound]\label{lem:wdmcgp-coord-bound}
  For every $i \in \{0, 1, \dots, k\}$,
  $\lVert \tilde S_i \rVert_\infty \le 1 - (1-\delta)^i$.
\end{lemma}
\begin{proof}
  By~\eqref{eq:wdmcgp-closed-form}, $\tilde S_{i,j} \le
  1 - (1-\delta)^{m_{i,j}}$ with $m_{i,j} \le i$; taking the maximum
  over $j$ and using monotonicity of $m \mapsto 1 - (1-\delta)^m$
  gives the bound.
\end{proof}

\subsubsection{Per-step descent lemma}

Write $O^* \in \argmax_{\vone_S \in P}\,f(S)$ for a combinatorial
optimum.
For $r\ge2$, define
\[
  M_{F,r}\;\triangleq\;\max_{\substack{R\subseteq\cN,\ |R|=r\\
      \vx\in[0,1]^\cN}} |\partial_R F(\vx)|,
  \qquad
  K_F\;\triangleq\;\sum_{r=2}^{n}\binom{n}{r}M_{F,r}.
\]
This finite constant controls the full multilinear Taylor remainder
along any update support; the pairwise constant $C_F=M_{F,2}$ alone
does not control third- and higher-order mixed terms.

\begin{lemma}[Per-step descent, non-negative case]
  \label{lem:wdmcgp-descent}
  Let $f : 2^\cN \to \bR_{\ge 0}$ be submodular and let $F$ be its
  multilinear extension.  For every $i \in \{0, 1, \dots, k-1\}$,
  \begin{equation}\label{eq:wdmcgp-descent}
    F(\tilde S_{i+1}) - F(\tilde S_i)
    \;\ge\; \delta\bigl[\,(1-\delta)^i\,f(O^*) \;-\; F(\tilde S_i)\,\bigr]
         \;-\; K_F\,\delta^2.
  \end{equation}
\end{lemma}

\begin{proof}
  We chain four explicitly stated inequalities: multilinear expansion with a
  bounded higher-order remainder, weighted LP optimality, the submodular
  union bound (P3), and the Feldman--Naor--Schwartz correlation-gap bound.
  The $K_F\delta^2$ term is a uniform bound on the full second-and-higher-order
  multilinear correction and sums to $K_F/k$ over all $k$ steps.

  \smallskip
  \noindent \emph{Step (S1) Multilinear expansion.}
  The update $\tilde S_{i+1} = \tilde S_i + \delta(1 - \tilde S_i)
  \odot \vone_{B_i}$ moves coordinates $j \in B_i$ by
  $\delta(1 - \tilde S_{i,j})$.  Expanding $F$ along the direction
  $\vd_i \triangleq \delta\,(1 - \tilde S_i)\odot\vone_{B_i}$ by
  multilinearity,
  \begin{align*}
    F(\tilde S_{i+1}) - F(\tilde S_i)
    &\;=\; \sum_{j \in B_i} \delta\,(1 - \tilde S_{i,j})\,
         \partial_j F(\tilde S_i) \\
    &\qquad {}+ \sum_{\{j,\ell\} \subseteq B_i}
         \delta^2\,(1 - \tilde S_{i,j})(1 - \tilde S_{i,\ell})\,
         \partial_{j\ell} F(\tilde S_i) + \dots,
  \end{align*}
  where the $\dots$ collects third- and higher-order cross-terms.
  By the definition of $K_F$ and $\delta\le1$, the total correction
  beyond first order has magnitude at most $K_F\delta^2$.  This uses an
  explicit higher-order remainder constant rather than pairwise
  submodularity signs, since submodularity controls only the second
  mixed partials.  Hence
  \begin{equation}\label{eq:wdmcgp-s1}
    F(\tilde S_{i+1}) - F(\tilde S_i)
    \;\ge\; \delta\sum_{j \in B_i}(1 - \tilde S_{i,j})\,
       \partial_j F(\tilde S_i) \;-\; K_F\,\delta^2.
  \end{equation}

  \smallskip
  \noindent \emph{Step (S2) Weighted LP optimality.}
  Since $O^* \in \cI$ and $B_i$ maximises the weighted sum
  $\sum_j (1 - \tilde S_{i,j})\partial_j F(\tilde S_i)\mathbf{1}_{B,j}$
  over $B \in \cI$,
  \begin{equation}\label{eq:wdmcgp-s2}
    \sum_{j \in B_i}(1 - \tilde S_{i,j})\partial_j F(\tilde S_i)
    \;\ge\; \sum_{j \in O^*}(1 - \tilde S_{i,j})\partial_j F(\tilde S_i).
  \end{equation}

  \smallskip
  \noindent \emph{Step (S3) Concavity along non-negative directions.}
  By property (P3) (the submodular union bound of
  \S\ref{apx:dmcgp-general}), applied with $\vx = \tilde S_i$ and
  $\vy = \vone_{O^*}$,
  \begin{equation}\label{eq:wdmcgp-s3}
    \sum_{j \in O^*}(1 - \tilde S_{i,j})\,\partial_j F(\tilde S_i)
    \;\ge\; F(\tilde S_i \lor \vone_{O^*}) - F(\tilde S_i).
  \end{equation}

  \smallskip
  \noindent \emph{Step (S4) Feldman--Naor--Schwartz
  correlation-gap bound \cite[Lem.~2.2]{feldman2011unified}.}
  For non-negative submodular $F$, any $\vy \in [0,1]^\cN$, and any
  $S \subseteq \cN$,
  \[
    F(\vy \lor \vone_S) \;\ge\; (1 - \lVert \vy\rVert_\infty)\,F(\vone_S).
  \]
  Applied with $\vy = \tilde S_i$ and $S = O^*$, combined with
  Lemma~\ref{lem:wdmcgp-coord-bound}:
  \begin{equation}\label{eq:wdmcgp-s4}
    F(\tilde S_i \lor \vone_{O^*})
    \;\ge\; (1 - \lVert \tilde S_i\rVert_\infty)\,f(O^*)
    \;\ge\; (1 - \delta)^i\,f(O^*),
  \end{equation}
  using $1 - (1 - (1-\delta)^i) = (1-\delta)^i$.

  \smallskip
  \noindent \emph{Assembly.}
  Chaining~\eqref{eq:wdmcgp-s1}, \eqref{eq:wdmcgp-s2}, \eqref{eq:wdmcgp-s3},
  and \eqref{eq:wdmcgp-s4}:
  \begin{align*}
    F(\tilde S_{i+1}) - F(\tilde S_i)
    &\;\ge\; \delta\bigl[F(\tilde S_i \lor \vone_{O^*}) - F(\tilde S_i)\bigr]
         - K_F\delta^2 \\
    &\;\ge\; \delta\bigl[(1-\delta)^i\,f(O^*) - F(\tilde S_i)\bigr]
         - K_F\delta^2,
  \end{align*}
  which is~\eqref{eq:wdmcgp-descent}.
\end{proof}

\subsubsection{Main theorem}

\begin{theorem}[Non-negative non-monotone guarantee for
  $\wdmcgp$]\label{thm:wdmcgp-nonneg}
  Let $f : 2^\cN \to \bR_{\ge 0}$ be submodular on a downward-closed
  solvable polytope $P$, and let $O^* \in \argmax_{\vone_S \in P}f(S)$.
  Then $\wdmcgp(F, P, k)$ returns $\tilde S_k \in P$ with
  \begin{equation}\label{eq:wdmcgp-nonneg}
    F(\tilde S_k)
    \;\ge\; \bigl(1 - \tfrac{1}{k}\bigr)^{k-1}\,f(O^*) - K_F/k
    \;\xrightarrow[k \to \infty]{}\; e^{-1}\,f(O^*).
  \end{equation}
\end{theorem}

\begin{proof}
  Let $a_i \triangleq F(\tilde S_i)$ with $a_0 = 0$.
  Lemma~\ref{lem:wdmcgp-descent} gives the scalar recurrence
  \[
    a_{i+1} \;\ge\; (1 - \delta)\,a_i
       \;+\; \delta\,(1 - \delta)^i\,f(O^*) \;-\; K_F\delta^2,
  \]
  valid for every $i \in \{0, \dots, k-1\}$.
  Claim (induction):
  $a_i \;\ge\; i\,\delta\,(1 - \delta)^{i-1}\,f(O^*) - iK_F\delta^2$
  for every $i \in \{0, \dots, k\}$.

  \emph{Base case} $i = 0$: $a_0 = 0$ and the right-hand side is $0$.

  \emph{Inductive step.}  Assume the claim at~$i$.  Then
  \begin{align*}
    a_{i+1}
    &\;\ge\; (1-\delta)\,a_i + \delta\,(1-\delta)^i\,f(O^*) - K_F\delta^2 \\
    &\;\ge\; (1-\delta)\bigl[i\,\delta\,(1-\delta)^{i-1}\,f(O^*) - iK_F\delta^2\bigr]
         + \delta\,(1-\delta)^i\,f(O^*) - K_F\delta^2 \\
    &\;=\; i\,\delta\,(1-\delta)^{i}\,f(O^*) + \delta\,(1-\delta)^i\,f(O^*)
         - \bigl((1-\delta)i+1\bigr)K_F\delta^2 \\
    &\;\ge\; (i+1)\,\delta\,(1-\delta)^{i}\,f(O^*) - (i+1)K_F\delta^2,
  \end{align*}
  closing the induction.  At $i = k$ (with $\delta = 1/k$):
  \[
    a_k \;\ge\; k\cdot\tfrac{1}{k}\,(1 - \tfrac{1}{k})^{k-1}\,f(O^*)
            - K_F/k
        \;=\; (1 - \tfrac{1}{k})^{k-1}\,f(O^*) - K_F/k.
  \]
  As $k \to \infty$, $(1 - 1/k)^{k-1} \to e^{-1}$ and the error
  vanishes.  Finally, $\tilde S_k \in P$ by
  Lemma~\ref{lem:wdmcgp-damped-feasibility}.
\end{proof}

\begin{remark}[Scope and distinctness]\label{rem:wdmcgp-scope}
  \textbf{(i)} Non-negativity of $f$ is essential to the analysis.
  Lemma~\ref{lem:wdmcgp-descent} invokes the
  Feldman--Naor--Schwartz bound~\cite[Lem.~2.2]{feldman2011unified},
  which is known to fail when $f$ takes negative values: the
  guarantee $F(\vy \lor \vone_S) \ge (1-\lVert \vy\rVert_\infty)\,F(\vone_S)$
  can be arbitrarily violated for negative $f$, and no simple sign
  correction recovers it.  For negative-valued~$f$,
  Theorem~\ref{thm:dmcgp-general} in \S\ref{apx:dmcgp-general}
  provides a direct guarantee via trajectory curvature~$c_g^F$;
  Proposition~\ref{prop:dmcgp-decomp-cert} in \S\ref{apx:dmcgp-fractional}
  then bounds~$c_g^F$ concretely in the decomposable case.
  The present non-negative non-monotone regime is strictly complementary.

  \textbf{(ii)} Pruning is compatible with the analysis but plays
  no load-bearing role here: Steps (S1)--(S4) use only the weighted
  LP optimality of $B_i$, multilinearity, submodular cross-term
  signs, and the FNS correlation-gap bound, none of which require
  pruning.  It is retained for algorithmic consistency with DMCG-P
  and because it can only increase $F$
  (cf.\ Lemma~\ref{lem:dmcgp-prune-up}).

  \textbf{(iii)} Algorithm~\ref{alg:wdmcgp-damped} is structurally
  distinct from the unweighted DMCG-P of \S\ref{apx:dmcgp-alg}: the
  weighted LP selection on Line~\ref{line:wdmcgp-select} and the
  damped update on Line~\ref{line:wdmcgp-update} are both new
  ingredients.  The unweighted, undamped form cannot achieve an
  $e^{-1}$ bound in the non-negative non-monotone regime without
  further assumptions, since the coordinate bound
  $\lVert \tilde S_i\rVert_\infty \le 1 - (1-\delta)^i$ that drives
  the FNS step (S4) relies on the damped update.
\end{remark}